\documentclass{article}
\usepackage{xr}

\usepackage[preprint, nonatbib]{neurips_2025}
\usepackage[numbers]{natbib}

\usepackage[utf8]{inputenc} 
\usepackage[T1]{fontenc}    
\usepackage{hyperref}       
\usepackage{url}            
\usepackage{booktabs}       
\usepackage{amsfonts}       
\usepackage{nicefrac}       
\usepackage{microtype}      
\usepackage[dvipsnames]{xcolor}         
\usepackage{wrapfig}

\usepackage{algorithm}
\usepackage{algorithmic}
\usepackage{enumitem}
\setlist[enumerate]{leftmargin=*, itemsep=0pt, topsep=0pt}

\title{Faithful Group Shapley Value}

\author{%
    Kiljae Lee$^*$\\
    The Ohio State University\\
    \texttt{lee.10428@osu.edu}
    \And
    Ziqi Liu\thanks{Lee and Liu equally contributed.  
    They are {\bf co-first authors} and were listed {\bf alphabetically}.}\\
    Carnegie Mellon University\\
    \texttt{ziqiliu2@andrew.cmu.edu}
    \AND
    Weijing Tang\\
    Carnegie Mellon University\\
    \texttt{weijingt@andrew.cmu.edu}
    \And
    Yuan Zhang\thanks{Corresponding author.}\\
    The Ohio State University\\
    \texttt{yzhanghf@stat.osu.edu}
}

\usepackage{amsfonts,amsmath,amsthm,amssymb,color,graphicx,bm,bbm,mathrsfs}
\usepackage{etoc}
\usepackage{cleveref}
\usepackage{wrapfig,lipsum,booktabs}
\usepackage{thmtools}
\Crefname{equation}{Eq.}{Eqs.}

\usepackage{tcolorbox}
\tcbuselibrary{listings, breakable}

\newtcolorbox[auto counter, number within=section]{algobox}[2][]{%
    breakable,
    floatplacement=wrap,
    colback=white,
    colframe=black,
    fonttitle=\bfseries,
    title=Algorithm~\thetcbcounter: #2,
    label={#1},
    width=0.5\textwidth,
    boxrule=0.5pt,
    arc=3pt,
    halign=center
}

\newtheorem{theorem}{Theorem}
\newtheorem*{theorem*}{Theorem}
\newtheorem{lemma}{Lemma}
\newtheorem{proposition}{Proposition}

\newtheorem{definition}{Definition}
\newtheorem{assumption}{Assumption}
\newtheorem{remark}{Remark}
\newtheoremstyle{example}
  {3pt}   
  {3pt}   
  {\itshape}  
  {}      
  {\bfseries} 
  {.}     
  { }     
  {\thmname{#1}\thmnumber{ #2}\thmnote{ (#3)}} 

\theoremstyle{example}

\makeatletter
\renewcommand{\section}{\@startsection {section}{1}{\z@}%
  {-1ex plus -0.2ex minus -.1ex}
  {0.7ex plus 0.1ex minus 0.1ex}
  {\large\bf\raggedright}}

\renewcommand{\subsection}{\@startsection{subsection}{2}{\z@}%
  {-0.9ex plus -0.2ex minus -.1ex}
  {0.4ex plus 0.1ex}
  {\normalsize\bf\raggedright}}

\renewcommand{\subsubsection}{\@startsection{subsubsection}{3}{\z@}%
  {-0.7ex plus -0.2ex minus -.1ex}
  {0.3ex plus 0.1ex}
  {\normalsize\bf\raggedright}}

\renewcommand{\paragraph}{\@startsection{paragraph}{4}{\z@}%
  {0.05ex plus 0.01ex minus .005ex}
  {-0.3em}
  {\normalsize\bf}}
\makeatother

\newcommand{\pr}{\mathbb{P}}
\newcommand{\ep}{\mathbb{E}}

\definecolor{yuancolor}{rgb}{0,0.42,0.24}
\definecolor{bleudefrance}{rgb}{0.19, 0.55, 0.91}

\newcommand{\SV}{\operatorname{SV}}
\newcommand{\GSV}{\operatorname{GSV}}
\newcommand{\FGSV}{\operatorname{FGSV}}

\renewcommand{\hat}{\widehat}
\renewcommand{\tilde}{\widetilde}

\newtheorem{keyobservation}{Key observation}

\usepackage[normalem]{ulem}

\begin{document}

\maketitle

\begin{abstract}
    Data Shapley is an important tool for data valuation, which quantifies the contribution of individual data points to machine learning models. 
    In practice, group-level data valuation is desirable when data providers contribute data in batch. 
    However, we identify that existing group-level extensions of Data Shapley are vulnerable to \textit{shell company attacks}, where strategic group splitting can unfairly inflate valuations. 
    We propose Faithful Group Shapley Value (FGSV) that uniquely defends against such attacks. 
    Building on original mathematical insights, we develop a provably fast and accurate approximation algorithm for computing FGSV. 
    Empirical experiments demonstrate that our algorithm significantly outperforms state-of-the-art methods in computational efficiency and approximation accuracy, while ensuring faithful group-level valuation.
\end{abstract}

\section{Introduction}
\label{sec:intro}

As data become increasingly crucial in modern machine learning, quantifying its value has significant implications for faithful compensation and data market design~\citep{sestino2025decoding, zhang2024survey}. 
The Shapley value, a foundational concept from cooperative game theory~\citep{shapley1953value}, stands out as the unique valuation method satisfying four desirable axioms for faithful data valuation~\citep{ghorbani2019data, jia2019towards}. 
Consequently, it has been applied to diverse applications, including collaborative intelligence~\citep{li2021shapley, liu2022gtg} and copyright compensation~\citep{wang2024economic}, and
quantifying feature importance in explainable AI~\citep{lundberg2017unified}.

\begin{wrapfigure}{r}{0.45\textwidth} 
    \vspace{-2em} 
    \includegraphics[width=0.4\textwidth]{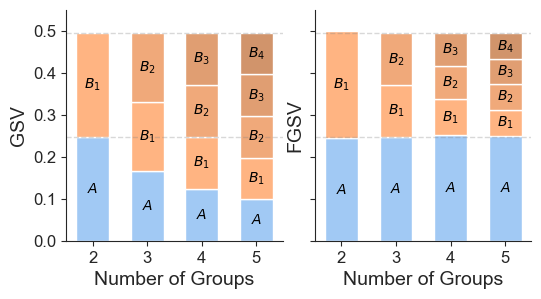}
    \vspace{-1em}
    \caption{
        Left: GSV; right: FGSV (our method).
        Vertical span: valuation.
        Group $A$ is fixed; group $B$ engages increasing degrees of \emph{shell company attack} (left$\to$right).
        Detailed experimental set-up in Appendix~\ref{sec::exp_detail}.
    }
    \vspace{-1em} 
    \label{fig::simu::gsv-vs-fgsv}
\end{wrapfigure}

Many real-world scenarios demand evaluation of \emph{data sets} instead of individual data points.
In applications like data marketplaces~\citep{zhang2024survey} or compensation allocation for generative AI~\citep{deng2023computational, wang2024economic}, data are contributed by owners who possess entire datasets, making group-level data valuation a natural choice. 
A more technical motivation is that individual-level Shapley values are often computationally challenging to approximate for big data; while group-level valuation is much faster and can provide useful insights~\citep{kang2024autoscale}.
In explainable AI, group-level feature importance reportedly provides more robust and interpretable valuation \citep{wang2025group, covert2021improving}.

As group data valuation frameworks are increasingly used in real-world scenarios such as data markets and copyright compensation, their robustness against adversarial manipulation becomes important.
Prior studies on attacks and defenses in data valuation have focused primarily on the \emph{individual-level} data valuation, such as the \textit{copier attack} that duplicates existing data to unjustly inflate value~\citep{Falconer2025-kj,han2022replication}. In contrast, we identify a new vulnerability unique to group data valuation, which has been largely unexplored in literature.
Existing extensions of the Shapley framework to group data valuation treat each pre-defined group as an atomic unit and then apply the standard Shapley value formulation~\citep{jullum2021groupshapley, kwon2024group, wang2024economic, wang2025group}, which we refer to as Group Shapley Value (GSV). 
However, we demonstrate---both theoretically and empirically---that this approach suffers from susceptibility to strategic manipulation through partitioning.
Consider the valuation of a fixed data group $A$. 
As shown in Figure~\ref{fig::simu::gsv-vs-fgsv}, partitioning the remaining data points into smaller subgroups reduces the GSV for $A$, despite no change in data content.
In other words,  malicious players may exploit the loophole in GSV by splitting their data among puppy subsidiaries, a vulnerability we coin as {\bf shell company attack}.

To address this issue, we propose a faithfulness axiom for group data valuation, such that the total valuation of the same set of data remains unchanged, regardless of how others are subdivided. 
Based on this, we introduce the {\bf Faithful Group Shapley Value (FGSV)} for a group as the sum of individual Data Shapley values of its members, and prove that it uniquely satisfies the set of axioms that are desirable for faithful group data valuation.
Figure~\ref{fig::simu::gsv-vs-fgsv} demonstrates that FGSV effectively defends against the shell company attack.

While FGSV safeguards faithfulness, its exact computation requires combinatorial computation of  individual Shapley values. 
Although numerous approximation algorithms for individual Shapley values have been developed~\citep{castro2009polynomial,ghorbani2019data, jia2019towards, wang2023note2, lundberg2017unified, covert2021improving, zhang2023efficient, li2024one, musco2025provably}, simply applying them and summing up results in each group is slow for large groups and may compound approximation errors.
Importantly, we proved that a small subset of terms dominate in FGSV's formula.
As a result, we developed an original algorithm that directly approximates FGSV fast and accurately.
Our numerical experiments demonstrate its computational efficiency and approximation accuracy against summing up individual Shapley values computed by state-of-the-art (SOTA) methods. 
We also applied FGSV to faithful copyright compensation on Stable Diffusion models for image generation.
\section{Preliminaries: individual and group Shapley values} \label{sec::background}

{\bf Individual Shapley value (SV).}
Let $\mathcal{D}=\{z_1,\ldots,z_n\} \in \mathcal{Z}^n$ denote the training data, where $n:=|\mathcal{D}|$.
Let $S\subseteq[n]$ be an index set of cardinality $s:=|S|$, and write ${\cal S}:=\{z_i:i\in S\}$.
A \emph{utility function} $U({\cal S})$ assigns a performance score (such as classification accuracy) to method trained based on data ${\cal S}$.
To simplify notation, we may refer to ``$U({\cal S})$'' by index as ``$U(S)$'' and write them interchangeably.
The \emph{individual Shapley value} of data point $i$ is defined as
\begin{equation}
    \SV(i) := 
    \sum_{S \subseteq [n] \setminus \{i\}} \frac{|S|! (n - |S| - 1)!}{n!} 
    \{ U(S \cup \{i\}) - U(S) \}.
    \label{eqn::shapley_value_def}
\end{equation}
The widespread adoption of Shapley value stems from its strong theoretical foundation that it uniquely satisfies four desired axioms, called \emph{null player}, \emph{symmetry}, \emph{linearity}, and \emph{efficiency}.
We will review these principles as part of our axiomatization for group Shapley value in Definition \ref{definition::axioms}.

{\bf Group Shapley value (GSV).} 
As mentioned in Section \ref{sec::background}, most group-level Shapley values \citep{jullum2021groupshapley, wang2024economic, wang2025group} adopt the so-called \emph{group-as-individual} (GaI) approach.
Suppose the entire data set is partitioned into $(K+1)$ disjoint groups: $\mathcal{D} = S_0\cup S_1\cup\cdots\cup S_K$, in which, $S_0 \subseteq [n]$ is the group we aim to evaluate.
The \emph{group Shapley value (GSV)} of $S_0$, as in \cite{jullum2021groupshapley, wang2024economic, wang2025group}, is defined as
\begin{align} 
    \GSV(S_0) 
    =&~ 
    \sum_{I \subseteq \{1, \dots, K\}} 
    \frac{|I|! (K - |I|)!}{(K+1)!} 
    \Big\{
        U\Big( \big\{\cup_{k \in I} S_{k}\big\} \cup S_0 \Big) - U\Big( \cup_{k \in I} S_{k} \Big)
    \Big\}.
    \label{eqn::def_GSV_GaI}
\end{align}
Clearly, \eqref{eqn::def_GSV_GaI} is a ``set version'' of \eqref{eqn::shapley_value_def}.
Thus GSV also satisfies the four axioms for SV at the group level.
However, it is questionable whether these axioms, originally formulated for individual-level valuation, are appropriate for groups. 

{\bf Fairness issue with GSV.}
The GSV valuation of $S_0$ as in \eqref{eqn::def_GSV_GaI} can be impacted by how the rest of the data are grouped.
A serious consequence is that GSV is prone to the {\bf shell company attack}: splitting a group into smaller subgroups will increase the total valuation earned by the same set of data points.
We have seen an illustration in Figure \ref{fig::simu::gsv-vs-fgsv} in Section \ref{sec:intro}.
Now we theoretically formalize this observation.
\begin{proposition}[Shell company attack] 
    \label{proposition::problems_with_group_as_individual}
    Let $\mathcal{P}$ be the underlying data distribution over the sample space $\mathcal{Z}$, and write $\mathcal{P}^s$ for the $s$-fold product distribution; thus $S \sim \mathcal{P}^s$ denotes an i.i.d.\ sample $S=(z_1,\ldots,z_s)\in\mathcal{Z}^s$ with $z_i \sim \mathcal{P}$.
    Let $\bar{U}(s) = \mathbb{E}_{S \sim \mathcal{P}^s} [ U(S) ]$ denote the expected utility for data $S$ (which only depends on $s=|S|$). 
    Now we split a group $S_k$ into two non-empty subgroups $S_k'$ and $S_k''$ (i.e., $S_k'\cup S_k'' = S_k$, $S_k'\cap S_k''=\emptyset$, and $S_k', S_k'' \neq \emptyset$),
    If $\bar{U}(s)$ satisfies a \emph{prudence} condition:
    \begin{align}
        \Delta_s^3 \bar{U}(s) := \bar{U}(s+3) - 3\bar{U}(s+2) + 3\bar{U}(s+1) - \bar{U}(s) > 0,
        \label{eqn::prudence}
    \end{align}
    then
    \begin{equation} \label{eqn::problems_with_group_as_individual}
        \ep[ \GSV(S_k) ] < \ep[ \GSV(S_k') ] + \ep[ \GSV(S_k'') ].
    \end{equation}
\end{proposition}
Condition \eqref{eqn::prudence} is a familiar concept in economics, 
characterizing risk-aversion behaviors under uncertainty \citep{gollier2001economics}.
It is also observed in machine learning, where the performance of a learning method saturates as the same type of data repeatedly come in \citep{kaplan2020scaling, hestness2017deep}.

\section{Our method}
\label{section::our-method}

\subsection{Faithfulness axiom and Faithful Group Shapley Value (FGSV)}

Motivated by the faithfulness issue, we introduce a set of axioms that formalize desirable principles for faithful group data valuation. 
We first axiomatize group-level data valuation that generalizes the Shapley's four axioms for individual valuation.
\begin{definition}[Group data valuation] 
    \label{definition::group-level-data-valuation-ziqi}
    For data $\mathcal{D} = \{z_1, \dots, z_n\}$, utility $U$ and partition $\Pi = \{S_1, \dots, S_K\}$ such that $S_{k_1}\cap S_{k_2}=\emptyset, \forall 1\leq k_1<k_2\leq n$ and $S_1 \cup \cdots \cup S_K = [n]$, 
    a \emph{group data valuation} method $\nu_{U, \mathcal{D}, \Pi}(\cdot): \Pi \to \mathbb{R}$ is a mapping  that assigns a real-valued score to each $S_k\in \Pi$.
\end{definition} 

Next, we propose a set of axioms that a faithful group data valuation method should satisfy.

\begin{definition}[Axioms for faithful group data valuation] 
    \label{definition::axioms}
    A group data valuation $\nu_{U, \mathcal{D}, \Pi}$ is called \textbf{faithful} if it satisfies the following axioms for any dataset $\mathcal{D}$, utility $U$, and partition $\Pi$:
    \begin{enumerate}
        \item \textbf{Null player}: For any $S\in \Pi$, if every subset $S' \subseteq S$ satisfies $U(S'' \cup S') = U(S'')$ for all $S'' \subseteq [n] \setminus S'$, 
        then $\nu_{U, \mathcal{D}, \Pi}(S) = 0$.
        
        \item \textbf{Symmetry}: For any $S_1, S_2 \in \Pi$, $|S_1|=|S_2|$, if there is a bijection $\sigma: S_1 \to S_2$, s.t. $U(S'' \cup S') = U(S'' \cup \sigma(S'))$ for all $S' \subseteq S_1$ and $S'' \subseteq [n] \setminus (S' \cup \sigma(S'))$, 
        then $\nu_{U, \mathcal{D}, \Pi}(S_1) = \nu_{U, \mathcal{D}, \Pi}(S_2)$.

        \item \textbf{Linearity}: For any utility functions $U_1, U_2$ and scalars $\alpha_1, \alpha_2 \in \mathbb{R}$, we have for all $S \in \Pi$ $\nu_{\alpha_1 U_1 + \alpha_2 U_2, \mathcal{D}, \Pi}(S)
        = \alpha_1 \nu_{U_1, \mathcal{D}, \Pi}(S) + \alpha_2 \nu_{U_2, \mathcal{D}, \Pi}(S)$.
    
        \item \textbf{Efficiency}: For any partition $\Pi = \{S_1, \dots, S_K\}$, we have 
        $
        \sum_{k=1}^{K} \nu_{U, \mathcal{D}, \Pi}(S_k) = U([n])$.
        
        \item \textbf{Faithfulness}: 
        \label{axiom::faithfulness}
        For any group $S \in \Pi_1 \cap \Pi_2$, we have $\nu_{U,\mathcal{D}, \Pi_1}(S) = \nu_{U,\mathcal{D}, \Pi_2}(S)$. 
    \end{enumerate}
\end{definition}

Notably, our \emph{null player} and \emph{symmetry} require group valuation to faithfully reflect the contributions of its individual members. 
They are strictly weaker assumptions than their GaI counterparts.
For example, a group assigned zero value in GaI can receive nonzero value under our axioms if some of its members have marginal contributions, but not the other way around.
\emph{Linearity} and \emph{efficiency} are standard group-level extensions.
Our newly introduced Axiom \ref{axiom::faithfulness} 
requires that a group's value is determined only on its own competitive merit and rules out the unfair competition means of shell company attack.

It turns out that there is a unique group data valuation method that can satisfy all axioms in Definition~\ref{definition::axioms}: simply add up all member's individual Shapley values.

\begin{theorem}
    \label{theorem::axioms}
    The only group data valuation method $\nu_{U, \mathcal{D}, \Pi}$ that satisfies all axioms in Definition~\ref{definition::axioms} is 
    $
    \nu_{U, \mathcal{D}, \Pi}(S) = \sum_{i \in S} \SV(i),
    $
    where $\SV(i)$ is the individual Shapley value  defined in \eqref{eqn::shapley_value_def}.
\end{theorem}

In view of Theorem~\ref{theorem::axioms}, we propose {\bf Faithful Group Shapley Value (FGSV)}:
\begin{equation}
    \FGSV(S_0) := \sum_{i \in S_0} \SV(i).
    \label{eqn::sum-shapley-in-S0}
\end{equation}

\subsection{Fast and accurate approximation algorithm for FGSV}

The exact evaluation of FGSV requires combinatorial computation.
To develop a feasible approximation method, we make a series of key mathematical observations leading to an efficient algorithm.

To start, from the definition of FGSV \eqref{eqn::sum-shapley-in-S0} and \eqref{eqn::shapley_value_def}, we see that $\FGSV(S_0)$ is a complicated linear combination of $U(S)$ terms, where $S$ ranges over all subsets of $[n]$.
Therefore, the first step towards simplification is to discover the pattern in the coefficient in front of each $U(S)$ term.
\begin{keyobservation}
    \label{keyobs::coef}
    In $\FGSV(S_0)$, the coefficient of $U(S)$ depends on $S$ and $S_0$ only through the tuple
    $
        (s_1,s,s_0),
    $
    where recall that $s_0:=|S_0|$ and $s:=|S|$, and define $s_1:=|S_0\cap S|$.
\end{keyobservation}
In other words, any two terms $U(S)$ and $U(S')$ with $|S|=|S'|$ and $|S_0\cap S| = |S_0\cap S'|$ share the same coefficient in $\FGSV(S_0)$.
This motivates us to aggregate these terms in our analysis.
Let 
$
    \mathscr{A}_{s, s_1} 
    := 
    \{S: |S|=s, |S\cap S_0|=s_1\}
$
collect all $S$'es with the same $(s, s_1)$ configuration, and~define
\begin{align}
    \mu\left( \frac{s_1}{s}; s, s_0, n \right) 
    :=&~ 
    \frac{ \sum_{S\in {\cal A}_{s,s_1}} U(S) }
    {|{\cal A}_{s,s_1}|}
    =
    \frac{ \sum_{S: |S|=s, |S\cap S_0|=s_1} U(S) }
    { \binom{s_0}{s_1} \binom{n - s_0}{s - s_1} }.
    \label{eqn::def_mu}
\end{align}
When sampling a subset $S \subseteq [n]$ of size $s$ without replacement, the size of $S\cap S_0$, which we now denote as the boldfaced $\boldsymbol{s_1}$ to emphasize its randomness, follows a hypergeometric distribution: $
    \pr(\boldsymbol{s_1}=s_1)
    =
    \binom{s_0}{s_1}\binom{n-s_0}{s-s_1}/\binom ns.
$
Using this fact, we can re-express FGSV in terms of $\mu$. 
\begin{lemma}
    \label{lemma::group_shapley_closed_form_exact}
    Let $\boldsymbol{s_1}\sim\mathcal{HG}(n, s_0, s)$.
    We can rewrite FGSV as 
    \begin{align}
        \FGSV(S_0) 
        =&~ 
        \frac{s_0}{n} \left[ U([n]) - U(\varnothing) \right] 
        + \sum_{s=1}^{n - 1} 
        {\cal T}(s),
        \label{eqn::rewrite-FGSV}
    \end{align}
    where 
    \begin{align}
        {\cal T}(s)
        :=&~
        \mathbb{E}_{\boldsymbol{s_1} \sim \mathcal{HG}(n, s_0, s)} \left[ \frac{n}{n - s} \left( \frac{\boldsymbol{s_1}}{s} - \frac{s_0}{n} \right) \mu\left( \frac{\boldsymbol{s_1}}{s}; s, s_0, n \right) \right].
        \label{eqn::def::T(s)}
    \end{align}
\end{lemma}
Instead of directly estimating ${\cal T}(s)$ via Monte Carlo, we discover two key observations that deepen our understandings of \eqref{eqn::rewrite-FGSV} and \eqref{eqn::def::T(s)}, building on which, we can greatly reduce computational cost.
\begin{keyobservation}
    \label{keyobservation::HG-exp-decay}
    The probability $\pr(\boldsymbol{s_1}=s_1)$ decays exponentially in $|s_1 - s s_0/n|$.
\end{keyobservation} 
Key observation \ref{keyobservation::HG-exp-decay} implies that ${\cal T}(s)$ is dominated by the values of $\boldsymbol{s_1}$ around its mean $\ep[\boldsymbol{s_1}] = s s_0/n$.
To deepen our understanding of ${\cal T}(s)$, for this moment, we informally deem $\mu$ as a smooth function of the continuous variable $s_1/s$, with formal characterization provided later. Then, applying a ``Taylor expansion'' of $\mu\left( s_1/s ; s, s_0, n \right)$ around $s_1/s = s_0/n$ leads to the following intuition.
\begin{keyobservation}[Informal]
    \label{keyobservation::T(s)-approx-formula}
    $
        {\cal T}(s)
        \approx
        s^{-1} (s_0/n)(1-s_0/n) \mu'(s_0/n; s,s_0,n).
    $
\end{keyobservation}
Key observation~\ref{keyobservation::T(s)-approx-formula} reveals that, under suitable conditions, the term ${\cal T}(s)$ can be efficiently estimated by evaluating the derivative $\mu'\left( \cdot ; s, s_0, n \right)$ at a single point $s_0/n$. 

Next, we formalize the above discoveries into rigorous mathematical results. 
\begin{assumption}[Boundedness]
    \label{assump::utility_boundedness}
    For all $s \in \mathbb{N}$ and $\mathcal{S} \in \mathcal{Z}^s$,
    $
        |U(\mathcal{S})| \leq C
    $
    for a universal constant $C$.
\end{assumption}
\begin{assumption}[Second-order algorithmic stability of utility]
    \label{assump::utility_stability}
    There exist constants $C > 0$ and $\upsilon > 0$ such that for all $s \in \mathbb{N}$, $\mathcal{S} \in \mathcal{Z}^s$ and $z_1, z_1', z_2, z_2' \in \mathcal{Z}$,
    \begin{equation}
        \big| U(\mathcal{S} \cup \{z_1, z_1'\}) - U(\mathcal{S} \cup \{z_1, z_2\}) - U(\mathcal{S} \cup \{z_1', z_2'\}) + U(\mathcal{S} \cup \{z_2, z_2'\}) \big| \leq 
        C s^{-(3/2+\upsilon)}.
        \notag
    \end{equation}
\end{assumption}Assumption~\ref{assump::utility_boundedness} is a standard regularity condition commonly adopted in the data valuation literature~\citep{jia2019towards, wang2023data}, 
and Assumption~\ref{assump::utility_stability} introduces a mild second-order stability requirement. 
In Section~\ref{subsec::examples_of_second_order_stable_algorithms}, we will show examples that Assumption \ref{assump::utility_stability} is satisfied by some commonly used utility functions. 
\begin{theorem} 
    \label{proposition::calT}
    Under Assumptions~\ref{assump::utility_boundedness} and~\ref{assump::utility_stability}, for each $s \in \{1, \dots, n - 1\}$, we have
    \begin{align} 
        \mathcal{T}(s) 
        =&~ 
        n/(n-1)\cdot\alpha_0 (1 - \alpha_0) 
        \Big\{
            \Delta\mu\big( s_1^*/s; s, s_0, n \big) 
            + 
            O\big(s^{-(1 + \upsilon)}\big)
        \Big\},
        \label{eqn::calT_approximation}
    \end{align}
    where 
    $s_1^* := \lfloor s s_0/n \rfloor$, 
    $\alpha_0 := s_0/n$,
    and
    $\Delta\mu(\frac{s_1}{s}; s, s_0, n ) := \mu(\frac{s_1+1}{s}; s, s_0, n ) - \mu(\frac{s_1}{s}; s, s_0, n )$.
\end{theorem}
Theorem \ref{proposition::calT} consolidates Key observation \ref{keyobservation::T(s)-approx-formula} and, moreover, shows that its approximation error decays rapidly as $s$ grows. 
For large $s$, we can use \eqref{eqn::calT_approximation}
to design the estimator for $\mathcal{T}(s)$.
Yet, some extra care is needed.  
In principle, each $\mu(s_1/s;s,s_0,n)$ can be estimated by subsampling $S$ from ${\cal A}_{s,s_1}$.
\begin{align} 
    \hat{\mu}_m\left( \frac{s_1}{s}; s, s_0, n \right) 
    :=&~ 
    \frac{1}{m} \sum_{j=1}^m U(S^{(j)}), \quad \text{where } \{S^{(j)}\}_{j=1}^m \stackrel{\text{i.i.d.}}{\sim} \operatorname{Uniform}(\mathscr{A}_{s, s_1}).
    \label{eqn::mu_approximation_empirical}
\end{align}
However, estimating $\mu(\frac{s_1+1}{s}; s, s_0, n)$ and $\mu(\frac{s_1}{s}; s, s_0, n)$ \emph{separately} can be \emph{statistically inefficient}, as the noise from two independent Monte Carlo estimates could mask the signal in their difference---especially when the true gap is small for large $s$. Following the variance-reduction technique used in stochastic simulation~\citep{law1982simulation}, 
we propose to estimate $\Delta\mu$ directly using \emph{paired} Monte Carlo terms:
\begin{align}
    \widehat{\Delta\mu}_m\left( \frac{s_1}{s}; s, s_0, n \right) 
    :=&~ 
    \frac{1}{m} \sum_{j=1}^m 
    \Big\{ 
        U\big( S^{(j)} \cup \{i_1^{(j)}\} \big) 
        - 
        U\big( S^{(j)} \cup \{i_2^{(j)}\} \big)
    \Big\}, 
    \label{eqn::Delta_mu_approximation_empirical}
\end{align}
where the tuple $\{(S^{(j)}, i_1^{(j)}, i_2^{(j)})\}$ is i.i.d. sampled from 
$
    \big\{ 
        (S, i_1, i_2): |S| = s, |S \cap S_0| = s_1, i_1 \in S_0 \setminus S, i_2 \in S_0^c \setminus S 
    \big\}
$. 

When $s$ is small and the approximation becomes less accurate, we instead estimate ${\cal T}(s)$ via direct Monte Carlo using \eqref{eqn::def::T(s)} and \eqref{eqn::mu_approximation_empirical}.

\begin{wrapfigure}{r}{0.56\textwidth}
  \vspace{-\intextsep}
  \begin{minipage}{\linewidth}
  {\small
    \begin{algorithm}[H] 
    \caption{Approximate $\FGSV(S_0)$}
    \label{alg::gsv}
    \begin{algorithmic}[1]
    \REQUIRE Dataset $\mathcal{D}$, group $S_0$, threshold $\bar{s}$, subsample sizes $m_1, m_2$.
    \STATE Initialize $n =|\mathcal{D}|$, $s_0= |S_0|$ and $\alpha_0 = s_0 / n$.
    \FOR{$s = 1$ to $n - 1$}
        \IF{$s < \bar{s}$}
            \STATE 
            Estimate $\hat{\mu}_{m_1}(\frac{s_1}{s}; s, s_0, n)$ for each $s_1\in [\max\{0, s{+}s_0{-}n\}, \min\{s, s_0\}]$ by Eq.~\eqref{eqn::mu_approximation_empirical}. 
            \STATE 
            Compute $\hat{\mathcal{T}}(s)$ by \eqref{eqn::def::T(s)}, replacing $\mu$ by $\hat\mu_{m_1}$.
        \ELSE
            \STATE 
            $s_1^* \gets \lfloor s \alpha_0 \rfloor$.
            
            \STATE 
            \label{algline:estimate_deltamu_final}
            Estimate $\widehat{\Delta\mu}_{m_2}(\frac{s_1^*}{s}; s, s_0, n)$ by Eq.~\eqref{eqn::Delta_mu_approximation_empirical}.
            
            \STATE $\hat{\mathcal{T}}(s) \gets \frac{n}{n - 1} \alpha_0(1 - \alpha_0) \cdot \widehat{\Delta\mu}_{m_2}(\frac{s_1^*}{s}; s, s_0, n)$.
        \ENDIF
    \ENDFOR
    \RETURN $\frac{s_0}{n} \left[ U([n]) - U(\varnothing) \right] + \sum_{s=1}^{n-1}\hat{\mathcal{T}}(s)$.
    \end{algorithmic}
    \end{algorithm}
    }
    \vspace{-1.7cm}
  \end{minipage}
\end{wrapfigure}
We formally present our method as Algorithm~\ref{alg::gsv}. 
Later, our Theorem~\ref{theorem::computationalcomplexity} will provide quantitative guidance on choosing this threshold for deciding whether $s$ is small or large.

\subsection{Computational complexity}

Following the convention, we measure computational complexity by the number of utility function evaluations required for a given approximation accuracy.
\begin{definition}[$(\epsilon, \delta)$-approximation]
    \label{def:epsilon_delta_approx}
    For a target vector $\theta \in \mathbb{R}^d$, an estimator $\hat{\theta}$ is called an \textbf{$(\epsilon, \delta)$-approximation}, if 
    $
        \pr\big( \|\hat{\theta} - \theta\|_2 \ge \epsilon \big) \leq \delta.
    $
    \vspace{-0.1cm}
\end{definition}
Our theoretical analysis for approximating $\mathcal{T}(s)$ relies on the stability of the utility function $U$. 
Specifically, we employ the concept of \emph{deletion stability} to quantify the maximum change in the utility function when a single data point is removed \citep{bousquet2002stability, hardt2016train}. 
\begin{definition}[Deletion Stability] \label{definition::deletion_stability}
    A utility function $U$ is $\beta(s)$-deletion stable for a non-increasing function $\beta: \mathbb{N} \to \mathbb{R}^+$, if
    \begin{equation*}
        |U(\mathcal{S} \cup \{z\}) - U(\mathcal{S})| \leq \beta(s),
    \end{equation*}
    for all $s \in \mathbb{N}$, $\mathcal{S} \in \mathcal{Z}^{s-1}$ and $z \in \mathcal{Z}$.
    
\end{definition}

The regime $\beta(s) = O(1/s)$ for deletion stability is commonly assumed in the literature on individual Data Shapley approximation~\citep{wu2023variance, wu2024uncertainty} and in the analysis of algorithm stability~\citep{bousquet2002stability, hardt2016train}. 

\begin{theorem} \label{theorem::computationalcomplexity}
    Suppose the utility function $U$ is $O(1/s)$-deletion stable, then Algorithm~\ref{alg::gsv} guarantees that for any truncation threshold $\bar{s}$ and sample sizes $m_1, m_2$, with probability at least $1-\delta$,
    $$
    \bigl|\widehat{\mathrm{FGSV}}(S_0)-\mathrm{FGSV}(S_0)\bigr| \lesssim \bar s\sqrt{\frac{\log(n/\delta)}{m_1}}  +  \alpha_0(1-\alpha_0)\sqrt{\frac{\log(n/\delta)}{m_2}} \log n + \alpha_0(1-\alpha_0)\bar s^{-\upsilon}.
    $$
    Specifically, choosing
    $$
    \bar s  \asymp  \epsilon^{-1/\upsilon}, \quad m_1  \asymp  \epsilon^{-\frac{4+2\upsilon}{\upsilon}}  \log\bigl(n/\delta\bigr), \quad m_2  \asymp  \max\Bigl\{{1,  \epsilon^{-2}(\alpha_0(1-\alpha_0))^2(\log(n/\delta))^3\Bigr\}},
    $$
    yields an $(\epsilon, \delta)$‑approximation of $\mathrm{FGSV}(S_0)$ with $O\left(n\cdot \max \left\{1, (\alpha_0(1 - \alpha_0))^2(\log n)^3\right\}\right)$ utility evaluations.
\end{theorem}

Theorem~\ref{theorem::computationalcomplexity} implies that our algorithm requires only $O(n\,\mathrm{Poly}(\log n))$ utility evaluations to achieve an $(\epsilon, \delta)$-approximation for the FGSV of a group $S_0$ whose size scales as a constant fraction of $n$.

To compare our method's computational complexity with existing works, we notice
numerous recent efficient algorithms for approximating individual Shapley values \citep{jia2019towards, wang2023note2, li2024one, lundberg2017unified, covert2021improving, musco2025provably, li2024faster, zhang2023efficient}.
They target an $(\epsilon, \delta)$-approximation on the full individual Shapley-value vector $\|\hat{\theta} - \theta\|_2$, where $\theta:=(\SV(1),\ldots,\SV(n))$ and $\hat\theta$ is the approximation for $\theta$. 
The SOTA method achieves this guarantee in $O(n\epsilon ^{-2}\log (n/\delta))$ utility evaluations. 
Approximating $\FGSV(S_0) = \sum_{i \in S_0} \theta_i$ by simply summing up individual approximations can lead to an additive error bounded by $\sqrt{s_0}\|\hat{\theta} - \theta\|_2$. Thus, achieving the same $(\epsilon, \delta)$-approximation on $\FGSV(S_0)$ requires tightening the error tolerance to $\epsilon /\sqrt{s_0}$, increasing the sample complexity to $O\big(\alpha_0 n^2 \mathrm{Poly}(\log n)\big)$.
By directly targeting the group objective, our method avoids this quadratic blow-up and offers a significant speed-up over the SOTA. 
It continues to outperform SOTA individual-based methods even when evaluating multiple groups in parallel, provided that the number of groups is $o(n)$.

\subsection{Utility function values for small input sets}
\label{sec::noninf}

In some machine learning scenarios, $U(S)$ might not be well-defined for small $|S|$.
For instance, methods such as LLM's would only produce meaningful result from sufficiently large data sets.
Meanwhile, Shapley value emphasizes the contributions from small games, therefore, we cannot simply ignore the small $|S|$ terms in the computation of FGSV \eqref{eqn::rewrite-FGSV}.
One way is to use variants of Shapley value, such as beta-Shapley and Banzhaf, that down-weight small $|S|$ terms, but these alternatives do not satisfy all axioms for faithful group data valuation per Theorem \ref{theorem::axioms}.
The other way is to fill in $U(S)$ for small $S$ with random or zero values.
However, these ad-hoc solutions lack principle and may risk significantly distorting the valuation.

To motivate our approach, we make two simple observations.
First, when $S=\emptyset$, the trained method should behave as if it were trained on pure noise, corresponding to a \emph{baseline} utility value.
Second, big-data reliant methods might not value small $S$ very differently than baseline, regardless of their content.
In this context, both quality and quantity of data matter for valuation.

Our remedy is very simple.
We set a threshold for input size, denoted by $B$, such that $U(S)$ is always well-defined and meaningful for $|S|\geq B$.  For example, in a linear regression with $p$ predictors, observing that $U(S)$ is undefined for $|S|<p$, we may set $B=cp$ for some constant $c\geq 1$.
If $|S|<B$, we inject $B-|S|$ \emph{non-informative} data points, to elevate the input size to $B$; otherwise no change is made to $U(S)$. 
 As a concrete example for the non-informative distribution, consider a supervised learning scenario with data $D = \{(x_i,y_i)\}^n_{i=1}$. 
Here, we can randomly shuffle $y_i$'s and use the resulting empirical data distribution as the non-informative distribution $\mathcal{P}_{\rm null}$.
Clearly, we can expect that as the size of informative data (i.e., $|S|$) increases, less amount of non-informative data will be injected, and $U(S)$ gradually becomes non-baseline.
For further theoretical and algorithmic details, see Appendix \ref{sec::algorithm-2}.

While our approach was inspired by the feature deletion technique in \emph{machine unlearning}, to our best knowledge, we are the first to adapt the idea for data Shapley, in a distinct fashion.

\subsection{Examples of second-order stable algorithms} \label{subsec::examples_of_second_order_stable_algorithms}

Among the conditions our theory needs, Assumption \ref{assump::utility_stability} is arguably the least intuitive.
To demonstrate that Assumption \ref{assump::utility_stability} is in fact quite mild, here, we verify that it holds for two important applications.

\subsubsection{Stochastic Gradient Descent (SGD)}
\label{subsec::our-method::SGD}

SGD on a training set $\mathcal{S}=\{z_i\}_{i=1}^s$ minimizes the empirical loss
$
    {\cal L}(w) 
    = 
    \frac1s\sum_{i=1}^s\ell(w;z_i).
$
Starting from an initial $w_0$, SGD runs for $T$ steps, each step $t\in[T]$ updating
$
    w_t 
    = 
    w_{t-1} - \alpha_t\cdot\frac1m\sum_{i\in I_t}\nabla\ell(w_{t-1};z_i)
$
with learning rate $\alpha_t$, where the mini-batch $I_t$ of size $m$ is drawn independently and uniformly from $[n]$.
Denote the final output by $w_T=w(z_{I_1},\dots,z_{I_T})$.
The performance of $w_T$ can be represented as some scalar function $u(\cdot)$ (e.g., classification accuracy on a test data).
Define the utility as the expected performance:
\begin{align}
    U(\mathcal{S})
    :=&~
    \ep_{I_1,\ldots,I_T}[u(w_T)].
    \label{eqn::utility::SGD}
\end{align}

\begin{assumption}
    \label{assump::smoothness_for_SGD}
    Suppose $\ell$ and $u$ satisfy the following regularity conditions.
    All conditions hold for some constants $C, L, \beta, \rho>0$ and all $w,w',z$.
    \begin{enumerate}
        \item (Smoothness of $u$)
        $
            \|\nabla u(w)\| \leq C
        $
        and
        $
            \|\nabla u(w') - \nabla u(w)\| \leq C \|w' - w\|.
        $
        
        \item ($L$-Lipschitz of $\ell$)
        $
            |\ell(w'; z) - \ell(w; z)| \leq L \|w' - w\|.
        $
        
        \item ($\beta$-smoothness of $\ell$)
        $
            \|\nabla_w \ell(w'; z) - \nabla_w \ell(w; z)\| \leq \beta \|w' - w\|.
        $
        
        \item (Smoothness of $\ell$'s Hessian)
        $
            \|\nabla^2_w \ell(w'; z) - \nabla^2_w \ell(w; z)\| \leq \rho \|w' - w\|.
        $
    \end{enumerate}
\end{assumption}

Assumption~\ref{assump::smoothness_for_SGD}.1 imposes mild smoothness on the utility function $u$, while the remaining assumptions concern the loss function $\ell$.  
Compared with the classical analysis of SGD stability by~\citet{hardt2016train}, the only additional requirement is Assumption~\ref{assump::smoothness_for_SGD}.4, which ensures smoothness of the Hessian. This condition is mild and is satisfied by standard smooth loss functions, such as the squared loss, when the prediction function is twice continuously differentiable in terms of parameters $w$.

\begin{proposition}
    \label{proposition::SGD_final_prop}
    Set $\alpha_t \asymp s^{-\tau_1}/t$, $m \asymp s^{\tau_2}$, and $T \asymp s^{\tau_3}$, where constants $\tau_1 > 0$ and $\tau_2, \tau_3 \geq 0$. 
    Suppose $\upsilon := 2\tau_1 + \tau_2 - 1/2 \geq 0$. 
    Under Assumption~\ref{assump::smoothness_for_SGD}, the utility function \eqref{eqn::utility::SGD} satisfies:
    \begin{equation*}
        \left| U(\mathcal{S} \cup \{z_1, z_1'\}) - U(\mathcal{S} \cup \{z_1, z_2\}) - U(\mathcal{S} \cup \{z_1', z_2'\}) + U(\mathcal{S} \cup \{z_2, z_2'\}) \right| \lesssim s^{-(3/2 + \upsilon)}.
    \end{equation*}
\end{proposition}

\subsubsection{Influence Function (IF)}

The influence function (IF) method \citep{hampel1974influence, koh2017understanding} considers the following regularized empirical loss:
\begin{align} \label{eqn::param_definition_IF}
    \hat\theta_\mathcal{S}
    :=&~
    \arg\min_\theta
    \mathcal{L}(\theta; \mathcal{S})
    := 
    \arg\min_\theta
    \Big\{
        s^{-1} \sum_{z \in \mathcal{S}} \ell(\theta; z) + (\lambda/2)\cdot \|\theta\|_2^2
    \Big\}.
\end{align}
Like in Section \ref{subsec::our-method::SGD}, suppose the utility can be written as $U(S) = u(\hat{\theta}_S)$.
Under mild conditions (see Proposition \ref{proposition::application::IF}), standard IF theory~\citep{hampel1974influence, koh2017understanding} implies that
\begin{align} 
    U(\mathcal{S}\cup \{z_1, z_1'\}) - U(\mathcal{S}) 
    \approx&~
    s^{-1}\nabla_\theta u(\hat{\theta}_\mathcal{S})^\top H_{\hat{\theta}_\mathcal{S}}^{-1} 
    \Big\{ 
        \nabla \ell(\hat{\theta}_\mathcal{S}; z_1) + \nabla \ell(\hat{\theta}_\mathcal{S}; z_1')
    \Big\}
    \label{eqn::IF_influence_functions_approx}
\end{align}
for any $z_1, z_1'$,
where \( H_{\hat{\theta}_\mathcal{S}} \) is the Hessian of ${\cal L}$ evaluated at \( \hat{\theta}_\mathcal{S} \).
This yields the following proposition.

\begin{proposition} 
    \label{proposition::application::IF}
    Suppose $u$ is continuously differentiable with $L$‑Lipschitz and bounded gradient.
    Assume that the loss function $\ell(\theta; z)$ is convex in $\theta$, three times continuously differentiable with all derivatives up to order three uniformly bounded.
    Then for any \( z_1, z_1', z_2, z_2' \in \mathcal{Z} \),
    \begin{align}
        \big| 
            U(\mathcal{S}\cup\{z_1, z_1'\}) - U(\mathcal{S}\cup\{z_1, z_2\}) - U(\mathcal{S}\cup\{z_1', z_2'\}) + U(\mathcal{S}\cup\{z_2, z_2'\}) 
        \big| 
        \lesssim s^{-2}.
        \label{eqn::application::IF}
    \end{align}
\end{proposition}

Proposition \ref{proposition::application::IF} is not surprising in view of \eqref{eqn::IF_influence_functions_approx}, as the main terms cancel out on the LHS of \eqref{eqn::application::IF}.

\section{Experiments}

We empirically compare our method (FGSV) to the approach of first computing individual Shapley values using SOTA methods and then summing them to form a group valuation.
The results demonstrate the superiority of our method in both speed and approximation accuracy and its faithfulness under shell company attack in applications such as copyright attribution in generative AI and explainable AI. 
We report the key findings in the main paper and relegate full details to Appendix~\ref{sec::exp_detail}.

\subsection{Approximation accuracy and computational efficiency}
\label{sec::benchmark}

In this experiment, we compare our method to the following benchmarks:
(1) \textbf{Permutation}-based estimator \citep{castro2009polynomial, ghorbani2019data} that averages marginal contributions over random data permutations; 
(2) \textbf{Group Testing}-based estimator \citep{jia2019towards, wang2023note2} that uses randomized group inclusion tests to estimate pairwise Shapley value differences; 
(3) \textbf{Complementary Contribution} estimator \citep{zhang2023efficient} that uses stratified sampling of complementary coalitions; 
(4) \textbf{One-for-All} estimator \citep{li2024one} that uses weighted subsets and reuses utility evaluations; 
(5) \textbf{KernelSHAP} \citep{lundberg2017unified}, a regression-based method with locally weighted samples;  
(6) \textbf{Unbiased KernelSHAP} \citep{covert2021improving}, KernelSHAP with a bias-correction; 
and (7) \textbf{LeverageSHAP} \citep{musco2025provably} that speeds up KernelSHAP by weighted sampling based on leverage scores. 

\paragraph{Experimental setup.}  We consider the Sum-of-Unanimity (SOU) cooperative game~\citep{li2024one}, where the individual Shapley values have a closed-form expression. 
Specifically, the utility function is defined as $U(S) = \sum_{j=1}^{d} \alpha_j \mathbf{1}_{\mathcal{A}_j \subseteq S}$, 
and the corresponding Shapley value for player $i \in[n]$ is given by
$\SV(i) = \sum_{j=1}^{d} \frac{\alpha_j}{|\mathcal{A}_j|} \mathbf{1}_{i \in \mathcal{A}_j}$. 
Here, each subset $\mathcal{A}_j$ is generated by sampling a random size uniformly from $1$ to $n$ and then drawing that many players without replacement. 
We set $d=n^2$ and the coefficient $\alpha_j$ as the average of the weights of all players $i\in \mathcal{A}_j$, where each player $i$ has a weight $\frac{(i \bmod 4)}{4}$. 
We partition all players into $4$ groups based on their indices:
$S_k = \big\{ i \in [n] : i \equiv k-1 (\bmod 4) \big\}$ for $k\in [4]$, and our goal is to estimate the group value $\FGSV(S_k) = \sum_{i \in S_k} \SV(i)$. 
Given a fixed budget of 20,000 utility function evaluations, we record the absolute relative error of the $\FGSV$ estimate every 200 iterations for each method.
We summarize convergence behavior using the \textit{Area Under the Convergence Curve} (AUCC), defined as:
$\operatorname{AUCC}(S_k) = \frac{1}{100} \sum_{t=1}^{100} | \frac{ \FGSV(S_k) - \widehat{\FGSV}^{(200 \cdot t)}(S_k)}{ \FGSV(S_k) } |,$
which captures both speed and stability of convergence.
We also report the \textit{Absolute Relative Error} (ARE) of the final estimate:
$
\operatorname{ARE}(S_k) = \left| \frac{ \FGSV(S_k) - \widehat{\FGSV}^{(20000)}(S_k)}{ \FGSV(S_k) } \right|.
$
Lower AUCC and ARE indicate faster convergence and better accuracy, respectively.
We record the average runtime per iteration for each method.
\begin{figure}[!t]
    \centering
    \vspace{-0.2cm}
    \includegraphics[width=\textwidth]{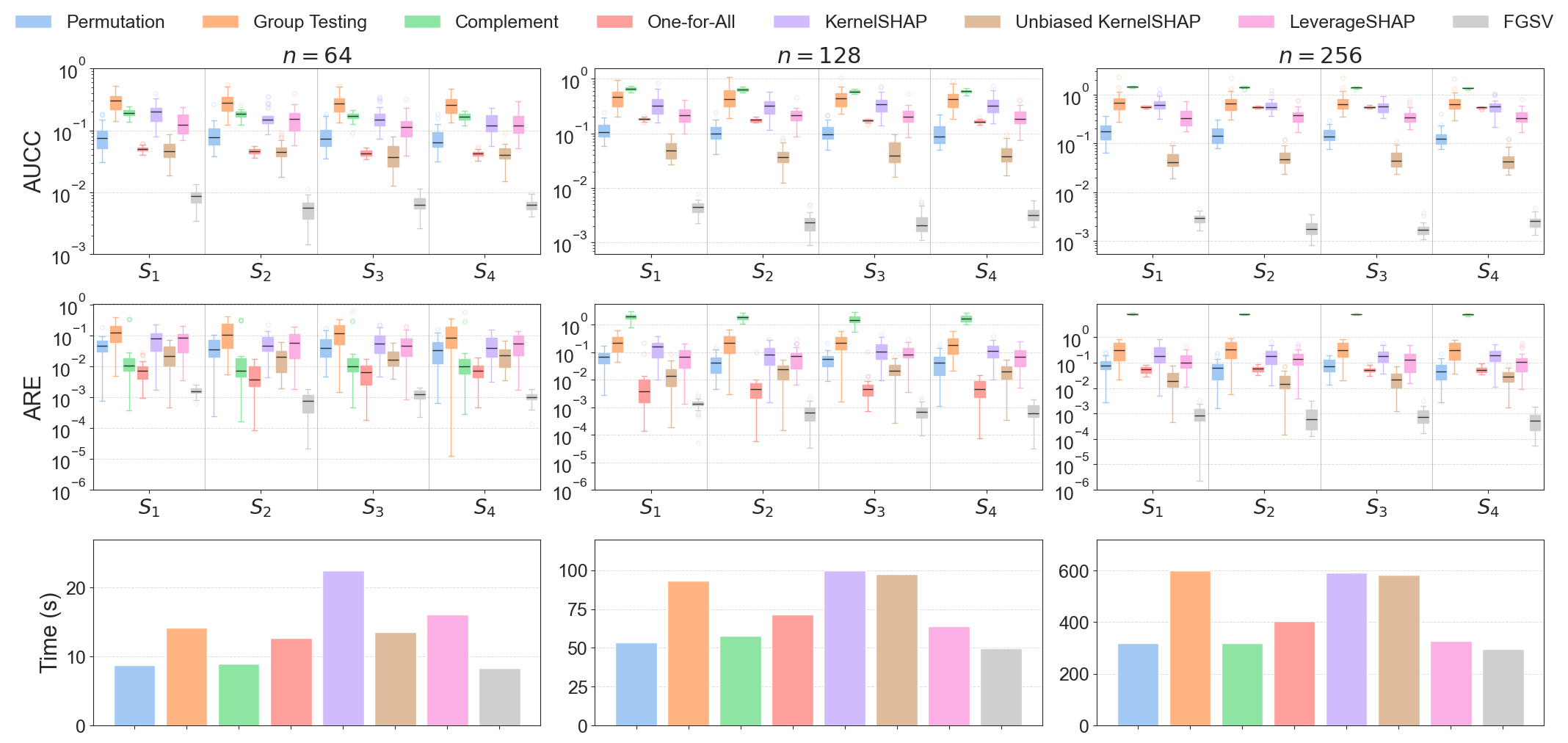}
    \vspace{-0.5cm}
    \caption{Performance comparison in the SOU game. 
    Top: Our method (FGSV) achieves the lowest AUCC and ARE across all problem sizes.
    Bottom: Our method costs the lowest runtime per iteration.}
    \label{fig:sou_results}
\end{figure}
\paragraph{Results.}
Figure~\ref{fig:sou_results} summarizes the average performance over $30$ replications for $n\in \{64, 128, 256\}$. 
The top two rows together indicate that our method shows overall superior performance across all problem sizes $n$, exhibiting both the lowest average AUCC and ARE.
The bottom row suggests that, while using the same utility evaluation budget, our method is among the fastest. 
In contrast, baselines such as Group Testing, KernelSHAP, and Unbiased KernelSHAP incur substantial higher computational overhead, due to internal optimization processes.  
Overall, our method achieves both faster convergence and improved computational efficiency, which demonstrates the benefits of directly estimating the group value rather than aggregating individual Shapley estimates.

\subsection{Application to faithful copyright attribution in generative AI}
\label{sec::copyright}

Group data valuation is important for fairly compensating copyright holders whose data are used to train generative AI models.  
Existing approaches, such as the Shapley Royalty Share (SRS) proposed by~\cite{wang2024economic}, adopt the GaI approach based on GSV. 
Given a partition of the training data into $K$ disjoint groups $S_1, \dots, S_K$, SRS is defined as $\operatorname{SRS}(S_k; x_{\mathrm{gen}}) := \frac{\GSV(S_k; x_{\mathrm{gen}}) }{\sum_{j=1}^K \GSV(S_j; x_{\mathrm{gen}}) }$. 
However, as discussed in Section~\ref{sec::background}, GSV is prone to the shell company attack. 
To address this vulnerability, we propose the Faithful Shapley Royalty Share (FSRS) that replaces GSV with our $\FGSV$ to faithfully reward individuals in a group by their contributions, not by tactical grouping strategies: 
$$
    \operatorname{FSRS}(S_k; x_{\mathrm{gen}}) := \frac{\FGSV(S_k; x_{\mathrm{gen}})}{\sum_{j=1}^K \FGSV(S_j; x_{\mathrm{gen}})}.
$$

\paragraph{Experimental setup.} Following \cite{wang2024economic}, we fine-tune Stable Diffusion v1.4~\citep{rombach2022high} using Low-Rank Adaptation (LoRA; \cite{hu2022lora}) on four brand logos from FlickrLogo-27~\citep{kalantidis2011scalable}. The utility $U(\cdot;x^{(\mathrm{gen})})$ is the average log-likelihood of generating $20$ brand-specific images  $x^{(\mathrm{gen})}$ using the prompt ``\texttt{A logo by [brand name]}’’ (see example images in Panel (a) of Figure~\ref{fig:srs-comparison}). 
We compare SRS and FSRS under two grouping scenarios: (1) $30$ images from each brand form a single group, and (2) the Google and Sprite datasets are each split into two subgroups (20/10 images), launching a shell company attack. 
Importantly, the total data per brand remains unchanged across two scenarios.

\begin{figure}[h]
    \centering
    \includegraphics[width=\linewidth]{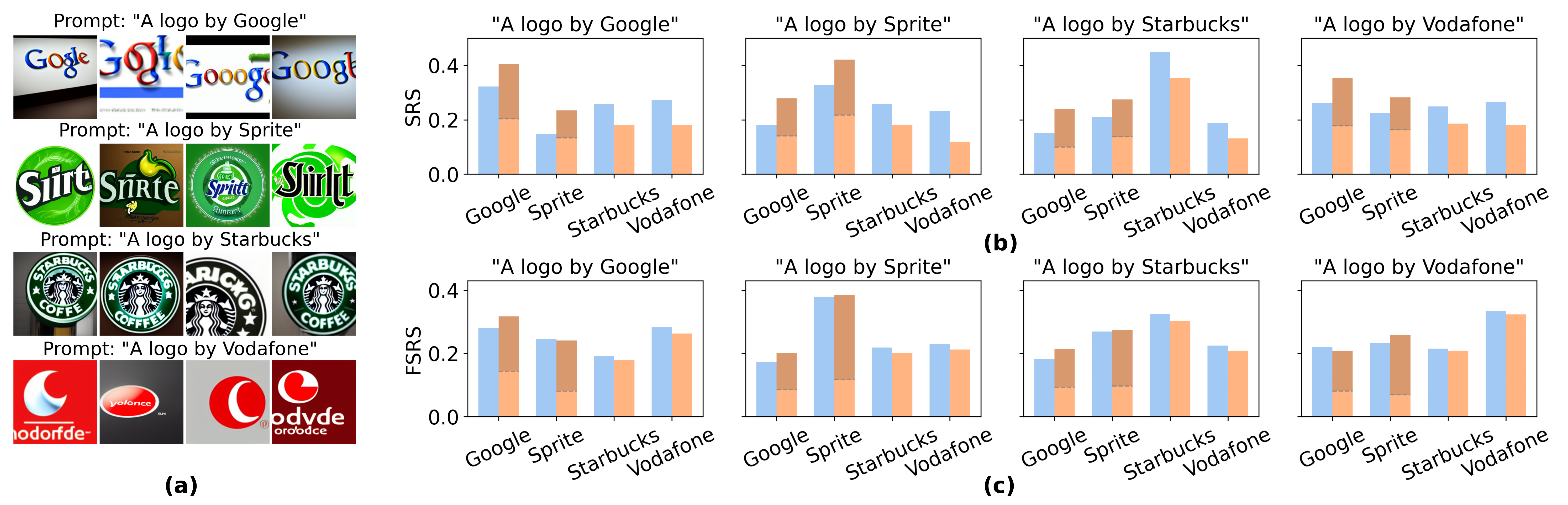}
    \vspace{-1.5em}
    \caption{
        Comparison of SRS and FSRS for copyright attribution. 
        \textbf{(a)} Example images generated using brand prompts.
        \textbf{(b)} Shapley Royalty Share (SRS, \citep{wang2024economic}) based on GSV. 
        \textbf{(c)} Faithful SRS (FSRS, our method) based on $\FGSV$.
        Blue bars: valuation under Scenario 1 (single group per brand);
        orange bars: valuation under Scenario 2 (Google/Sprite data each split into size-20/10 subgroups, colored in dark and light orange).
        }
        \vspace{-0.5em}
    \label{fig:srs-comparison}
\end{figure}

\paragraph{Results.}
Figure~\ref{fig:srs-comparison} illustrates the impact of the shell company attack on copyright attribution under SRS and FSRS. 
Each subplot corresponds to a brand-specific prompt and shows the royalty shares assigned to each brand. 
Conceptually, the brand corresponding to the prompt should receive the highest royalty share. 
Panel (b) shows that applying a shell company attack that favors the Google and Sprite groups indeed inflates their total SRS shares substantially, while reducing others', despite the contents contributed by each group remain unchanged. 
This produces misleading results---for example, Google and Sprite receive higher overall SRS than Vodafone under the ``\texttt{A logo by Vodafone}'' prompt.
In contrast, Panel~(c) shows that FSRS yields consistent and stable valuations under the shell company attack. 
This demonstrates that FSRS mitigates the effects of strategic data partitioning and provides a more faithful reflection of group contribution for copyright attribution.

\subsection{Application to faithful explainable AI}
\label{sec::explainable}

Group data valuation is also a crucial tool in explainable AI, providing interpretable summaries of data contributions at the group level--- particularly in contexts where groups correspond to socially or scientifically significant categories. 
For example, \cite{ghorbani2019data} used GSV to quantify the contributions of data from different demographic groups to patient readmission prediction accuracy \citep{strack2014impact}. 
However, GSV's sensitivity to the choice of grouping can cause group values to fluctuate dramatically under different data partitions, leading to inconsistent interpretations. 
To address this problem, we recommend practitioners to use FGSV and empirically demonstrate that it produces more consistent and reliable interpretations across a variety of grouping configurations.

\paragraph{Experimental setup.}
We conduct our experiment on the Diabetes dataset~\citep{efron2004least}, which contains 442 individuals, each described by 10 demographic and health-related features (e.g., {\it sex}, {\it age}, and {\it BMI}). 
The task is to predict the progression of diabetes one year after baseline.
We construct 7 grouping schemes, partitioned by all non-empty subsets of the variables \textit{sex}, \textit{age}, and \textit{BMI}, excluding the trivial no-partition case. 
For the continuous variables \textit{age} and \textit{BMI}, we discretize each into three quantile-based categories.
Overall, we will have between 2 to 18 total groups. 
For each grouping scheme, we compute GSV exactly and estimate FGSV via 30 Monte Carlo replications.
Our predictive model is ridge regression, and we measure utility as the \textit{negative} mean squared error on a held-out test set, with the null utility set to the variance of the test responses.
Thus, higher group valuations indicate greater contributions to the model's predictive accuracy.
To compare category-level contributions across different grouping schemes, we aggregate group values as follows: for a given variable (e.g., \textit{sex}), we sum the (F)GSVs of all groups that include that category.
This allows for consistent cross-scheme comparisons, where one would expect a category to receive stable valuations after aggregation across different grouping strategies if the valuation method is robust.

\begin{figure}[!t]
    \centering
    \vspace{-0.2cm}
    \includegraphics[width=\textwidth]{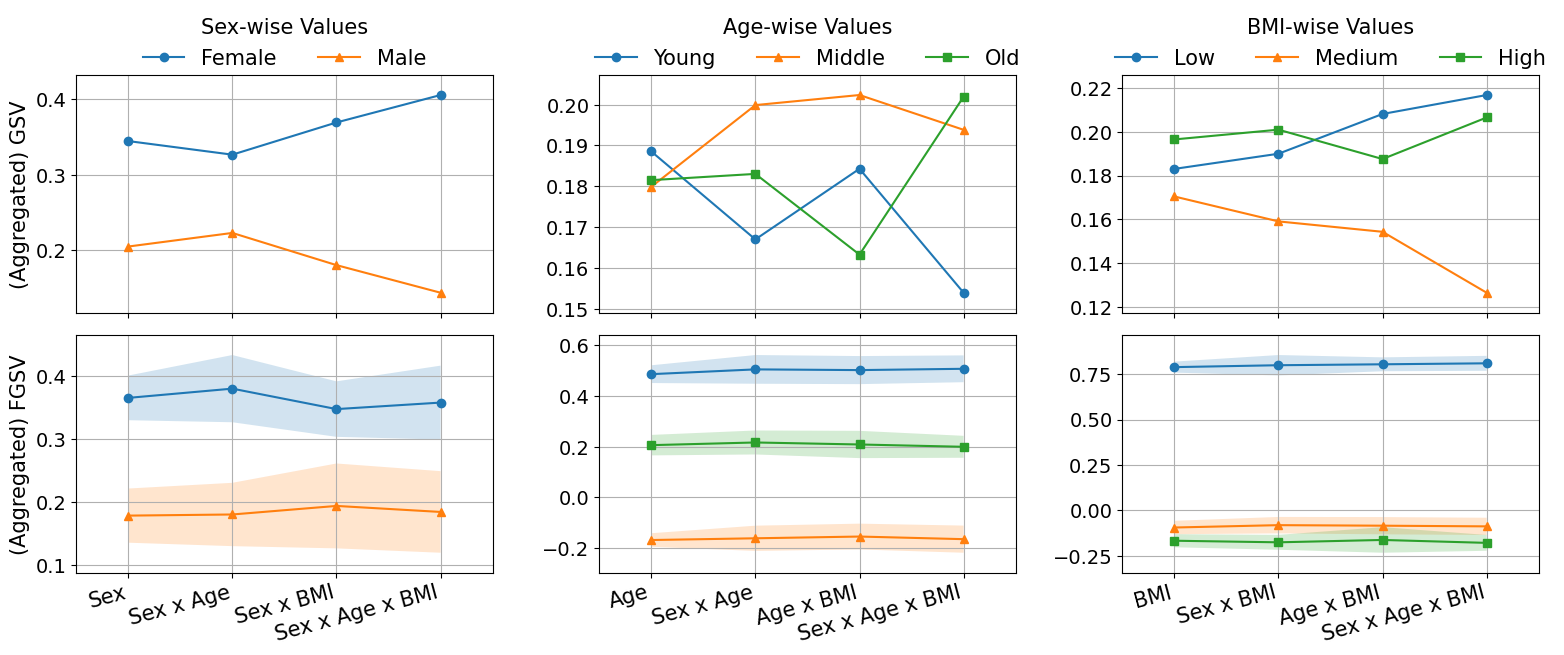}
    \vspace{-0.5cm}
    \caption{
Comparison of GSV (top row) and FGSV (bottom row) in a regression task for explainable AI. 
Each column aggregates category-level values for a specific variable: \textit{sex} (left), \textit{age} (middle), and \textit{BMI} (right). Shaded areas represent $\pm1$ standard deviation across 30 replications. 
}
    \label{fig:exp_results}
\end{figure}

\paragraph{Results.}
Figure~\ref{fig:exp_results} visualizes the aggregated category-level values from GSV (top row) and FGSV (bottom row) across various grouping schemes. 
FGSV produces significantly more stable and consistent group rankings compared to GSV.
For instance, in the \textit{age}-wise plots (second column), GSV gives inconsistent rankings: each of the three age groups---``Young,'' ``Middle,'' and ``Old''---appears as the top-ranked group under at least one grouping scheme.
This instability undermines the reliability of the interpretation.
In contrast, FGSV consistently assigns the highest value to the ``Young'' group across all grouping schemes, demonstrating its robustness.
Overall, FGSV provides more reliable explanations that are less sensitive to arbitrary grouping schemes.

\section{Conclusion and discussion}
\label{section::discussion}

In this paper, we proposed FGSV for faithful group data valuation. 
We showed that FGSV is the unique group valuation method that satisfies a desirable set of principles, including faithfulness, which ensures that group value remains unchanged to arbitrary re-grouping among other players, thereby defending against the shell company attack. 
Our algorithm also achieves lower sample complexity over SOTA methods that sum up individual Shapley value estimates, as demonstrated through both theoretical analysis and numerical experiments. 
We further illustrated the robustness of FGSV in applications to copyright attribution and explainable AI, where it faithfully reflects and fairly rewards individual contributions from group members.

Beyond the shell company attack, there exists another unfair competition strategy, namely, the {\bf copier attack}, in which a group may steal valuation from other groups by duplicating their high-value data points.
An ad-hoc remedy is to pre-process the dataset to detect and remove such duplicated entries before applying our FGSV method.
In existing literature, \cite{Falconer2025-kj} suggests designing the utility function using Pearl’s ``do operation''. 
This defense can be incorporated into our framework, letting our method also defend against the copier attack; but such do-utility function is not always available.
\cite{han2022replication} shows that with a submodular utility function, Banzhaf \citep{wang2023data} and leave-one-out can successfully discourage a genuine contributor from duplicating itself to gain higher total valuation; however, their method is prone to ``pure infringers'' who only copy from valuable data points without contributing original contents.
Overall, it remains an open challenge to defend against the copier attack for general utility functions, while maintaining the defense against the shell company attack to ensure a safe group data valuation.

\section*{Code}
The code and instructions to reproduce the experiments are provided in the supplementary material and available at \url{https://github.com/KiljaeL/Faithful_GSV}.

\section*{Acknowledgements}
Lee and Zhang were supported by NSF DMS-2311109. 
Liu and Tang were supported by NSF DMS-2412853.

\bibliographystyle{abbrvnat}
\bibliography{reference}

\clearpage

\appendix

\section{Proofs}

\etocsettocdepth{subsubsection}
\localtableofcontents

\subsection{Proof of Proposition~\ref{proposition::problems_with_group_as_individual}}
\label{sec::proof_proposition_problems_with_group_as_individual}

We first introduce some notation and preliminary results. Proposition~\ref{proposition::problems_with_group_as_individual} considers a group $S_k$ (from an initial set of $K+1$ groups $S_0, \dots, S_K$) which is split into two non-empty, disjoint subgroups, $S_k'$ and $S_k''$. We denote their cardinalities as $s_l = |S_l|$ (for $l=0, \dots, K$), $s_k' = |S_k'|$, and $s_k'' = |S_k''|$, such that $s_k = s_k' + s_k''$.
Let $\bar{U}(s)$ be the expected utility function defined in Proposition~\ref{proposition::problems_with_group_as_individual}, which depends only on the dataset size $s$.
We define an auxiliary function $\Delta(x)$ for any $x \in \mathbb{N}$ (where $x$ will typically represent the sum of cardinalities of other groups in a coalition):
\begin{align*}
    \Delta(x) := \bar{U}(x + s_{k}') + \bar{U}(x + s_{k}'') - \bar{U}(x) - \bar{U}(x + s_{k}' + s_{k}'').
\end{align*}
The proof of Proposition~\ref{proposition::problems_with_group_as_individual} will rely on the following technical lemmas. Their proofs, which build upon the prudence condition \eqref{eqn::prudence}, are deferred to Appendix~\ref{sec::lemma_proposition_problems_with_group_as_individual}.

\begin{restatable}{lemma}{LemmaPrudenceConvexDeltaU}[Prudence Implication for First Differences] \label{lemma:prudence_convex_delta_U}
    Suppose the expected utility function $\bar{U}(s)$ satisfies the prudence condition \eqref{eqn::prudence}.
    Let $\Delta \bar{U}(t) := \bar{U}(t+1) - \bar{U}(t)$ for $t \ge 0$. Then, for any integer $x \ge 0$ and positive integers $a, b \in \mathbb{N}_+$, it holds that
    \begin{equation*}
        \Delta \bar{U}(x) + \Delta \bar{U}(x + a + b) > \Delta \bar{U}(x + a) + \Delta \bar{U}(x + b).
    \end{equation*}
    (This property indicates that the first-order difference $\Delta \bar{U}(x)$ is a strictly convex function.)
\end{restatable}

\begin{restatable}{lemma}{LemmaDeltaXStrictlyDecreasing}[Strict Monotonicity of $\Delta(x)$] \label{lemma:delta_x_strictly_decreasing}
    If $\bar{U}(s)$ satisfies the prudence condition \eqref{eqn::prudence} (which implies Lemma~\ref{lemma:prudence_convex_delta_U}), then $\Delta(x)$ is a strictly decreasing function of $x$. That is, for any $x_2 > x_1 \ge 0$, we have $\Delta(x_1) > \Delta(x_2)$.
\end{restatable}

\begin{restatable}{lemma}{LemmaAlgebraicSumIdentity}[Sum Symmetrization Identity] \label{lemma:algebraic_sum_identity_prop1}
    Let $[n] := \{1, \dots, n\}$ be the set of player indices, and let $f: 2^{[n]} \to \mathbb{R}$ be a function that assigns a real value to each subset of $[n]$. Then, for any integer $m$ such that $0 \leq m \leq \lfloor n/2 \rfloor$, the following identity holds:
    \begin{equation*}
        \sum_{\substack{I \subseteq [n] \\ |I| = m}} f(I) + \sum_{\substack{I \subseteq [n] \\ |I| = n - m}} f(I)
        = \frac{1}{\binom{n - m}{m}} \sum_{\substack{I \subseteq [n] \\ |I| = m}} \sum_{\substack{J \subseteq [n] \setminus I \\ |J| = n - 2m}} \left[ f(I) + f(I \cup J) \right].
    \end{equation*}
\end{restatable}

We now proceed to prove Proposition~\ref{proposition::problems_with_group_as_individual}.

\begin{proof}[Proof of Proposition~\ref{proposition::problems_with_group_as_individual}]
Let $N_0 = \{0, \dots, K\}$ be the set of indices for the initial $K+1$ groups $S_0, \dots, S_K$. 
The Group Shapley Value for $S_k$ can be re-expressed in terms of $\bar U$, as follows:
\begin{align}
    \mathbb{E}[\GSV(S_k)]
    &= \sum_{I^* \subseteq N_0\setminus\{k\}} \frac{|I^*|!(K-|I^*|)!}{(K+1)!} \left[ \bar{U}\left( \sum_{l \in I^*} s_{l} + s_k \right) - \bar{U}\left( \sum_{l \in I^*} s_{l} \right) \right] \nonumber \\
    &= \frac{1}{K+1} \sum_{m^*=0}^{K} \frac{1}{\binom{K}{m^*}} \sum_{\substack{I^* \subseteq N_0\setminus\{k\} \\ |I^*|=m^*}} \left[ \bar{U}\left( \sum_{l \in I^*} s_{l} + s_k \right) - \bar{U}\left( \sum_{l \in I^*} s_{l} \right) \right]. \label{eq:gsv_sk_expanded}
\end{align}
After splitting $S_k$ into $S_k'$ and $S_k''$, we have $K+2$ groups. Let $N' = (N_0\setminus\{k\}) \cup \{k', k''\}$ be the set of indices for these $K+2$ groups, where $k'$ and $k''$ index $S_k'$ and $S_k''$ respectively. The GSVs for $S_k'$ and $S_k''$ are:
\begin{align*}
    \mathbb{E}[\GSV(S_k')]
    &= \sum_{I \subseteq N'\setminus\{k'\}} \frac{|I|!(K+1-|I|)!}{(K+2)!} \left[ \bar{U}\left( \sum_{l \in I} s_{l} + s_k' \right) - \bar{U}\left( \sum_{l \in I} s_{l} \right) \right]\\
    &= \sum_{m^*=0}^{K+1}  \frac{1}{K+2}\frac{1}{\binom{K+1}{m^*}} \sum_{\substack{I^* \subseteq N'\setminus\{k'\} \\ |I^*|=m^*}} \left[ \bar{U}\left( \sum_{l \in I^*} s_{l} + s_k' \right) - \bar{U}\left( \sum_{l \in I^*} s_{l} \right) \right], \\
    \mathbb{E}[\GSV(S_k'')]
    &= \sum_{I \subseteq N'\setminus\{k''\}} \frac{|I|!(K+1-|I|)!}{(K+2)!} \left[ \bar{U}\left( \sum_{l \in I} s_{l} + s_k'' \right) - \bar{U}\left( \sum_{l \in I} s_{l} \right) \right]\\
    &= \sum_{m^*=0}^{K+1}  \frac{1}{K+2}\frac{1}{\binom{K+1}{m^*}} \sum_{\substack{I^* \subseteq N'\setminus\{k''\} \\ |I^*|=m^*}} \left[ \bar{U}\left( \sum_{l \in I^*} s_{l} + s_k'' \right) - \bar{U}\left( \sum_{l \in I^*} s_{l} \right) \right].
\end{align*}

Let $I_0 = N_0 \setminus \{k\}$ be the set of $K$ "other" groups.
To analyze $\mathbb{E}[\text{GSV}(S_k')] + \mathbb{E}[\text{GSV}(S_k'')]$, we examine the terms associated with each $I^* \subseteq I_0$. Let $m^* := |I^*|$.
For a given $I^*$, its coalition can be joined by $S_k'$ alone, $S_k''$ alone, or by $S_k'$ and $S_k''$ in either order.
Collecting the marginal contributions for $S_k'$ and $S_k''$ across these scenarios for a fixed $I^*$, we obtain:
\begin{align*}
    & \frac{1}{K+2} \left\{ \frac{1}{\binom{K+1}{m^*}} \left[ \bar{U}\left(\sum_{l\in I^*}s_l + s_k'\right) - \bar{U}\left(\sum_{l\in I^*}s_l\right) + \bar{U}\left(\sum_{l\in I^*}s_l + s_k''\right) - \bar{U}\left(\sum_{l\in I^*}s_l\right) \right] \right. \\
    & \left. + \frac{1}{\binom{K+1}{m^*+1}} \left[ \bar{U}\left(\sum_{l\in I^*}s_l + s_k'' + s_k'\right) - \bar{U}\left(\sum_{l\in I^*}s_l + s_k''\right) + \bar{U}\left(\sum_{l\in I^*}s_l + s_k' + s_k''\right) - \bar{U}\left(\sum_{l\in I^*}s_l + s_k'\right) \right] \right\}.
\end{align*}
Summing over all $I^* \subseteq I_0$, the total expected value is:
\begin{align*}
    &\mathbb{E}[\text{GSV}(S_k')] + \mathbb{E}[\text{GSV}(S_k'')] \\
    & = \sum_{I^* \subseteq I_0} \frac{1}{K+2} \left\{ \frac{1}{\binom{K+1}{|I^*|}} \left[ \bar{U}\left(\sum_{l\in I^*}s_l + s_k'\right) - \bar{U}\left(\sum_{l\in I^*}s_l\right) + \bar{U}\left(\sum_{l\in I^*}s_l + s_k''\right) - \bar{U}\left(\sum_{l\in I^*}s_l\right) \right] \right. \\
    &\quad \left. + \frac{1}{\binom{K+1}{|I^*|+1}} \left[ \bar{U}\left(\sum_{l\in I^*}s_l + s_k'' + s_k'\right) - \bar{U}\left(\sum_{l\in I^*}s_l + s_k''\right) + \bar{U}\left(\sum_{l\in I^*}s_l + s_k' + s_k''\right) - \bar{U}\left(\sum_{l\in I^*}s_l + s_k'\right) \right] \right\}.
\end{align*}

Let $s_{I^*} = \sum_{l \in I^*} s_l$. 
The above form
can be simplified into:
\begin{align*}
    &\mathbb{E}[\GSV(S_k')] + \mathbb{E}[\GSV(S_k'')] \\
    & = \sum_{I^* \subseteq I_0} \frac{1}{K+2} \left\{ \frac{1}{\binom{K+1}{|I^*|}} \left[ \bar{U}(s_{I^*} + s_k') + \bar{U}(s_{I^*} + s_k'') - 2\bar{U}(s_{I^*}) \right] \right. \\
    &\quad \quad \quad \quad \quad \quad  \left. \quad + \frac{1}{\binom{K+1}{|I^*|+1}} \left[ 2\bar{U}(s_{I^*} + s_k) - \bar{U}(s_{I^*} + s_k'') - \bar{U}(s_{I^*} + s_k') \right] \right\}.
\end{align*}

For $\mathbb{E}[\GSV(S_k)]$ in \Cref{eq:gsv_sk_expanded}, using the identity $\frac{1}{K+1}\frac{1}{\binom{K}{m^*}} = \frac{1}{K+2}\left(\frac{1}{\binom{K+1}{m^*}} + \frac{1}{\binom{K+1}{m^*+1}}\right)$, we have:
\begin{align*}
    \mathbb{E}[\GSV(S_k)] &= \sum_{I^* \subseteq I_0} \frac{1}{K+1}\frac{1}{\binom{K}{|I^*|}} \left[ \bar{U}(s_{I^*} + s_k) - \bar{U}(s_{I^*}) \right] \\
    &= \sum_{I^* \subseteq I_0} \frac{1}{K+2} \left(\frac{1}{\binom{K+1}{|I^*|}} + \frac{1}{\binom{K+1}{|I^*|+1}}\right) \left[ \bar{U}(s_{I^*} + s_k) - \bar{U}(s_{I^*}) \right].
\end{align*}
Then, the difference in the expected GSVs is:
\begin{align*}
    &\mathbb{E}[\text{GSV}(S_k')] + \mathbb{E}[\text{GSV}(S_k'')] - \mathbb{E}[\text{GSV}(S_k)]  \\
    & = \sum_{I^* \subseteq I_0} \frac{1}{K+2} \left\{ \frac{1}{\binom{K+1}{|I^*|}} \left[ \bar{U}(s_{I^*} + s_k') + \bar{U}(s_{I^*} + s_k'') - 2\bar{U}(s_{I^*}) - \bar{U}(s_{I^*} + s_k) + \bar{U}(s_{I^*})\right] \right. \\
    & \quad \quad \quad \quad \quad \quad \left. \quad + \frac{1}{\binom{K+1}{|I^*|+1}} \left[ 2\bar{U}(s_{I^*} + s_k) - \bar{U}(s_{I^*} + s_k'') - \bar{U}(s_{I^*} + s_k') - \bar{U}(s_{I^*} + s_k) + \bar{U}(s_{I^*}) \right] \right\} \\
   & = \sum_{I^* \subseteq I_0} \frac{1}{K+2} \left( \frac{1}{\binom{K+1}{m^*}} - \frac{1}{\binom{K+1}{m^*+1}} \right) \Delta(s_{I^*}).
\end{align*}
We further group the terms with the same cardinality $m^* = |I^*|$, which ranges from $0$ to $K$. For each~$m^*$, there are $\binom{K}{m^*}$ such coalitions $I^*$. This yields:
\begin{align}
    &\mathbb{E}[\text{GSV}(S_k')] + \mathbb{E}[\text{GSV}(S_k'')] - \mathbb{E}[\text{GSV}(S_k)] \nonumber \\
    &= \frac{1}{K+2} \sum_{m^*=0}^{K} \binom{K}{m^*} \left( \frac{1}{\binom{K+1}{m^*}} - \frac{1}{\binom{K+1}{m^*+1}} \right) \frac{\sum_{I^*:|I^*|=m^*}\Delta(s_{I^*})}{\binom{K}{m^*}} \nonumber \\
    &= \frac{1}{(K+1)(K+2)} \sum_{m^*=0}^{K} (K - 2m^*) \frac{\sum_{I^*:|I^*|=m^*}\Delta(s_{I^*})}{\binom{K}{m^*}}, \label{eq:diff_sum_final_form}
\end{align}
where the last step uses the fact $\binom{K}{m^*} \left( \frac{1}{\binom{K+1}{m^*}} - \frac{1}{\binom{K+1}{m^*+1}} \right) = \frac{K-2m^*}{K+1}$.

To determine the sign of the RHS of \eqref{eq:diff_sum_final_form}, we define an auxiliary function 
$$
    g(I^*) := \frac{K-2|I^*|}{\binom{K}{|I^*|}} \Delta(s_{I^*})
$$ 
for all ${I^*} \subseteq I_0$. 
We have
\begin{align}
    \mathbb{E}[\GSV(S_k')] + \mathbb{E}[\GSV(S_k'')] - \mathbb{E}[\GSV(S_k)] = \frac{1}{(K+1)(K+2)} \sum_{{I^*} \subseteq I_0} g({I^*}). \label{eq:D_total_sum_gX}
\end{align}
We can rewrite the sum $\sum_{{I^*} \subseteq I_0} g({I^*})$ by pairing up terms corresponding to coalitions of sizes $m^*$ and $K-m^*$, for each $m^* = 0, \dots, \lfloor (K-1)/2 \rfloor$.
We have
\begin{align*}
    \sum_{{I^*} \subseteq I_0} g({I^*}) = \sum_{m^*=0}^{\lfloor (K-1)/2 \rfloor} \left( \sum_{\substack{I^* \subseteq I_0 \\ |I^*|=m^*} } g(I^*) + \sum_{\substack{J^* \subseteq I_0 \\ |J^*|=K-m^*} } g(J^*) \right).
\end{align*}
(If $K$ is even, the middle term where $m^*=K/2$ has $K-2m^*=0$, so $g({I^*})=0$ for $|{I^*}|=K/2$.)
Applying Lemma~\ref{lemma:algebraic_sum_identity_prop1} to each parenthesized pair of sums:
\begin{align*}
    \sum_{\substack{I^* \subseteq I_0 \\ |I^*|=m^*} } g(I^*) + \sum_{\substack{J^* \subseteq I_0 \\ |J^*|=K-m^*} } g(J^*) = \frac{1}{\binom{K-m^*}{m^*}} \sum_{\substack{I^* \subseteq I_0 \\ |I^*|=m^*} } \sum_{\substack{L^* \subseteq I_0 \setminus I^* \\ |L^*|=K-2m^*} } \left[ g(I^*) + g(I^* \cup L^*) \right].
\end{align*}
Substituting $g(I^*) = \frac{K-2m^*}{\binom{K}{m^*}} \Delta(s_{I^*})$ and noting that $|I^* \cup L^*| = K-m^*$, we have $g(I^* \cup L^*) = \frac{K-2(K-m^*)}{\binom{K}{K - m^*}} \Delta(s_{I^* \cup L^*}) = -\frac{(K-2m^*)}{\binom{K}{m^*}} \Delta(s_{I^* \cup L^*})$.
Thus, the term $[g(I^*) + g(I^* \cup L^*)]$ becomes $\frac{K-2m^*}{\binom{K}{m^*}} \left[ \Delta(s_{I^*}) - \Delta(s_{I^* \cup L^*}) \right]$.
So, $\mathbb{E}[\GSV(S_k')] + \mathbb{E}[\GSV(S_k'')] - \mathbb{E}[\GSV(S_k)]$ is:
\begin{align*}
    \frac{1}{(K+1)(K+2)} \sum_{m^*=0}^{\lfloor (K-1)/2 \rfloor} \frac{1}{\binom{K-m^*}{m^*}} \frac{K - 2m^*}{\binom{K}{m^*}} \sum_{\substack{I^* \subseteq I_0 \\ |I^*|=m^*} } \sum_{\substack{L^* \subseteq I_0 \setminus I^* \\ |L^*|=K-2m^*} } \left[ \Delta(s_{I^*}) - \Delta(s_{I^* \cup L^*}) \right].
\end{align*}
For $m^* < K/2$, we have:
\begin{enumerate}
    \item The coefficient $\frac{K-2m^*}{(K+1)(K+2)\binom{K-m^*}{m^*}\binom{K}{m^*}}$ is strictly positive.
    \item The set $L^*$ is non-empty (since $K-2m^* > 0$). Assuming at least some $s_l > 0$ for $l \in L^*$, we have $s_{I^* \cup L^*} > s_{I^*}$.
    \item By Lemma~\ref{lemma:delta_x_strictly_decreasing}, we have $\Delta(s_{I^*}) > \Delta(s_{I^* \cup L^*})$.
\end{enumerate}
Combining these observations shows $\ep[\GSV(S_k')] + \ep[\GSV(S_k'')] > \ep[\GSV(S_k)]$ and completes the proof of Proposition~\ref{proposition::problems_with_group_as_individual}.

\end{proof}

\subsubsection{Technical lemmas used in the proof of Proposition~\ref{proposition::problems_with_group_as_individual}} \label{sec::lemma_proposition_problems_with_group_as_individual}

\LemmaPrudenceConvexDeltaU*

\begin{proof}[Proof of Lemma~\ref{lemma:prudence_convex_delta_U}]
Define the second-order difference as $\Delta^2 \bar{U}(t) := \Delta \bar{U}(t+1) - \Delta \bar{U}(t)$. 
The prudence condition on $\bar{U}(s)$ is given by \eqref{eqn::prudence}:
$ \Delta^3 \bar{U}(x) := \bar{U}(x+3) - 3\bar{U}(x+2) + 3\bar{U}(x+1) - \bar{U}(x) > 0 $
for any valid $x$. We can express $\Delta^3 \bar{U}(x)$ as follows:
\begin{align*}
    \Delta^3 \bar{U}(x) &= \left( \bar{U}(x+3) - \bar{U}(x+2) \right) - 2\left( \bar{U}(x+2) - \bar{U}(x+1) \right) + \left( \bar{U}(x+1) - \bar{U}(x) \right) \\
    &= \Delta \bar{U}(x+2) - 2\Delta \bar{U}(x+1) + \Delta \bar{U}(x) \\
    &= \Delta^2\bar{U}(x+1) - \Delta^2\bar{U}(x).
\end{align*}
Since $\Delta^3 \bar{U}(x) > 0$, we have $\Delta^2 \bar{U}(x+1) > \Delta^2 \bar{U}(x)$. As this holds for any valid $x$, it implies that the function $\Delta^2 \bar{U}(t)$ is strictly increasing in $t$.

We want to prove that for any integer $x \ge 0$ and positive integers $a, b \in \mathbb{N}_+$:
\begin{equation*}
    \Delta \bar{U}(x) + \Delta \bar{U}(x + a + b) > \Delta \bar{U}(x + a) + \Delta \bar{U}(x + b).
\end{equation*}
This inequality can be rearranged to:
\begin{equation}
    \Delta \bar{U}(x + a + b) - \Delta \bar{U}(x + a) > \Delta \bar{U}(x + b) - \Delta \bar{U}(x). \label{eq:lemma_target_rearranged_proof}
\end{equation}
The left-hand side (LHS) of \eqref{eq:lemma_target_rearranged_proof} can be written as a sum of second-order differences:
\begin{align*}
    \Delta \bar{U}(x + a + b) - \Delta \bar{U}(x + a) &= \sum_{j=0}^{b-1} \left( \Delta \bar{U}(x+a+j+1) - \Delta \bar{U}(x+a+j) \right) \\
    &= \sum_{j=0}^{b-1} \Delta^2 \bar{U}(x+a+j).
\end{align*}
Similarly, the right-hand side (RHS) of \eqref{eq:lemma_target_rearranged_proof} is:
\begin{align*}
    \Delta \bar{U}(x + b) - \Delta \bar{U}(x) &= \sum_{j=0}^{b-1} \left( \Delta \bar{U}(x+j+1) - \Delta \bar{U}(x+j) \right) \\
    &= \sum_{j=0}^{b-1} \Delta^2 \bar{U}(x+j).
\end{align*}
Since $a \in \mathbb{N}_+$, we have $a \ge 1$. Therefore, for each $j \in \{0, \dots, b-1\}$, the argument $x+a+j$ is strictly greater than $x+j$.
Because $\Delta^2 \bar{U}(t)$ is a strictly increasing function of $t$, it follows that for each term in the summations:
\begin{equation*}
    \Delta^2 \bar{U}(x+a+j) > \Delta^2 \bar{U}(x+j).
\end{equation*}
Summing these strict inequalities from $j=0$ to $b-1$ (since $b \in \mathbb{N}_+$, there is at least one term in the sum):
\begin{equation*}
    \sum_{j=0}^{b-1} \Delta^2 \bar{U}(x+a+j) > \sum_{j=0}^{b-1} \Delta^2 \bar{U}(x+j).
\end{equation*}
This directly implies that the LHS of \eqref{eq:lemma_target_rearranged_proof} is strictly greater than its RHS:
\begin{equation*}
    \Delta \bar{U}(x + a + b) - \Delta \bar{U}(x + a) > \Delta \bar{U}(x + b) - \Delta \bar{U}(x).
\end{equation*}
Rearranging this yields the desired result:
\begin{equation*}
    \Delta \bar{U}(x) + \Delta \bar{U}(x + a + b) > \Delta \bar{U}(x + a) + \Delta \bar{U}(x + b).
\end{equation*}
This completes the proof.
\end{proof}

\LemmaDeltaXStrictlyDecreasing*

\begin{proof}[Proof of Lemma~\ref{lemma:delta_x_strictly_decreasing}]
We want to show that $\Delta(x)$ is a strictly decreasing function of $x$. This is equivalent to proving that $\Delta(x) - \Delta(x+1) > 0$ for any $x \ge 0$ (assuming $x$ and $x+1$ lead to valid arguments for $\bar{U}$ as per the lemma statement).

Let the first-order difference of $\bar{U}$ be defined as $\Delta \bar{U}(t) := \bar{U}(t+1) - \bar{U}(t)$, for $t \ge 0$, ensuring all arguments to $\bar{U}$ are valid.
Recall the definition of $\Delta(x)$:
\begin{equation*}
    \Delta(x) = \bar{U}(x + s_{k}') + \bar{U}(x + s_{k}'') - \bar{U}(x) - \bar{U}(x + s_{k}' + s_{k}'').
\end{equation*}
Then, the difference $\Delta(x) - \Delta(x+1)$ is:
\begin{align*}
    &\Delta(x) - \Delta(x+1) \\
    &= \left[ \bar{U}(x + s_{k}') + \bar{U}(x + s_{k}'') - \bar{U}(x) - \bar{U}(x + s_{k}' + s_{k}'') \right] \\
    &\quad - \left[ \bar{U}(x+1 + s_{k}') + \bar{U}(x+1 + s_{k}'') - \bar{U}(x+1) - \bar{U}(x+1 + s_{k}' + s_{k}'') \right].
\end{align*}
We regroup the terms and leverage the definition of $\Delta \bar{U}(t)$, this becomes:
\begin{align*}
    \Delta(x) - \Delta(x+1) = \Delta \bar{U}(x) + \Delta \bar{U}(x + s_{k}' + s_{k}'') - \left[ \Delta \bar{U}(x + s_{k}') + \Delta \bar{U}(x + s_{k}'') \right].
\end{align*}
Since $S_k'$ and $S_k''$ are non-empty, $s_k'$ and $s_k''$ are positive integers (i.e., $a, b \in \mathbb{N}_+$). By Lemma~\ref{lemma:prudence_convex_delta_U}, we have:
\begin{equation*}
    \Delta \bar{U}(x) + \Delta \bar{U}(x + s_{k}' + s_{k}'') > \Delta \bar{U}(x + s_{k}') + \Delta \bar{U}(x + s_{k}'').
\end{equation*}
Therefore,
\begin{equation*}
    \Delta \bar{U}(x) + \Delta \bar{U}(x + s_{k}' + s_{k}'') - \left[ \Delta \bar{U}(x + s_{k}') + \Delta \bar{U}(x + s_{k}'') \right] > 0.
\end{equation*}
This implies $\Delta(x) - \Delta(x+1) > 0$, so $\Delta(x) > \Delta(x+1)$.
Since this holds for any valid $x \ge 0$, $\Delta(x)$ is a strictly decreasing function of $x$. Consequently, for any $x_2 > x_1 \ge 0$, we have $\Delta(x_1) > \Delta(x_2)$.
\end{proof}

\LemmaAlgebraicSumIdentity*

\begin{proof}[Proof of Lemma~\ref{lemma:algebraic_sum_identity_prop1}]
We start by analyzing the right-hand side (RHS) of the proposed identity:
\begin{equation*}
    \text{RHS} = \frac{1}{\binom{n - m}{m}} \sum_{\substack{I \subseteq [n] \\ |I| = m}} \sum_{\substack{J \subseteq [n] \setminus I \\ |J| = n - 2m}} \left[ f(I) + f(I \cup J) \right].
\end{equation*}
We can split the sum inside the square brackets into two parts:
\begin{equation*}
    \text{RHS} = \frac{1}{\binom{n - m}{m}} \left[ \sum_{\substack{I \subseteq [n] \\ |I| = m}} \sum_{\substack{J \subseteq [n] \setminus I \\ |J| = n - 2m}} f(I) + \sum_{\substack{I \subseteq [n] \\ |I| = m}} \sum_{\substack{J \subseteq [n] \setminus I \\ |J| = n - 2m}} f(I \cup J) \right].
\end{equation*}

Let's analyze the first double summation: $\sum_{\substack{I \subseteq [n] \\ |I| = m}} \sum_{\substack{J \subseteq [n] \setminus I \\ |J| = n - 2m}} f(I)$.
For any fixed subset $I$ such that $|I|=m$, the term $f(I)$ is constant with respect to the inner sum over $J$. The number of ways to choose a subset $J \subseteq [n] \setminus I$ with $|J| = n - 2m$ is given by $\binom{|[n] \setminus I|}{n - 2m}$. Since $|[n] \setminus I| = n - m$, this count is $\binom{n-m}{n-2m}$.
Using the identity $\binom{k}{r} = \binom{k}{k-r}$, we have $\binom{n-m}{n-2m} = \binom{n-m}{(n-m) - (n-2m)} = \binom{n-m}{m}$.
So, for each fixed $I$ with $|I|=m$, $f(I)$ appears $\binom{n-m}{m}$ times.
Thus, the first double summation is:
\begin{equation*}
    \sum_{\substack{I \subseteq [n] \\ |I| = m}} \binom{n-m}{m} f(I).
\end{equation*}

Now, let's analyze the second double summation: $\sum_{\substack{I \subseteq [n] \\ |I| = m}} \sum_{\substack{J \subseteq [n] \setminus I \\ |J| = n - 2m}} f(I \cup J)$.
Let $K = I \cup J$. Since $I \cap J = \emptyset$, $|I|=m$, and $|J|=n-2m$, it follows that $|K| = m + (n-2m) = n-m$.
We want to count how many times a specific subset $K \subseteq [n]$ with $|K|=n-m$ appears as $I \cup J$ in this summation.
For a fixed $K$ (where $|K|=n-m$), we need to find an $I \subseteq K$ such that $|I|=m$. Once such an $I$ is chosen, $J$ is uniquely determined as $J = K \setminus I$. The size of this $J$ will be $|K \setminus I| = (n-m) - m = n-2m$. Also, $J \subseteq [n] \setminus I$ is satisfied.
The number of ways to choose such a subset $I \subseteq K$ with $|I|=m$ is $\binom{|K|}{m} = \binom{n-m}{m}$.
So, for each fixed $K$ with $|K|=n-m$, the term $f(K)$ appears $\binom{n-m}{m}$ times.
Thus, the second double summation is:
\begin{equation*}
    \sum_{\substack{K \subseteq [n] \\ |K| = n - m}} \binom{n-m}{m} f(K).
\end{equation*}
(Renaming the dummy variable $K$ to $I$ for consistency with the LHS):
\begin{equation*}
    \sum_{\substack{I \subseteq [n] \\ |I| = n - m}} \binom{n-m}{m} f(I).
\end{equation*}

Substituting these back into the expression for the RHS:
\begin{align*}
    \text{RHS} &= \frac{1}{\binom{n - m}{m}} \left[ \sum_{\substack{I \subseteq [n] \\ |I| = m}} \binom{n-m}{m} f(I) + \sum_{\substack{I \subseteq [n] \\ |I| = n - m}} \binom{n-m}{m} f(I) \right] \\
    &= \sum_{\substack{I \subseteq [n] \\ |I| = m}} f(I) + \sum_{\substack{I \subseteq [n] \\ |I| = n - m}} f(I).
\end{align*}
This is exactly the left-hand side (LHS) of the identity stated in the lemma. The condition $0 \leq m \leq \lfloor n/2 \rfloor$ ensures that $n-2m \ge 0$ (so $|J|$ is a valid size) and $m \le n-m$ (so $\binom{n-m}{m}$ is well-defined and typically non-zero, unless $n-m < m$ which is prevented by $m \le n/2$. If $m=n/2$, $\binom{n-m}{m}=\binom{m}{m}=1$).
This completes the proof.
\end{proof}

\subsection{Proof of 
Theorem~\ref{theorem::axioms}}

\begin{proof}[Proof of Theorem~\ref{theorem::axioms}]
Let $[n] = \{1, \dots, n\}$ be the set of indices for the data points. The utility function $U$ is defined on subsets of $[n]$, and $\Pi$ is a partition of $[n]$.
The proof consists of two parts: sufficiency and uniqueness.

\textbf{Sufficiency.}
We verify that the proposed valuation $\nu_{U, \mathcal{D}, \Pi}(S) = \sum_{i \in S} \SV(i)$ satisfies all five axioms from Definition~\ref{definition::axioms}. For any $S_j \in \Pi$:
\begin{enumerate}
    \item \textbf{Null player:} If every subset $S' \subseteq S_j$ satisfies $U(S'' \cup S') = U(S'')$ for all $S'' \subseteq [n] \setminus S'$, then as a special case, for every element $k \in S_j$, it holds that $U(S'' \cup \{k\}) = U(S'')$ for all $S'' \subseteq [n] \setminus \{k\}$.
    By the definition of individual Shapley value, we have $\SV(k) = 0$. Thus, $\nu_{U, \mathcal{D}, \Pi}(S_j) = \sum_{k \in S_j} \SV(k) = 0$.
    
    \item \textbf{Symmetry:} 
    Let $S_1,S_2$ be the said sets and $\sigma$ be the said mapping between them, all as described in the axiom.
    For every individual $k \in S_1$, we take $S'=\{k\}$.
    By the axiom, for
    all $S'' \subset [n] \setminus \{k, \sigma(k)\}$, it holds that $U(S'' \cup \{k\}) = U(S'' \cup \{\sigma(k)\})$.
    Then it is easy to verify that $\SV(k) = \SV(\sigma(k))$.
    Therefore, $\nu_{U, \mathcal{D}, \Pi}(S_1) = \sum_{k \in S_1} \SV(k) = \sum_{k \in S_1} \SV(\sigma(k)) = \sum_{l \in S_2} \SV(l) = \nu_{U, \mathcal{D}, \Pi}(S_2)$.
    \item \textbf{Linearity:} For utility functions $U_1, U_2$ and scalars $\alpha_1, \alpha_2$, the individual Shapley value is linear: $\SV_{\alpha_1 U_1 + \alpha_2 U_2}(k) = \alpha_1 \SV_{U_1}(k) + \alpha_2 \SV_{U_2}(k)$. Summing over $k \in S_j$ preserves this linearity:
    $
    \nu_{\alpha_1 U_1 + \alpha_2 U_2, \mathcal{D}, \Pi}(S_j) = \sum_{k \in S_j} \SV_{\alpha_1 U_1 + \alpha_2 U_2}(k) 
    = \alpha_1 \sum_{k \in S_j} \SV_{U_1}(k) + \alpha_2 \sum_{k \in S_j} \SV_{U_2}(k) 
    = \alpha_1 \nu_{U_1, \mathcal{D}, \Pi}(S_j) + \alpha_2 \nu_{U_2, \mathcal{D}, \Pi}(S_j)$.
    \item \textbf{Efficiency:} For any partition $\Pi = \{S_1, \dots, S_M\}$ of $[n]$, the axiom implies that $\sum_{j=1}^{M} \nu_{U, \mathcal{D}, \Pi}(S_j) = \sum_{j=1}^{M} \sum_{k \in S_j} \SV(k) = \sum_{k \in [n]} \SV(k)$. By the efficiency property of the individual Shapley value, we have $\sum_{k \in [n]} \SV(k) = U([n])$.
    \item \textbf{Faithfulness:} The proposed valuation $\nu_{U, \mathcal{D}, \Pi}(S) = \sum_{i \in S} \SV(i)$ depends only on the group $S$ itself (and $U, \mathcal{D}$), not on the specific partition $\Pi$ that $S$ belongs to. Thus, if $S_0 \in \Pi_1 \cap \Pi_2$, then $\nu_{U, \mathcal{D}, \Pi_1}(S_0) = \sum_{i \in S_0} \SV(i)$ and $\nu_{U, \mathcal{D}, \Pi_2}(S_0) = \sum_{i \in S_0} \SV(i)$. These are equal, so the Faithfulness axiom is satisfied.
\end{enumerate}
All axioms are satisfied.

\textbf{Uniqueness.}
Let $\nu'_{U, \mathcal{D}, \Pi}$ be any group-level data valuation method satisfying Axioms 1-5 from Definition~\ref{definition::axioms}.
We have the following progressive arguments.
\begin{enumerate}
    \item Consider the partition $\Pi^* = \{\{k\}\}_{k \in [n]}$, where each group is a singleton. When Axioms 1-4 (null player, symmetry, linearity, and efficiency) are applied to the groups $\{k\} \in \Pi^*$, they correspond to the standard axioms for individual player valuations. The Shapley value $\SV(k)$ is the unique valuation satisfying these axioms for individual players. Thus, for any $k \in [n]$, we must have $\nu'_{U, \mathcal{D}, \Pi^*}(\{k\}) = \SV(k)$.

    \item For any non-empty subset $S \subseteq [n]$, consider the specific partition $\Pi_S = \{S\} \cup \{\{j\}\}_{j \in [n] \setminus S}$.
    For any singleton group $\{j\}$ where $j \in [n] \setminus S$, note that $\{j\} \in \Pi_S$ and also $\{j\} \in \Pi^*$. By Axiom 5 (faithfulness):
    \[
    \nu'_{U, \mathcal{D}, \Pi_S}(\{j\}) = \nu'_{U, \mathcal{D}, \Pi^*}(\{j\}).
    \]
    Using the result from step 1, $\nu'_{U, \mathcal{D}, \Pi^*}(\{j\}) = \SV(j)$. So, for $j \in [n] \setminus S$:
    \[
    \nu'_{U, \mathcal{D}, \Pi_S}(\{j\}) = \SV(j).
    \]
    Now, apply Axiom 4 (efficiency) to the partition $\Pi_S$:
    \[
    \nu'_{U, \mathcal{D}, \Pi_S}(S) + \sum_{j \in [n] \setminus S} \nu'_{U, \mathcal{D}, \Pi_S}(\{j\}) = U([n]).
    \]
    Substituting $\nu'_{U, \mathcal{D}, \Pi_S}(\{j\}) = \SV(j)$ for $j \in [n] \setminus S$:
    \[
    \nu'_{U, \mathcal{D}, \Pi_S}(S) + \sum_{j \in [n] \setminus S} \SV(j) = U([n]).
    \]
    By the efficiency of individual Shapley values, $U([n]) = \sum_{k \in [n]} \SV(k)$. Therefore:
    \[
    \nu'_{U, \mathcal{D}, \Pi_S}(S) = \sum_{k \in [n]} \SV(k) - \sum_{j \in [n] \setminus S} \SV(j) = \sum_{i \in S} \SV(i).
    \]

    \item Now, consider any arbitrary partition $\Pi$ of $[n]$ such that $S \in \Pi$.
    Since $S \in \Pi$ and $S \in \Pi_S$, by Axiom 5 (faithfulness):
    \[
    \nu'_{U, \mathcal{D}, \Pi}(S) = \nu'_{U, \mathcal{D}, \Pi_S}(S).
    \]
    Substituting the result from step 2:
    \[
    \nu'_{U, \mathcal{D}, \Pi}(S) = \sum_{i \in S} \SV(i).
    \]
\end{enumerate}
Since this holds for any valuation $\nu'$ satisfying the axioms and for any $S \in \Pi$, the valuation method is unique and given by the sum of individual Shapley values.
\end{proof}

\subsection{Proof of Lemma~\ref{lemma::group_shapley_closed_form_exact}}

\begin{proof}[Proof of Lemma~\ref{lemma::group_shapley_closed_form_exact}]
By definition, $\operatorname{FGSV}(S_0) = \sum_{i \in S_0} SV(i)$. Using the standard definition of the individual Shapley value $SV(i)$ from \eqref{eqn::shapley_value_def}:
\begin{align}
    \operatorname{FGSV}(S_0) &= \sum_{i \in S_0} \sum_{S \subseteq [n]\setminus\{i\}} \frac{|S|!(n-|S|-1)!}{n!} \left[ U(S \cup \{i\}) - U(S) \right]. \label{eqn::fgsb_def_sum_sv}
\end{align}
The definition \eqref{eqn::fgsb_def_sum_sv} shows that $\operatorname{FGSV}(S_0)$ is a linear combination of the utility values $U(S)$ for $S \subseteq [n]$. We now determine the coefficient $C(S)$ for a specific subset $S$ with cardinality $s = |S|$.
A term $U(S)$ appears positively in the sum for $SV(i)$ if $i \in S_0 \cap S$, and negatively if $i \in S_0 \setminus S$.
The respective Shapley weights are $\frac{(s-1)!(n-s)!}{n!}$ and $\frac{s!(n-s-1)!}{n!}$. Summing over $i \in S_0$, the coefficient $C(S)$ is therefore:
\begin{align*}
    C(S) &= |S_0 \cap S| \frac{(s-1)!(n-s)!}{n!} - |S_0 \setminus S| \frac{s!(n-s-1)!}{n!}.
\end{align*}
Let $s_1 = |S_0 \cap S|$. Then $|S_0 \setminus S| = s_0 - s_1$. Substituting $s=|S|$ and simplifying the factorials (for $0 < s < n$):
\begin{align*}
    C(S) &= s_1 \frac{(s-1)!(n-s)!}{n!} - (s_0 - s_1) \frac{s!(n-s-1)!}{n!} \\
         &= \frac{s!(n-s)!}{n!} \left[ \frac{s_1}{s} - \frac{s_0 - s_1}{n-s} \right] \\
         &= \frac{1}{\binom{n}{s}} \left[ \frac{s_1(n-s) - (s_0 - s_1)s}{s(n-s)} \right] = \frac{s_1 n - s_0 s}{s(n-s)} \frac{1}{\binom{n}{s}}.
\end{align*}
This coefficient depends only on $s=|S|$ and $s_1=|S_0 \cap S|$. We rewrite $\operatorname{FGSV}(S_0) = \sum_{S \subseteq [n]} C(S) U(S)$ by grouping terms with the same $s$ and $s_1$. We handle the edge cases $s=0$ ($S=\varnothing$) and $s=n$ ($S=[n]$) separately.
For $S=\varnothing$ ($s=0$), $s_1=0$, the coefficient is $C(\varnothing) = -s_0/n$.
For $S=[n]$ ($s=n$), $s_1=s_0$, the coefficient is $C([n]) = s_0/n$.
So we can write:
\begin{align*}
    \operatorname{FGSV}(S_0) &= C([n]) U([n]) + C(\varnothing) U(\varnothing) + \sum_{s=1}^{n-1} \sum_{\substack{S \subseteq [n] \\ |S|=s}} C(S) U(S) \\
    &= \frac{s_0}{n} \left[ U([n]) - U(\varnothing) \right] + \sum_{s=1}^{n-1} \sum_{\substack{S \subseteq [n] \\ |S|=s}} \left( \frac{s_1 n - s_0 s}{s(n-s)} \frac{1}{\binom{n}{s}} \right) U(S),
\end{align*}
where $s_1 = |S \cap S_0|$ within the inner sum. Now, we focus on handling
\begin{equation*}
    \mathcal{T} := \sum_{s=1}^{n-1} \sum_{\substack{S \subseteq [n] \\ |S|=s}} \left( \frac{s_1 n - s_0 s}{s(n-s)} \frac{1}{\binom{n}{s}} \right) U(S).
\end{equation*}
By grouping the inner sum by the value of $s_1 = |S \cap S_0|$ for each fixed size $s$, we have:
\begin{align*}
   \mathcal{T} = \sum_{s=1}^{n-1} \sum_{s_1=\max\{0,s+s_0-n\}}^{\min\{s,s_0\}} \left( \frac{s_1 n - s_0 s}{s(n-s)} \frac{1}{\binom{n}{s}} \right) \sum_{\substack{S: |S|=s \\ |S \cap S_0|=s_1}} U(S).
\end{align*}
Using the definition $\mu\left( \frac{s_1}{s}; s, s_0, n \right) = \frac{\sum_{S:|S|=s, |S\cap S_0|=s_1} U(S)}{\binom{s_0}{s_1}\binom{n-s_0}{s-s_1}}$ (from \eqref{eqn::def_mu}), we have:
\begin{align*}
    \mathcal{T} = \sum_{s=1}^{n-1} \sum_{s_1} \frac{s_1 n - s_0 s}{s(n-s)} \frac{\binom{s_0}{s_1}\binom{n-s_0}{s-s_1}}{\binom{n}{s}} \mu\left( \frac{s_1}{s}; s, s_0, n \right).
\end{align*}
Recognizing the hypergeometric probability $\pr_{\boldsymbol{s_1}\sim \mathcal{HG}(n, s_0, s)}(\boldsymbol{s_1}=s_1) = \frac{\binom{s_0}{s_1}\binom{n-s_0}{s-s_1}}{\binom{n}{s}}$:
\begin{align*}
    \mathcal{T} &= \sum_{s=1}^{n-1} \sum_{s_1} \pr(\boldsymbol{s_1}=s_1) \frac{n (s_1 - s s_0/n)}{s(n-s)} \mu\left( \frac{s_1}{s}; s, s_0, n \right) \\
    &= \sum_{s=1}^{n-1} \mathbb{E}_{\boldsymbol{s_1} \sim \mathcal{HG}(n,s_0,s)} \left[ \frac{n}{s(n - s)} \left( \boldsymbol{s_1} - \frac{s s_0}{n} \right) \mu\left( \frac{\boldsymbol{s_1}}{s}; s, s_0, n \right) \right] \\
    &= \sum_{s=1}^{n-1} \mathcal{T}(s),
\end{align*}
where $\mathcal{T}(s)$ matches the definition \eqref{eqn::def::T(s)}.
Combining the parts, we get $\operatorname{FGSV}(S_0) = \frac{s_0}{n} \left[ U([n]) - U(\varnothing) \right] + \sum_{s=1}^{n - 1} \mathcal{T}(s)$, proving the lemma.
\end{proof}

\subsection{Proof of Theorem~\ref{proposition::calT}} \label{sec::proof_proposition_calT}

Before presenting the proof of~\Cref{proposition::calT}, we establish two lemmas that demonstrate how Assumption~\ref{assump::utility_stability} implies a smoothness condition on the function $\mu(\cdot)$. The first lemma shows that this assumption yields a bound on the second-order finite difference of $\mu$.

\begin{restatable}{lemma}{LemmaMuSecondOrderStability}
Under Assumption~\ref{assump::utility_stability}, the following second-order difference bound holds for $\mu\left(\frac{s_1}{s}; s, s_0, n\right)$:
\begin{equation} \label{eqn::mu_stability}
    \left| \mu\left(\frac{s_1+1}{s}; s, s_0, n\right) - 2\mu\left(\frac{s_1}{s}; s, s_0, n \right) + \mu\left(\frac{s_1 - 1}{s}; s, s_0, n\right)\right| \leq \frac{C}{s^{\frac{3}{2}+\upsilon}}.
\end{equation}
\label{lemma::mu_second_order_stability}
\end{restatable}

The proof of Lemma~\ref{lemma::mu_second_order_stability} is in Appendix~\ref{sec::lemma_proposition_calT}.

Next, we show that this second-order difference bound implies a first-order approximation error bound -- a discrete analogue of Taylor’s theorem. 
To formalize this, we define the following continuity-extended versions of a discrete function and its first-order derivative.

\begin{definition}[Continuity-extended function]
Let $f_s(x)$ be a function defined on rational points of the form $x = \frac{s_1}{s}$ for integers $s_1 \in \{0, 1, \dots, s\}$. The \emph{continuity-extended function} $\tilde{f}_s(x)$ is defined as the linear interpolation:
\begin{equation*}
    \tilde{f}_s(x) = \begin{cases} 
        (sx - \lfloor sx \rfloor) \cdot f_s\left( \frac{\lfloor sx \rfloor + 1}{s} \right) + (\lfloor sx \rfloor + 1 - sx) \cdot f_s\left( \frac{\lfloor sx \rfloor}{s} \right), & x \in [0, 1), \\
        f_s(1), & x = 1.
    \end{cases}
\end{equation*}
\label{def::continuity_extended_func}
\end{definition}

Intuitively, $\tilde f_s$ linearly interpolates $f_s$ between adjacent grid points.
Notice that $\tilde f_s$ is indeed smooth in between, thus $\tilde f_s'$ is well-defined there.
When taking the derivative of $\tilde f_s$ w.r.t. $x$, notice that $\lfloor sx \rfloor$ is constant between grid points, thus differentiate to zero.
This naturally leads to the following definition.

\begin{definition}[Continuity-extended first-order derivative]
Let $f_s(x)$ be as above. The \emph{continuity-extended first-order derivative} $\tilde{f}_s'(x)$ is defined as:
\begin{equation*}
    \tilde{f}_s'(x) = s \cdot \left[ f_s\left( \frac{\lfloor sx \rfloor + 1}{s} \right) - f_s\left( \frac{\lfloor sx \rfloor}{s} \right) \right], \quad x \in (0, 1).
\end{equation*}
\label{def::continuity_extended_diff}
\end{definition}

\begin{restatable}{lemma}{LemmaMuFirstOrderExpansion}
    \label{lemma::mu_first_order_expansion}
    Let $\{f_s(x)\}_{s=1}^\infty$ be a sequence of functions defined on rational points $x = \frac{s_1}{s}$ with $s_1 \in \{0, 1, \dots, s\}$. Suppose $f_s(x)$ satisfies the discrete second-order difference bound:
    \begin{equation} \label{eqn::discrete_second_order_diff}
        \left| f_s\left( \frac{s_1+1}{s} \right) - 2f_s\left(\frac{s_1}{s}\right) + f_s\left(\frac{s_1-1}{s}\right) \right| \leq \frac{C}{s^{\frac{3}{2}+\upsilon}}, \quad 1 \leq s_1 < s - 1,
    \end{equation}
    for constants $C > 0$ and $\upsilon > 0$. Then for all $s_1, x$ with $0 < \frac{s_1}{s}, x < 1$, the following first-order approximation holds:
    \begin{equation} \label{eqn::discrete_first_order_cond}
        \left| f_s\left( \frac{s_1}{s} \right) - \tilde{f}_s(x) - \tilde{f}_s'(x) \cdot \left( \frac{s_1}{s} - x \right) \right| \leq Cs^{\frac{1}{2} - \upsilon} \left( \frac{s_1}{s} - x \right)^2.
    \end{equation}
\end{restatable}
To intuitively understand Lemma \ref{lemma::mu_first_order_expansion}, notice that by definition, we have
\begin{align}
    f_s\left( \frac{s_1}{s} \right) - \tilde{f}_s(x) - \tilde{f}_s'(x) \cdot \left( \frac{s_1}{s} - x \right)
    =
    0,
    \quad
    \text{for $x\in [s_1/s,(s_1+1)/s)$}.
    \label{proofeqn::lemma-6-interpretation}
\end{align}
Then in view of \eqref{eqn::discrete_second_order_diff}, it is not difficult to understand that the error bound on the RHS of \eqref{eqn::discrete_first_order_cond} is the consequence of bounding telescoping sums that eventually reduces the problem to the case of \eqref{proofeqn::lemma-6-interpretation}, summing up the approximation errors along the way.

The formal proof of Lemma~\ref{lemma::mu_first_order_expansion} is in Appendix~\ref{sec::lemma_proposition_calT}. 

We now proceed to prove~\Cref{proposition::calT}.

\begin{proof}[Proof of~\Cref{proposition::calT}]
The core idea is to approximate the function $\mu(\frac{\boldsymbol{s_1}}{s}; s, s_0, n)$ inside the expectation defining $\mathcal{T}(s)$ using a first-order Taylor-like expansion around the mean proportion $\alpha_0 = s_0/n$. The validity of this expansion relies on the smoothness properties derived from Assumption~\ref{assump::utility_stability}.

Let $\mu_s(x) := \mu(x; s, s_0, n)$ denote the function $\mu$ for fixed $s, s_0, n$, where $x = s_1/s$ is the proportion of intersection. By Lemma~\ref{lemma::mu_second_order_stability}, Assumption~\ref{assump::utility_stability} ensures that $\mu_s$ satisfies the second-order difference bound given in \eqref{eqn::mu_stability}.
We can then apply Lemma~\ref{lemma::mu_first_order_expansion} (First-Order Approximation Bound) with $f_s \to \mu_s$, the evaluation point $x_0 = s_1/s$, and the expansion point $x = \alpha_0$, and get the following approximation for $\mu$:
\begin{align*}
    \mu\left(\frac{s_1}{s}; s, s_0, n\right) = \tilde{\mu}\left(\alpha_0; s, s_0, n\right) + \tilde{\mu}'\left(\alpha_0; s, s_0, n\right) \cdot \left( \frac{s_1}{s} - \alpha_0 \right) + R\left(\frac{s_1}{s}; s, s_0, n \right),
\end{align*}
where $\tilde{\mu}$ and $\tilde{\mu}'$ are the continuity-extended function and its first-order difference defined in Definitions~\ref{def::continuity_extended_func} and~\ref{def::continuity_extended_diff}, respectively. The remainder term $R$ satisfies the bound from Lemma~\ref{lemma::mu_first_order_expansion} (Eq.~\eqref{eqn::discrete_first_order_cond}):
\begin{equation*}
    \left| R\left(\frac{s_1}{s}; s, s_0, n \right) \right| \leq C' s^{\frac{1}{2} - \upsilon} \left( \frac{s_1}{s} - \alpha_0 \right)^2 = C' s^{-\frac{3}{2} - \upsilon} (s_1 - s\alpha_0)^2,
\end{equation*}
with $C'$ depending on the constant $C$ from Assumption~\ref{assump::utility_stability}.

Now, we substitute this expansion of $\mu$ into the definition of $\mathcal{T}(s)$ (Eq.~\eqref{eqn::def::T(s)}):
\begin{align*}
\mathcal{T}(s) &= \mathbb{E}_{\boldsymbol{s_1} \sim \mathcal{HG}(n, s_0, s)} \left[ \frac{n}{s(n - s)} \left(\boldsymbol{s_1} - s \alpha_0 \right) \mu\left( \frac{\boldsymbol{s_1}}{s}; s, s_0, n \right) \right] \\
&= \mathbb{E}_{\boldsymbol{s_1}} \left[ \frac{n}{s(n - s)} (\boldsymbol{s_1} - s \alpha_0) \left( \tilde{\mu}(\alpha_0; s, s_0, n) + \tilde{\mu}'(\alpha_0; s, s_0, n) \left( \frac{\boldsymbol{s_1}}{s} - \alpha_0 \right) + R\left( \frac{\boldsymbol{s_1}}{s}; s, s_0, n \right) \right) \right].
\end{align*}
Using the linearity of expectation, we decompose $\mathcal{T}(s)$ into three terms:
\begin{align*}
\mathcal{T}(s) &= \underbrace{ \tilde{\mu}\left( \alpha_0; s, s_0, n \right) \mathbb{E}_{\boldsymbol{s_1}}\left[ \frac{n}{s(n - s)} \left(\boldsymbol{s_1} - s\alpha_0 \right) \right] }_{\text{Term I}} \\
&\quad + \underbrace{ \tilde{\mu}'\left( \alpha_0; s, s_0, n \right) \cdot \mathbb{E}_{\boldsymbol{s_1}}\left[ \frac{n}{s(n - s)} \left(\boldsymbol{s_1} - s\alpha_0 \right) \left( \frac{\boldsymbol{s_1}}{s} - \alpha_0 \right) \right] }_{\text{Term II}} \\
&\quad + \underbrace{ \mathbb{E}_{\boldsymbol{s_1}}\left[ \frac{n}{s(n - s)} \left(\boldsymbol{s_1} - s\alpha_0 \right) R\left( \frac{\boldsymbol{s_1}}{s}; s, s_0, n \right) \right] }_{\text{Term III}}.
\end{align*}

\textbf{Analysis of Term I:}
The random variable $\boldsymbol{s_1}$ follows the Hypergeometric distribution $\mathcal{HG}(n, s_0, s)$, which has mean $\mathbb{E}[\boldsymbol{s_1}] = s s_0 / n = s \alpha_0$. Therefore, $\mathbb{E}[\boldsymbol{s_1} - s\alpha_0] = 0$, which implies Term I = 0.

\textbf{Analysis of Term II:}
The expectation in Term II involves the second central moment (variance) of $\boldsymbol{s_1}$:
\begin{align*}
    \mathbb{E}_{\boldsymbol{s_1}}\left[ \left(\boldsymbol{s_1} - s\alpha_0 \right) \left( \frac{\boldsymbol{s_1}}{s} - \alpha_0 \right) \right] &= \mathbb{E}_{\boldsymbol{s_1}}\left[ \frac{1}{s} (\boldsymbol{s_1} - \mathbb{E}\boldsymbol{s_1})^2 \right] = \frac{1}{s} \operatorname{Var}(\boldsymbol{s_1}).
\end{align*}
The variance of $\mathcal{HG}(n, s_0, s)$ is $\operatorname{Var}(\boldsymbol{s_1}) = s \alpha_0 (1-\alpha_0) \frac{n-s}{n-1}$. Substituting this into Term II:
\begin{align*}
    \text{Term II} &= \tilde{\mu}'\left( \alpha_0; s, s_0, n \right) \cdot \frac{n}{s(n - s)} \cdot \frac{1}{s} \operatorname{Var}(\boldsymbol{s_1}) \\
    &= \tilde{\mu}'\left( \alpha_0; s, s_0, n \right) \cdot \frac{n}{s^2(n - s)} \cdot \left( s \alpha_0 (1-\alpha_0) \frac{n-s}{n-1} \right) \\
    &= \tilde{\mu}'\left( \alpha_0; s, s_0, n \right) \cdot \frac{n \alpha_0(1 - \alpha_0)}{s(n - 1)}.
\end{align*}
Now, substitute the definition of $\tilde{\mu}'$ from Definition~\ref{def::continuity_extended_diff} evaluated at $x = \alpha_0 = s_0/n$. Let $s_1^* = \lfloor s \alpha_0 \rfloor$. Then:
\begin{equation*}
    \tilde{\mu}'\left( \alpha_0; s, s_0, n \right) = s \left[ \mu\left( \frac{s_1^* + 1}{s}; s, s_0, n \right) - \mu\left( \frac{s_1^*}{s}; s, s_0, n \right) \right] = s \, \Delta\mu\left( \frac{s_1^*}{s}; s, s_0, n \right).
\end{equation*}
Substituting this into the expression for Term II yields:
\begin{equation*}
    \text{Term II} = \left( s \, \Delta\mu\left( \frac{s_1^*}{s}; s, s_0, n \right) \right) \cdot \frac{n \alpha_0(1 - \alpha_0)}{s(n - 1)} = \frac{n}{n - 1} \alpha_0(1 - \alpha_0) \Delta\mu\left( \frac{s_1^*}{s}; s, s_0, n \right).
\end{equation*}

\textbf{Analysis of Term III (Remainder Term):}
We bound the absolute value using the bound on $|R|$ from Lemma~\ref{lemma::mu_first_order_expansion}:
\begin{align*}
    |\text{Term III}| &= \left| \mathbb{E}_{\boldsymbol{s_1}} \left[ \frac{n}{s(n - s)} \left( \boldsymbol{s_1} - s \alpha_0 \right) R\left( \frac{\boldsymbol{s_1}}{s}; s, s_0, n \right) \right] \right| \\
    &\le \frac{n}{s(n - s)} \mathbb{E}_{\boldsymbol{s_1}} \left[ |\boldsymbol{s_1} - \mathbb{E}\boldsymbol{s_1}| \cdot \left| R\left( \frac{\boldsymbol{s_1}}{s}; s, s_0, n \right) \right| \right] \\
    &\le \frac{n}{s(n - s)} \mathbb{E}_{\boldsymbol{s_1}} \left[ |\boldsymbol{s_1} - \mathbb{E}\boldsymbol{s_1}| \cdot C s^{-\frac{3}{2} - \upsilon} (\boldsymbol{s_1} - \mathbb{E}\boldsymbol{s_1})^2 \right] \\
    &= \frac{n C}{s^{5/2 + \upsilon}(n - s)} \mathbb{E}_{\boldsymbol{s_1}} \left[ |\boldsymbol{s_1} - \mathbb{E}\boldsymbol{s_1}|^3 \right].
\end{align*}
Using Cauchy-Schwarz ($\mathbb{E}[|X|^3] \le \sqrt{\mathbb{E}[X^2] \mathbb{E}[X^4]}$ where $X = \boldsymbol{s_1} - \mathbb{E}\boldsymbol{s_1}$), and the fact that $\operatorname{Var}(\boldsymbol{s_1}) = O\left(\alpha_0(1 - \alpha_0)\frac{s(n-s)}{n} \right)$ and $\operatorname{Kurtosis}(\boldsymbol{s_1}) = O(\frac{n}{s(n-s)\alpha_0(1 - \alpha_0)}) = O(\frac{1}{\alpha_0(1 - \alpha_0)})$ for $1 \leq s \leq n - 1$, we yields:
\begin{align*}
    \mathbb{E}_{\boldsymbol{s_1}} \left[ |\boldsymbol{s_1} - \mathbb{E}\boldsymbol{s_1}|^3 \right] &\lesssim \sqrt{\mathbb{E}_{\boldsymbol{s_1}} \left[ |\boldsymbol{s_1} - \mathbb{E}\boldsymbol{s_1}|^2 \right] \cdot \mathbb{E}_{\boldsymbol{s_1}} \left[ |\boldsymbol{s_1} - \mathbb{E}\boldsymbol{s_1}|^4 \right]}    \\
    &\lesssim \sqrt{\operatorname{Var}(\boldsymbol{s_1}) \cdot (\operatorname{Kurtosis}(\boldsymbol{s_1}) + 3) (\operatorname{Var}(\boldsymbol{s_1}))^2} \\
    &\lesssim \alpha_0(1 - \alpha_0) \left( \frac{s(n - s)}{n - 1} \right)^{3/2}.
\end{align*}
Substituting this into the bound for $|\text{Term III}|$:
\begin{align*}
    |\text{Term III}| &\lesssim \frac{n}{s^{5/2 + \upsilon}(n - s)} \alpha_0(1 - \alpha_0)\left( \frac{s(n - s)}{n - 1} \right)^{3/2} \\
    &\lesssim n \alpha_0(1-\alpha_0) \frac{(n-s)^{1/2}}{(n-1)^{3/2}} s^{-1-\upsilon} \\
    &= O\left(\frac{n}{n-1} \alpha_0(1-\alpha_0) s^{-(1 + \upsilon)} \right).
\end{align*}
So, Term III contributes the $O(s^{-(1+\upsilon)})$ error term, scaled by factors related to $\alpha_0$ and $n/(n-1)$.

\textbf{Conclusion:} Combining Terms I=0, Term II, and the bound on Term III gives:
\begin{align*}
\mathcal{T}(s) &= \frac{n}{n - 1} \alpha_0(1 - \alpha_0) \Delta\mu\left( \frac{s_1^*}{s}; s, s_0, n \right) + O\left(\frac{n}{n-1} \alpha_0(1 - \alpha_0) s^{-(1+\upsilon)}\right) \\
&= \frac{n}{n - 1} \alpha_0(1 - \alpha_0) \left[ \Delta\mu\left( \frac{s_1^*}{s}; s, s_0, n \right) + O\big(s^{-(1 + \upsilon)}\big) \right],
\end{align*}
where the constant factor in the $O(\cdot)$ term depends on $C, C', \upsilon$, but not on $s$ and $\alpha_0$. This proves~\Cref{proposition::calT}.
\end{proof}

\subsubsection{Technical lemmas used in the proof of Theorem~\ref{proposition::calT}} \label{sec::lemma_proposition_calT}

\LemmaMuSecondOrderStability*

\begin{proof}[Proof of Lemma~\ref{lemma::mu_second_order_stability}]
    The core idea is to relate the second-order difference of $\mu$ (which is an average of $U$ values) to the average of the second-order difference quantity from Assumption~\ref{assump::utility_stability}. Specifically, we consider the stability term $\Delta(S', z_1, z_1', z_2, z_2') := U(S' \cup \{z_1, z_1'\}) - U(S' \cup \{z_1', z_2'\}) - U(S' \cup \{z_1, z_2\}) + U(S' \cup \{z_2, z_2'\})$, which is bounded by $C |S'|^{-(3/2+\upsilon)}$ by Assumption~\ref{assump::utility_stability}. We will construct a sum, $\mathcal{M}(s_1, s_2)$, by averaging $\Delta(S', z_1, z_1', z_2, z_2')$ over appropriate base sets $S'$ and points $z_1, z_1' \in S_0$, $z_2, z_2' \notin S_0$. We then show that this average value, $\mathcal{M}(s_1, s_2)$ properly normalized, exactly equals the second-order finite difference of $\mu$ stated in the lemma. The bound on $\mu$'s second difference then follows from the bound on $\Delta$.

    Let $s_2 = s - s_1$. We define the sum $\mathcal{M}(s_1, s_2)$ over all possible configurations:
    \begin{equation} \label{eqn::def_M_prop2}
        \mathcal{M}(s_1, s_2) := \sum_{\substack{S_1' \subseteq S_0 \\ |S_1'| = s_1 - 1}} \sum_{\substack{S_2' \subseteq S_0^c \\ |S_2'| = s_2 - 1}} \sum_{\substack{z_1, z_1' \in S_0 \setminus S_1' \\ z_1 \neq z_1'}} \sum_{\substack{z_2, z_2' \in S_0^c \setminus S_2' \\ z_2 \neq z_2'}} \Delta(S_1' \cup S_2', z_1, z_1', z_2, z_2'),
    \end{equation}
    where the sum is defined for $s_1 \ge 1, s_2 \ge 1$ and requires $|S_0 \setminus S_1'| \ge 2$ (i.e., $s_0 - (s_1-1) \ge 2 \implies s_1 \le s_0-1$) and $|S_0^c \setminus S_2'| \ge 2$ (i.e., $(n-s_0) - (s_2-1) \ge 2 \implies s_2 \le n-s_0-1$).

    Our goal now is to demonstrate that this aggregated sum $\mathcal{M}(s_1, s_2)$, when properly normalized, equals the second-order finite difference of the average utility function $\mu$. Specifically, we aim to prove the following identity:
    \begin{equation} \label{eqn::M_eq_prop2}
        \frac{\mathcal{M}(s_1, s_2)}{N_{\rm terms}} = \mu\left(\frac{s_1+1}{s}; s, s_0, n\right) - 2\mu\left(\frac{s_1}{s}; s, s_0, n \right) + \mu\left(\frac{s_1 - 1}{s}; s, s_0, n\right),
    \end{equation}
    where $N_{\rm terms}$ is the total number of $\Delta(\cdot)$ terms summed in the definition of $\mathcal{M}(s_1, s_2)$ (Eq.~\eqref{eqn::def_M_prop2}). We establish this identity through the following combinatorial analysis.

    Firstly, when the $\Delta(\cdot)$ terms are expanded, $\mathcal{M}(s_1, s_2)$ becomes a sum of individual $U(S)$ terms. Each such set $S$ has size $s = s_1+s_2$. Furthermore, based on the structure of $\Delta(S', z_1, z_1', z_2, z_2')$, the number of elements from $S_0$ in any $S$ appearing with a non-zero coefficient must be $s_1+1$, $s_1$, or $s_1-1$. Let $\mathcal{A}_{s_1', s}$ denote the collection of all sets $S \subseteq [n]$ such that $|S|=s$ and $|S \cap S_0|=s_1'$. Due to the symmetric construction of $\mathcal{M}(s_1, s_2)$, all $U(S)$ terms for $S$ within the same collection $\mathcal{A}_{s_1', s}$ appear with the same net coefficient in the expansion of $\mathcal{M}(s_1, s_2)$.

    Secondly, we determine these net coefficients. Let $C(s_1', s)$ be the coefficient for any $U(S)$ where $S \in \mathcal{A}_{s_1', s}$. A detailed combinatorial count, considering the base sets $S'=S_1'\cup S_2'$ and the ordered pairs $(z_1, z_1')$ and $(z_2, z_2')$ involved in the definition of $\mathcal{M}(s_1, s_2)$, yields the following coefficients:
    \begin{itemize}
        \item For sets $S \in \mathcal{A}_{s_1+1, s}$ (meaning $|S \cap S_0|=s_1+1$ and thus $|S \cap S_0^c| = s-(s_1+1) = s_2-1$, where $s_2=s-s_1$): These utility terms $U(S)$ arise solely from the $+U(S' \cup \{z_1, z_1'\})$ component in the expansion of $\Delta(S', z_1, z_1', z_2, z_2')$ within the sum $\mathcal{M}(s_1, s_2)$ defined in \eqref{eqn::def_M_prop2}. For a fixed \( S = S_1 \cup S_2 \) in this collection (with \( |S_1| = s_1 + 1 \) and \( |S_2| = s_2 - 1 \)), the number of ways to form $S$ via $(S', z_1, z_1',  z_2, z_2')$ requires choosing ordered pairs $(z_1, z_1') \in S_1$, $(z_2, z_2') \in S_0^c \setminus S_2$. The resulting coefficient is:
        \begin{equation*}
            C(s_1+1, s) = (s_1 + 1)s_1 \times (n-s_0-s_2+1)(n-s_0-s_2).
        \end{equation*}

        \item For $S \in \mathcal{A}_{s_1, s}$ (which corresponds to $s_2' = s_2$): Terms arise from $-U(S' \cup \{z_1, z_2\})$ and $-U(S' \cup \{z_1', z_2'\})$ in the expansion of $\Delta(S', z_1, z_1', z_2, z_2')$. For a fixed $S = S_1 \cup S_2$ (with $|S_1|=s_1, |S_2|=s_2$), the number of ways to form it via $(S', z_1, z_1', z_2, z_2')$ requires choosing $S_1 = S_1' \setminus \{z_1\}$, $S_2 = S_2' \setminus \{z_2\}$, $z_1' \in S_0 \setminus S_1'$, $z_2' \in S_0^c \setminus S_2'$. Then the resulting coeffcient is:
        \begin{equation*}
            C (s_1, s) = s_1(n - s_0 - s_2)\times (s_0 - s_1)s_2.
        \end{equation*}
        
        \item For $S \in \mathcal{A}_{s_1-1, s}$ (which corresponds to $s_2' = s_2+1$): The term arises only from $+U(S' \cup \{z_2, z_2'\})$. With similar argument for $\mathcal{A}_{s_1+1, s}$, the resulting coefficient is:
        \begin{equation*}
            C(s_1-1, s) =  (s_0-s_1+1)(s_0-s_1)\times (s-s_1+1)(s-s_1).
        \end{equation*}
    \end{itemize}

    Thirdly, we determine the total number of $\Delta(\cdot)$ terms in the sum $\mathcal{M}(s_1, s_2)$. The number of ways to choose the base set $S' = S_1' \cup S_2'$ (with $|S_1'|=s_1-1, |S_2'|=s_2-1$) is $N_{S'} = \binom{s_0}{s_1-1}\binom{n-s_0}{s_2-1}$. For each $S'$, the number of ordered pairs $(z_1, z_1')$ with $z_1 \neq z_1'$ from $S_0 \setminus S_1'$ is $N_{z1} = P(s_0-s_1+1, 2) = (s_0-s_1+1)(s_0-s_1)$. The number of ordered pairs $(z_2, z_2')$ with $z_2 \neq z_2'$ from $S_0^c \setminus S_2'$ is $N_{z2} = P(n-s_0-s_2+1, 2) = (n-s_0-s_2+1)(n-s_0-s_2)$. Therefore, the total number of terms is:
    \begin{equation*}
        N_{\rm terms} = N_{S'} \times N_{z1} \times N_{z2} = \binom{s_0}{s_1 - 1}\binom{n - s_0}{s_2 - 1}(s_0 - s_1 + 1)(s_0 - s_1)(n - s_0 - s_2 + 1)(n - s_0 - s_2).
    \end{equation*}
    
    Gathering these observations, the crucial step relies on the fact that the derived coefficients $C(\cdot)$ and the normalization $N_{\rm terms}$, when combined with the number of sets in each collection $|\mathcal{A}_{s_1', s}| = \binom{s_0}{s_1'}\binom{n-s_0}{s-s_1'}$, simplify exactly as follows:
    \begin{align}
        \frac{\mathcal{M}(s_1, s_2)}{N_{\rm terms}} &= \frac{C(s_1+1, s)}{N_{\rm terms}}\sum_{S\in \mathcal{A}_{s_1+1, s}}U(S) + \frac{C(s_1, s)}{N_{\rm terms}}\sum_{S\in \mathcal{A}_{s_1, s}}U(S) + \frac{C(s_1-1, s)}{N_{\rm terms}}\sum_{S\in \mathcal{A}_{s_1-1, s}}U(S) \notag\\
        &= +1 \cdot \frac{\sum_{S\in \mathcal{A}_{s_1+1, s}}U(S)}{|\mathcal{A}_{s_1+1, s}|} -2 \cdot \frac{\sum_{S\in \mathcal{A}_{s_1, s}}U(S)}{|\mathcal{A}_{s_1, s}|} + 1 \cdot \frac{\sum_{S\in \mathcal{A}_{s_1-1, s}}U(S)}{|\mathcal{A}_{s_1-1, s}|} \notag\\
        &= \mu\left(\frac{s_1+1}{s}; s, s_0, n\right) - 2\mu\left(\frac{s_1}{s}; s, s_0, n \right) + \mu\left(\frac{s_1 - 1}{s}; s, s_0, n\right).
        \label{proofeqn::lemma-5::1}
    \end{align}
    This establishes the identity~\eqref{eqn::M_eq_prop2} relating the average $\Delta$ value to the second difference of $\mu$.

    Now, we bound the absolute value of the left-hand side of~\eqref{eqn::mu_stability}. 
    We have
    \begin{align*}
        |\text{RHS of \eqref{proofeqn::lemma-5::1}}| &= \left| \frac{1}{N_{\rm terms}} \sum_{S', z_1, z_1', z_2, z_2'} \Delta(S', z_1, z_1', z_2, z_2') \right| \\
        &\le \frac{1}{N_{\rm terms}} \sum_{S', z_1, z_1', z_2, z_2'} \left| \Delta(S', z_1, z_1', z_2, z_2') \right| \quad \text{(By Triangle Inequality)} \\
        &\le \max_{S', z_1, z_1', z_2, z_2'} \left| \Delta(S', z_1, z_1', z_2, z_2') \right|.
    \end{align*}
    By Assumption~\ref{assump::utility_stability}, for any base set $S'$ (of size $s-2$) and points $z_1, z_1', z_2, z_2'$:
    \begin{equation*}
        |\Delta(S', z_1, z_1', z_2, z_2')| \le C |S'|^{-(3/2+\upsilon)} = C (s-2)^{-(3/2+\upsilon)}.
    \end{equation*}
    Therefore,
    \begin{equation*}
        |\text{RHS of \eqref{proofeqn::lemma-5::1}}| \le C (s-2)^{-(3/2+\upsilon)}.
    \end{equation*}
    For $s \ge 2$, $(s-2)^{-(3/2+\upsilon)}$ is of the order $O(s^{-(3/2+\upsilon)})$. We can absorb the constant factor difference between $(s-2)^{-k}$ and $s^{-k}$ into a modified constant $C$ (or $C'$), yielding the desired bound:
    \begin{equation*}
        \left| \mu\left(\frac{s_1+1}{s}; s, s_0, n\right) - 2\mu\left(\frac{s_1}{s}; s, s_0, n \right) + \mu\left(\frac{s_1 - 1}{s}; s, s_0, n\right) \right| \le \frac{C'}{s^{3/2+\upsilon}}.
    \end{equation*}
    This completes the proof.
\end{proof}

\LemmaMuFirstOrderExpansion*

\begin{proof}[Proof of Lemma~\ref{lemma::mu_first_order_expansion}]
Let $\tilde{f}_s'(y)$ be the continuity-extended first-order difference as defined in Definition~\ref{def::continuity_extended_diff}. For any integer $t$ such that $0 \le t < s$, $\tilde{f}_s'(\frac{t}{s}) = s \left[ f_s\left( \frac{t+1}{s} \right) - f_s\left( \frac{t}{s} \right) \right]$.
The discrete second-order difference bound \eqref{eqn::discrete_second_order_diff} states that for $1 \le t \le s-1$:
\begin{equation*}
    \left| f_s\left( \frac{t+1}{s} \right) - 2f_s\left(\frac{t}{s}\right) + f_s\left(\frac{t-1}{s}\right) \right| \leq \frac{C}{s^{\frac{3}{2}+\upsilon}}.
\end{equation*}
Consider the difference of $\tilde{f}_s'$ at consecutive grid points $t/s$ and $(t-1)/s$ for $1 \le t < s$:
\begin{align*}
    \left|\tilde{f}_s'\left(\frac{t}{s}\right) - \tilde{f}_s'\left(\frac{t-1}{s}\right)\right|
    &= \left| s \left[ f_s\left( \frac{t+1}{s} \right) - f_s\left( \frac{t}{s} \right) \right] - s \left[ f_s\left( \frac{t}{s} \right) - f_s\left( \frac{t-1}{s} \right) \right] \right| \\
    &= s \left| f_s\left( \frac{t+1}{s} \right) - 2f_s\left( \frac{t}{s} \right) + f_s\left( \frac{t-1}{s} \right) \right| \\
    &\leq s \cdot \frac{C}{s^{\frac{3}{2}+\upsilon}} = \frac{C}{s^{\frac{1}{2}+\upsilon}}. \quad \text{(This holds for } 1 \le t \le s-1\text{)}.
\end{align*}
This implies a Lipschitz-like condition for $\tilde{f}_s'$ between grid points. For any two grid points $x_a = a/s$ and $x_b = b/s$ (assume $a < b$ without loss of generality):
\begin{align*}
    \left| \tilde{f}_s'(x_b) - \tilde{f}_s'(x_a) \right| &= \left| \sum_{j=a}^{b-1} \left( \tilde{f}_s'\left(\frac{j+1}{s}\right) - \tilde{f}_s'\left(\frac{j}{s}\right) \right) \right| \\
    &\le \sum_{j=a}^{b-1} \left| \tilde{f}_s'\left(\frac{j+1}{s}\right) - \tilde{f}_s'\left(\frac{j}{s}\right) \right| \le \sum_{j=a}^{b-1} \frac{C}{s^{\frac{1}{2}+\upsilon}} = (b-a) \frac{C}{s^{\frac{1}{2}+\upsilon}}.
\end{align*}
So, for any two grid points $x_i=i/s, x_j=j/s$, $|\tilde{f}_s'(x_i) - \tilde{f}_s'(x_j)| \leq \frac{C}{s^{\frac{1}{2}+\upsilon}} |i-j|$.
Since $\tilde{f}_s'(x)$ is piecewise constant between grid points (equal to $\tilde{f}_s'(\lfloor sx \rfloor / s)$), for a general $x \in [0,1)$ and a grid point $t/s$:
\begin{equation} \label{eq:fs_prime_lipschitz_like}
    \left| \tilde{f}_s'\left( \frac{t}{s} \right) - \tilde{f}_s'(x) \right| = \left| \tilde{f}_s'\left( \frac{t}{s} \right) - \tilde{f}_s'\left( \frac{\lfloor sx \rfloor}{s} \right) \right| \leq \frac{C}{s^{\frac{1}{2}+\upsilon}} |t - \lfloor sx \rfloor|.
\end{equation}

We want to bound $\left| f_s\left(\frac{s_1}{s}\right) - \tilde{f}_s(x) - \tilde{f}_s'(x) \cdot \left(\frac{s_1}{s} - x\right) \right|$.
Consider the case $\frac{s_1}{s} > x$. The term $f_s\left(\frac{s_1}{s}\right) - \tilde{f}_s(x)$ can be written as a sum of first differences plus a boundary term:
\begin{align*}
    &f_s\left(\frac{s_1}{s}\right) - \tilde{f}_s(x) \\
    &\quad = \left( f_s\left(\frac{s_1}{s}\right) - f_s\left(\frac{s_1-1}{s}\right) + \dots + f_s\left(\frac{\lfloor sx \rfloor+1}{s}\right) - f_s\left(\frac{\lfloor sx \rfloor}{s}\right) \right) - \left( \tilde{f}_s(x) - f_s\left(\frac{\lfloor sx \rfloor}{s}\right) \right) \\
    &\quad = \sum_{t=\lfloor sx \rfloor}^{s_1-1} \left( f_s\left(\frac{t+1}{s}\right) - f_s\left(\frac{t}{s}\right) \right) - \left( (sx - \lfloor sx \rfloor) \left( f_s\left(\frac{\lfloor sx \rfloor + 1}{s}\right) - f_s\left(\frac{\lfloor sx \rfloor}{s}\right) \right) \right) \quad (\text{by definition of } \tilde{f}_s) \\
    &\quad = \frac{1}{s} \sum_{t=\lfloor sx \rfloor}^{s_1-1} \tilde{f}_s'\left(\frac{t}{s}\right) - (sx - \lfloor sx \rfloor) \frac{1}{s} \tilde{f}_s'\left(\frac{\lfloor sx \rfloor}{s}\right) \\
    &\quad = \frac{1}{s} \sum_{t=\lfloor sx \rfloor + 1}^{s_1-1} \tilde{f}_s'\left(\frac{t}{s}\right) + (\lfloor sx \rfloor + 1 - sx) \frac{1}{s} \tilde{f}_s'\left( \frac{\lfloor sx \rfloor}{s}\right) \\
    &\quad = \frac{1}{s} \sum_{t=\lfloor sx \rfloor + 1}^{s_1-1} \tilde{f}_s'\left(\frac{t}{s}\right) + (\lfloor sx \rfloor + 1 - sx) \frac{1}{s} \tilde{f}_s'\left( x\right) \quad (\text{by definition of } \tilde{f}_s')
\end{align*}
The expression to bound is:
\begin{align*}
    & f_s\left( \frac{s_1}{s} \right) - \tilde{f}_s(x) - \tilde{f}_s'(x) \cdot \left( \frac{s_1}{s} - x \right) \\
    &= \left( \sum_{t=\lceil sx \rceil}^{s_1-1} \left[ f_s\left(\frac{t+1}{s}\right) - f_s\left(\frac{t}{s}\right) \right] + f_s\left(\frac{\lceil sx \rceil}{s}\right) \right) - \tilde{f}_s(x) - \tilde{f}_s'(x) \left( \frac{s_1}{s} - x \right).
\end{align*}
Therefore:
\begin{align*}
    &\left| f_s\left(\frac{s_1}{s}\right) - \tilde{f}_s(x) - \tilde{f}_s'(x) \cdot \left(\frac{s_1}{s} - x\right) \right| \\
    &= \left| \sum_{t=\lfloor sx \rfloor + 1}^{s_1-1} \frac{1}{s}  \tilde{f}_s'\left(\frac{t}{s}\right) + (\lfloor sx \rfloor + 1 - sx) \frac{1}{s} \tilde{f}_s'\left( x\right)  - \left( \frac{s_1}{s} - x\right)\cdot \tilde{f}_s'(x) \right| \\
    &= \left| \sum_{t=\lfloor sx \rfloor + 1}^{s_1-1} \frac{1}{s}  \tilde{f}_s'\left(\frac{t}{s}\right) - \left( \frac{s_1}{s} - \frac{\lfloor sx \rfloor + 1}{s}\right)\cdot \tilde{f}_s'(x) \right| \\
    &= \left| \sum_{t=\lfloor sx \rfloor + 1}^{s_1-1} \frac{1}{s}  \tilde{f}_s'\left(\frac{t}{s}\right) - \sum_{t=\lfloor sx \rfloor + 1}^{s_1-1} \frac{1}{s} \tilde{f}_s'(x) \right| \\
    &= \frac{1}{s} \left| \sum_{t = \lfloor sx \rfloor + 1}^{s_1 - 1} \left( \tilde{f}_s'\left( \frac{t}{s} \right) - \tilde{f}_s'(x) \right) \right| \\
    &\leq \frac{1}{s} \sum_{t = \lceil sx \rceil}^{s_1 - 1} \left| \tilde{f}_s'\left( \frac{t}{s} \right) - \tilde{f}_s'(x) \right|.
\end{align*}
Using the Lipschitz-like property from \eqref{eq:fs_prime_lipschitz_like}:
\begin{align*}
    \left| f_s\left(\frac{s_1}{s}\right) - \tilde{f}_s(x) - \tilde{f}_s'(x) \cdot \left(\frac{s_1}{s} - x\right) \right| &\leq \frac{1}{s} \sum_{t = \lfloor sx \rfloor + 1}^{s_1 - 1} \frac{C}{s^{\frac{1}{2} + \upsilon}} |t - \lfloor sx \rfloor| \\
    &\leq \frac{C}{s^{\frac{3}{2} + \upsilon}} \sum_{t = \lfloor sx \rfloor + 1}^{s_1 - 1} |t - \lfloor sx \rfloor| \\
    &\lesssim \frac{1}{s^{\frac{3}{2}+\upsilon}} (s_1 - sx)^2 \\
    &\lesssim s^{\frac{1}{2}-\upsilon} \left( \frac{s_1}{s} - x \right)^2.
\end{align*}
The case $s_1 < sx$ follows similarly. 
\end{proof}

\subsection{Proof of Theorem~\ref{theorem::computationalcomplexity}} \label{sec::proof_computational_complexity}

Algorithm~\ref{alg::gsv} employs a threshold, $\bar{s}$, to choose between two methods for estimating $\mathcal{T}(s)$:
\begin{itemize}
    \item For $s < \bar{s}$, compute $\hat{\mathcal{T}}(s)$ by directly approximating its definition formula \eqref{eqn::def::T(s)}, where each $\mu$ term is estimated by subsampling according to its average form, as in \eqref{eqn::mu_approximation_empirical}.
    \item For $s \ge \bar{s}$, compute $\hat{\mathcal{T}}(s)$ using the fast approximation formula \eqref{eqn::calT_approximation}, as per Theorem~\ref{proposition::calT}; the $\Delta\mu$ term in this formula can be efficiently approximated by the paired Monte Carlo estimator elaborated in \eqref{eqn::Delta_mu_approximation_empirical}.
\end{itemize}

To analyze the overall error and derive the computational complexity, we first establish Lemma~\ref{lemma::calT_hoeffding_type_error_bound}. This lemma bounds the total approximation error $|\hat{\mathcal{T}}_{\rm sum} - \mathcal{T}_{\rm sum}|$ by combining concentration bounds for the Monte Carlo estimators $\hat{\mu}_{m_1}$ and $\widehat{\Delta\mu}_{m_2}$ with the approximation error from Theorem~\ref{proposition::calT}.

\begin{restatable}{lemma}{LemmaCalTHoeffdingTypeErrorBound}(Error Bound for $\hat{\mathcal{T}}_{\rm sum}$)
    \label{lemma::calT_hoeffding_type_error_bound}
    Let $\mathcal{T}_{\rm sum} = \sum_{s=1}^{n-1} \mathcal{T}(s)$ and $\hat{\mathcal{T}}_{\rm sum} = \sum_{s=1}^{n-1} \hat{\mathcal{T}}(s)$ be the estimate computed by Algorithm~\ref{alg::gsv} using threshold $\bar{s}$ and sample sizes $m_1, m_2$. Under Assumptions~\ref{assump::utility_boundedness} and \ref{assump::utility_stability}, and assuming that the utility function $U$ is $\beta(s)$-deletion stable, , for any $\delta \in (0,1)$, with probability at least $1-\delta$:
    \begin{equation*}
        |\hat{\mathcal{T}}_{\rm sum} - \mathcal{T}_{\rm sum}| \lesssim \bar{s} \sqrt{ \frac{\log(n / \delta)}{m_1} } + \alpha_0(1-\alpha_0)\sqrt{ \frac{\log(n / \delta)}{m_2} } \sum_{s = \bar{s}}^{n - 1} \beta(s) + \alpha_0(1 - \alpha_0) O(\bar{s}^{-\upsilon}).
    \end{equation*}
\end{restatable}

Building on this error bound, Lemma~\ref{lemma::complexity_for_T_sum_detailed} establishes the computational cost for achieving an $(\epsilon, \delta)$-approximation of the sum $\mathcal{T}_{\rm sum}$.

\begin{restatable}{lemma}{LemmaComplexityForTSumDetailed}(Computational Cost for $(\epsilon,\delta)$-Approximation of $\mathcal{T}_{\rm sum}$)
    \label{lemma::complexity_for_T_sum_detailed}
    To achieve an $(\epsilon, \delta)$-approximation of $\mathcal{T}_{\rm sum} = \sum_{s=1}^{n-1} \mathcal{T}(s)$ (i.e., ensuring $|\hat{\mathcal{T}}_{\rm sum} - \mathcal{T}_{\rm sum}| \le \epsilon$ with probability at least $1-\delta$) using Algorithm~\ref{alg::gsv}, with parameters $m_1, m_2, \bar{s}$ chosen optimally based on Lemma~\ref{lemma::calT_hoeffding_type_error_bound}, the total number of utility function evaluations is:
    \begin{equation*}
        O\left(\epsilon^{-\frac{4 + 2\upsilon}{\upsilon}} \log(n / \delta) + n \left[1 + \epsilon^{-2} (\alpha_0(1 - \alpha_0))^{2} \left( \sum_{s=\bar{s}}^{n - 1} \beta(s) \right)^2 \log (n/\delta)\right] \right).
    \end{equation*}
\end{restatable}

With the computational cost for approximating $\mathcal{T}_{\rm sum} = \sum_{s=1}^{n-1} \mathcal{T}(s)$ established in Lemma~\ref{lemma::complexity_for_T_sum_detailed}, we are now equipped to prove our main complexity result for estimating $\operatorname{FGSV}(S_0)$. Theorem~\ref{theorem::computationalcomplexity} specifies this complexity under the condition that the deletion stability parameter $\beta(s)$ decays as $O(1/s)$.

\begin{proof}[Proof of Theorem~\ref{theorem::computationalcomplexity}]

By Lemma~\ref{lemma::group_shapley_closed_form_exact} in the main paper, we have $\operatorname{FGSV}(S_0) = G_0 + \mathcal{T}_{\rm sum}$, where $G_0 = \frac{s_0}{n} [U([n]) - U(\varnothing)]$ and $\mathcal{T}_{\rm sum} = \sum_{s=1}^{n-1} \mathcal{T}(s)$. 
The first term $G_0$ requires only two utility evaluations. 
We thus focus on the computational complexity in approximating the second term $\mathcal{T}_{\rm sum}$.
To analyze this term, we plug the theorem's assumption that $\beta(s) = O(1/s)$ into Lemma~\ref{lemma::complexity_for_T_sum_detailed}.
Specifically, this assumption implies that $\sum_{s=\bar{s}}^{n-1} \beta(s) = \sum_{s=\bar{s}}^{n-1} O(1/s) = O(\log n)$.
Therefore, from Lemma~\ref{lemma::calT_hoeffding_type_error_bound}, we have
\[
\bigl|\widehat{\mathrm{FGSV}}(S_0)-\mathrm{FGSV}(S_0)\bigr| \lesssim \bar s\sqrt{\frac{\log(n/\delta)}{m_1}}  +  \alpha_0(1-\alpha_0)\sqrt{\frac{\log(n/\delta)}{m_2}} \log n + \alpha_0(1-\alpha_0)\bar s^{-\upsilon}.
\]
Also, the number of utility evaluations, as in the displayed formula in Lemma~\ref{lemma::complexity_for_T_sum_detailed}, becomes:
\begin{align*}
    N_{\rm eval} &= O\left( \epsilon^{-\frac{4 + 2\upsilon}{\upsilon}} \log(n / \delta)  + n \left[1 + \epsilon^{-2} (\alpha_0(1 - \alpha_0))^{2} (\log n)^2 \log (n/\delta)\right] \right).
\end{align*}

While treating $\epsilon, \delta,$ and $\upsilon$ as constants, the second term, which is $O\left(n \left[1 + (\alpha_0(1-\alpha_0))^2 (\log n)^3\right]\right)$, dominates the first term (which has a lower power of $n$).
Thus, under the condition $\beta(s) = O(1/s)$, the total number of utility evaluations simplifies to $O\left(n \left[1 + (\alpha_0(1-\alpha_0))^2 (\log n)^3\right]\right)$.

\end{proof}

\subsubsection{Technical lemmas used in the proof of Theorem~\ref{theorem::computationalcomplexity}}

\LemmaCalTHoeffdingTypeErrorBound*

\begin{proof}[Proof of Lemma~\ref{lemma::calT_hoeffding_type_error_bound}]
The total approximation error is bounded using the triangle inequality:
\begin{equation*}
    |\hat{\mathcal{T}}_{\rm sum} - \mathcal{T}_{\rm sum}| \le \sum_{s=1}^{n-1} |\hat{\mathcal{T}}(s) - \mathcal{T}(s)| = \underbrace{\sum_{s=1}^{\bar{s}-1} |\hat{\mathcal{T}}(s) - \mathcal{T}(s)|}_{E_{\rm small}} + \underbrace{\sum_{s=\bar{s}}^{n-1} |\hat{\mathcal{T}}(s) - \mathcal{T}(s)|}_{E_{\rm large}}.
\end{equation*}

We first establish concentration bounds for our Monte Carlo estimators.
Under Assumption~\ref{assump::utility_boundedness}, $|U(S)| \le C$.
For $\hat{\mu}_{m_1}$ (Eq.~\eqref{eqn::mu_approximation_empirical}), which averages $m_1$ terms $U(S^{(j)})$ bounded in $[-C, C]$ (range $2C$), Hoeffding's inequality implies that for any specific $\mu(\frac{s_1}{s}; s, s_0, n)$ and any $\delta_1 > 0$, with probability at least $1-\delta_1$:
\begin{equation} \label{eqn::hoeffding_mu_proof}
    \left| \hat{\mu}_{m_1}\left( \frac{s_1}{s}; s, s_0, n \right) - \mu\left( \frac{s_1}{s}; s, s_0, n \right) \right| \le C \sqrt{\frac{2\log(2/\delta_1)}{m_1}} \asymp \sqrt{\frac{\log(1/\delta_1)}{m_1}}.
\end{equation}
For $\widehat{\Delta\mu}_{m_2}$ (Eq.~\eqref{eqn::Delta_mu_approximation_empirical}), it averages terms $U(S^{(j)} \cup \{i_1^{(j)}\}) - U(S^{(j)} \cup \{i_2^{(j)}\})$. By Definition~\ref{definition::deletion_stability}, these terms are bounded in $[-\beta(s), \beta(s)]$, for any specific $\Delta\mu(\frac{s_1^*}{s}; s, s_0, n)$ and $\delta_2 > 0$, with probability at least $1-\delta_2$:
\begin{equation} \label{eqn::hoeffding_deltamu_proof}
    \left| \widehat{\Delta\mu}_{m_2}\left( \frac{s_1^*}{s}; s, s_0, n \right) - \Delta\mu\left( \frac{s_1^*}{s}; s, s_0, n \right) \right| \lesssim \beta(s) \sqrt{ \frac{\log(1 / \delta_2)}{m_2} }.
\end{equation}

\textbf{Bounding $E_{\rm small}$ (Error for $s < \bar{s}$):}
For each $s < \bar{s}$, the estimate $\hat{\mathcal{T}}(s)$ is computed using $\hat{\mu}_{m_1}$ within the definition of $\mathcal{T}(s)$ in~\eqref{eqn::def::T(s)}. This definition involves a sum over approximately $O(s)$ different values of $s_1$, each requiring an estimate $\hat{\mu}_{m_1}(\frac{s_1}{s}; s, s_0, n)$. The coefficient multiplying each $\mu(\frac{s_1}{s}; s, s_0, n)$ term ($\operatorname{coeff}(s, s_1) = \frac{n}{s(n-s)}(s_1 - s\alpha_0)$) is $O(1)$, as established by analyzing its equivalent forms:
\begin{equation*}
    \text{Form 1: } \frac{n}{n - s}\left( \frac{s_1}{s} - \alpha_0\right) \quad \text{and} \quad \text{Form 2: } \frac{n}{s}\left( \alpha_0 - \frac{s_0 - s_1}{n - s}\right).
\end{equation*}
When $s \leq n/2$, Form 1 is bounded by $2 \cdot 1 = 2$, since $n/(n - s) \leq 2$ and $|\frac{s_1}{s} - \alpha_0| \leq 1$. When $s > n/2$, Form 2 is bounded by $2 \cdot 1 = 2$, since $n/s < 2$ and $|\alpha_0 - \frac{s_0 - s_1}{n - s}|$ can also be shown to be at most $1$ under the valid range of $s_1$. Thus, the coefficient is $O(1)$.

The error $|\hat{\mathcal{T}}(s) - \mathcal{T}(s)|$ for a given $s$ arises from propagating the errors from the $O(s)$ individual $\hat{\mu}_{m_1}$ estimations. The total number of distinct $\mu(\frac{s_1}{s}; s, s_0, n)$ terms that need to be estimated across all $s < \bar{s}$ is $N_{small\_ests} = \sum_{s=1}^{\bar{s}-1} O(s) = O(\bar{s}^2)$.
To ensure all these $N_{small\_ests}$ estimations are accurate simultaneously with high probability, we apply a union bound. We set the failure probability for each individual $\hat{\mu}_{m_1}$ estimation in \eqref{eqn::hoeffding_mu_proof} to $\delta_1 = \delta / (2N_{small\_ests})$.
Consequently, with probability at least $1-\delta/2$, each $|\hat{\mu}_{m_1} - \mu| \lesssim \sqrt{\log(N_{small\_ests}/\delta)/m_1}$.
The error for a single $\hat{\mathcal{T}}(s)$ is bounded by $\sum_{s_1} \pr(\boldsymbol{s_1}=s_1) | \text{coeff}(s,s_1) | |\hat{\mu}_{m_1} - \mu| \lesssim O(1) \sqrt{\log(N_{small\_ests}/\delta)/m_1}$, since $\sum_{s_1} \pr(\boldsymbol{s_1} = s_1)=1$ and coefficients are $O(1)$.
Summing these errors over $s=1, \dots, \bar{s}-1$, we yield that with probability $1 - \delta / 2$,
\begin{equation*}
    E_{\rm small} = \sum_{s=1}^{\bar{s}-1} |\hat{\mathcal{T}}(s) - \mathcal{T}(s)| \lesssim \sum_{s=1}^{\bar{s}-1} O(1) \sqrt{ \frac{\log(N_{small\_ests} / \delta)}{m_1} } \lesssim \bar{s} \sqrt{ \frac{\log(\bar{s}^2 / \delta)}{m_1} }.
\end{equation*}
For conciseness in the final error bound of the lemma, we replace $\log(\bar{s}^2/\delta)$ with the potentially looser but simpler $\log(n/\delta)$, yielding:
$E_{\rm small} \lesssim \bar{s} \sqrt{ \log(n / \delta) / m_1 }$.

\textbf{Bounding $E_{\rm large}$ (Error for $s \ge \bar{s}$):}
For $s \ge \bar{s}$, $\hat{\mathcal{T}}(s)$ is computed using the approximation from Proposition~\ref{proposition::calT} (Eq.~\eqref{eqn::calT_approximation}) with $\widehat{\Delta\mu}_{m_2}$. The error has two parts:
\begin{align*}
    |\hat{\mathcal{T}}(s) - \mathcal{T}(s)| &\le \left| \frac{n\alpha_0(1 - \alpha_0)}{n-1} \left( \widehat{\Delta\mu}_{m_2}\left(\frac{s_1^*}{s}; s, s_0, n\right) - \Delta\mu \left(\frac{s_1^*}{s}; s, s_0, n\right)\right) \right| \\
    & \quad + \left| \frac{n\alpha_0(1 - \alpha_0)}{n-1} O(s^{-(1+\upsilon)}) \right| \\
    &\lesssim \alpha_0(1-\alpha_0) \left| \widehat{\Delta\mu}_{m_2}\left(\frac{s_1^*}{s}; s, s_0, n\right) - \Delta\mu\left(\frac{s_1^*}{s}; s, s_0, n\right) \right| + \alpha_0(1-\alpha_0) O(s^{-(1+\upsilon)}).
\end{align*}
Let $N_{large\_ests} = n-\bar{s}+1 \approx n$. Using a union bound for the $\widehat{\Delta\mu}_{m_2}$ estimates (setting $\delta_2 = \delta / (2N_{large\_ests})$), with probability at least $1-\delta/2$:
\begin{align*}
    E_{\rm large} &= \sum_{s=\bar{s}}^{n-1} |\hat{\mathcal{T}}(s) - \mathcal{T}(s)| \\
    &\lesssim \alpha_0(1-\alpha_0) \sum_{s=\bar{s}}^{n-1} \beta(s) \sqrt{\frac{\log(N_{large\_ests}/\delta)}{m_2}} + \alpha_0(1-\alpha_0) \sum_{s=\bar{s}}^{n-1} O(s^{-(1+\upsilon)}) \\
    &\lesssim \alpha_0(1-\alpha_0) \sqrt{\frac{\log(n/\delta)}{m_2}} \sum_{s=\bar{s}}^{n-1} \beta(s) + \alpha_0(1-\alpha_0) O(\bar{s}^{-\upsilon}),
\end{align*}
where the last inequality uses the fact $\sum_{s=\bar{s}}^{n-1} s^{-(1+\upsilon)} = O(\bar{s}^{-\upsilon})$ for $\upsilon > 0$.

\textbf{Total Error:}
Combining $E_{\rm small}$ and $E_{\rm large}$, with overall probability at least $1-\delta$ (by union bound on the two failure probabilities $\delta/2$):
\begin{equation*}
    |\hat{\mathcal{T}}_{\rm sum} - \mathcal{T}_{\rm sum}| \lesssim \bar{s} \sqrt{ \frac{\log(n / \delta)}{m_1} } + \alpha_0(1-\alpha_0)\sqrt{ \frac{\log(n / \delta)}{m_2} } \sum_{s = \bar{s}}^{n - 1} \beta(s) + \alpha_0(1 - \alpha_0) O(\bar{s}^{-\upsilon}).
\end{equation*}
This completes the proof.
\end{proof}

\LemmaComplexityForTSumDetailed*

\begin{proof}
The total number of utility evaluations $N_{\rm eval}$ in Algorithm~\ref{alg::gsv} is approximately $O\left( \bar{s}^2 m_1 + (n-\bar{s}) m_2 \right)$, which can be simplified to $O\left( \bar{s}^2 m_1 + n m_2 \right)$ as $\bar{s} \le n$.

To achieve the target error $|\hat{\mathcal{T}}_{\rm sum} - \mathcal{T}_{\rm sum}| \le \epsilon$ with probability at least $1-\delta$, we refer to the error bound in Lemma~\ref{lemma::calT_hoeffding_type_error_bound}:
\begin{equation*}
    |\hat{\mathcal{T}}_{\rm sum} - \mathcal{T}_{\rm sum}| \lesssim \underbrace{\bar{s} \sqrt{ \frac{\log(n / \delta)}{m_1} }}_{\text{Term A}} + \underbrace{\alpha_0(1-\alpha_0)\sqrt{ \frac{\log(n / \delta)}{m_2} } \sum_{s = \bar{s}}^{n - 1} \beta(s)}_{\text{Term B}} + \underbrace{\alpha_0(1 - \alpha_0) O(\bar{s}^{-\upsilon})}_{\text{Term C}}.
\end{equation*}

We set parameters such that each dominant error component contributing to the bound in Lemma~\ref{lemma::calT_hoeffding_type_error_bound} is $O(\epsilon)$. This involves balancing Term A,  Term B, and Term C. The choices for $m_1, m_2,$ and $\bar{s}$ that achieve this balance and ensure the total error is $O(\epsilon)$ are:
\begin{align*}
    \bar{s} &\asymp  \epsilon^{- \frac{1}{\upsilon}} , \\
    m_1 &\asymp  \epsilon^{-\frac{2+2\upsilon}{\upsilon}}  \log(n/\delta), \\
    m_2 &\asymp \max \left\{1\;, \;   \epsilon^{-2} (\alpha_0(1 - \alpha_0))^2 \left( \sum_{s = \bar{s}}^{n - 1} \beta(s) \right)^2 \log(n/\delta)\right\}.
\end{align*}

Substituting these choices of $m_1, m_2,$ and $\bar{s}$ into the expression for $N_{\rm eval} = O(\bar{s}^2 m_1 + n m_2)$, we yield
\begin{align*}
    N_{\text{eval}} &= O\left( \epsilon^{-\frac{4 + 2\upsilon}{\upsilon}} \log(n / \delta) + \max \left\{ n\epsilon^{-2} (\alpha_0(1 - \alpha_0))^2 \left( \sum_{s = \bar{s}}^{n - 1} \beta(s) \right)^2 \log(n/\delta) \; , \; n \right\}\right) \\
    &= O\left(\epsilon^{-\frac{4 + 2\upsilon}{\upsilon}} \log(n / \delta) + n \left[1 + \epsilon^{-2} (\alpha_0(1 - \alpha_0))^{2} \left( \sum_{s=\bar{s}}^{n - 1} \beta(s) \right)^2 \log (n/\delta)\right] \right),
\end{align*}
which completes the proof.
\end{proof}

\subsection{Proof of Proposition~\ref{proposition::SGD_final_prop}}

Let $\mathcal{S} = \{z_i^\mathcal{S}\}_{i=1}^s \subset \mathcal{Z}^s$ denote the base training set, and let $z_1, z_1', z_2, z_2' \in \mathcal{Z}$ be four additional data points. We construct four augmented training data sequences, where the $i$-th element of a sequence $z^{(j,k)}$ is denoted $z_i^{(j,k)}$ and defined as follows:
\begin{align*}
  z^{(1,1)}_i &= 
  \begin{cases}
    z_i^\mathcal{S}, & 1 \leq i \leq s, \\
    z_1, & i = s+1, \\
    z_1', & i = s+2,
  \end{cases}
  &
  z^{(1,2)}_i &= 
  \begin{cases}
    z_i^\mathcal{S}, & 1 \leq i \leq s, \\
    z_1, & i = s+1, \\
    z_2, & i = s+2,
  \end{cases}
  \\
  z^{(2,1)}_i &= 
  \begin{cases}
    z_i^\mathcal{S}, & 1 \leq i \leq s, \\
    z_2', & i = s+1, \\
    z_1', & i = s+2,
  \end{cases}
  &
  z^{(2,2)}_i &= 
  \begin{cases}
    z_i^\mathcal{S}, & 1 \leq i \leq s, \\
    z_2', & i = s+1, \\
    z_2, & i = s+2.
  \end{cases}
\end{align*}
By construction, the first $s$ entries (the base set $\mathcal{S}$) are shared across all four sequences. The differences between the sequences are confined to the $(s+1)$-th and $(s+2)$-th positions. Specifically:
\begin{itemize}
    \item For the $(s+1)$-th position: $z_{s+1}^{(1,1)} = z_{s+1}^{(1,2)} (=z_1)$, and $z_{s+1}^{(2,1)} = z_{s+1}^{(2,2)} (=z_2')$.
    \item For the $(s+2)$-th position: $z_{s+2}^{(1,1)} = z_{s+2}^{(2,1)} (=z_1')$, and $z_{s+2}^{(1,2)} = z_{s+2}^{(2,2)} (=z_2)$.
\end{itemize}
We introduce the notation for the SGD trajectories \textbf{ under the same random seed}. 
Fixing a random seed implies a shared sequence of mini-batch index sets $(I_1, \dots, I_T)$ for training. The corresponding SGD trajectories are denoted $w_t^{(j,k)}$ ($j, k \in \{1, 2\}$), representing the model parameters after $t$ SGD steps. Specifically, they can be represented as:
\begin{equation}
\begin{aligned}
  w_t^{(1,1)} &= w_t\left(z_{I_1}^{(1,1)}, \dots, z_{I_t}^{(1,1)}\right), \\
  w_t^{(1,2)} &= w_t\left(z_{I_1}^{(1,2)}, \dots, z_{I_t}^{(1,2)}\right), \\
  w_t^{(2,1)} &= w_t\left(z_{I_1}^{(2,1)}, \dots, z_{I_t}^{(2,1)}\right), \\
  w_t^{(2,2)} &= w_t\left(z_{I_1}^{(2,2)}, \dots, z_{I_t}^{(2,2)}\right).
\end{aligned}
\label{eqn::SGD_same_seed_training_param}
\end{equation}
Here, $w_t(\cdot)$ is a function that takes a sequence of mini-batch data points as inputs and outputs a parameter vector; $z_I^{(j,k)} = \{z_i^{(j,k)} : i \in I\}$ denotes a mini-batch constructed from the data sequence $z^{(j,k)}$ using the index set $I$. The arguments $z_{I_\tau}^{(j,k)}$ in Eq.~\eqref{eqn::SGD_same_seed_training_param} thus represent the specific mini-batches used at each step $\tau \in [1,t]$. The variations among these trajectories arise only from the different data points at indices $s+1$ and $s+2$ within their respective sequences. This setup focuses on analyzing how these specific data perturbations affect the SGD iterations.

In this section, for clarity in tracking the dependence on the number of training steps $T$, we use $U(\cdot; T)$ to denote the expected utility after $T$ SGD steps, as opposed to $U(\cdot)$ used in the main text. 

To demonstrate the second-order stability of the utility function $U(\cdot; T)$, as stated in Proposition~\ref{proposition::SGD_final_prop}, our approach relies on key intermediate results that connect utility differences to parameter stability. Lemma~\ref{lemma::SGD_relationship_between_U_and_l2_dist} below establishes this crucial link.

\begin{restatable}{lemma}{LemmaSGDRelationhipBetweenUAndParams}
\label{lemma::SGD_relationship_between_U_and_l2_dist}
Let $\mathcal{S} \in \mathcal{Z}^s$ be a base dataset, and $z_1, z_1', z_2, z_2' \in \mathcal{Z}$ be additional data points. Recall that $U(\mathcal{X}; T)$ denotes the expected utility of a model trained on dataset $\mathcal{X}$ for $T$ SGD iterations. Then, under Assumption~\ref{assump::smoothness_for_SGD}, we have:
\begin{align*}
    &\left| U(\mathcal{S} \cup \{z_1, z_1'\}; T) - U(\mathcal{S} \cup \{z_1, z_2\}; T) - U(\mathcal{S} \cup \{z_1', z_2'\}; T) + U(\mathcal{S} \cup \{z_2, z_2'\}; T) \right| \\
    &\quad \lesssim \mathbb{E}_{I_1, \dots, I_T} \left\| w_T^{(1,1)} - w_T^{(1,2)} - w_T^{(2,1)} + w_T^{(2,2)} \right\| 
    + \max_{j_1, k_1, j_2, k_2 \in \{1,2\}} \mathbb{E}_{I_1, \dots, I_T} \left\| w_T^{(j_1, k_1)} - w_T^{(j_2, k_2)} \right\|^2,
\end{align*}
where $w_T^{(j, k)}$ denotes the final SGD iterate trained using the shared mini-batch sequence $(I_1, \dots, I_T)$ on the data sequence $z^{(j, k)}$ defined in Eq.~\eqref{eqn::SGD_same_seed_training_param}.
\end{restatable}

\begin{remark}
The utility terms $U(\mathcal{S} \cup \{\cdot, \cdot\}; T)$ on the left-hand side of the inequality in Lemma~\ref{lemma::SGD_relationship_between_U_and_l2_dist} are, by definition, expectations over independent SGD runs (i.e., each utility term averages over its own random mini-batch selections). In contrast, the parameter differences on the right-hand side are for iterates $w_T^{(j,k)}$ that are explicitly trained using a common, shared sequence of mini-batch draws $(I_1, \dots, I_T)$. This formulation is key, as it leverages the shared randomness for a tighter analysis of parameter stability.
\end{remark}

Having established the connection between the utility's stability and the SGD iterates via Lemma~\ref{lemma::SGD_relationship_between_U_and_l2_dist}, the next crucial step is to quantify the magnitude of these iterate differences. Specifically, the two terms on the right-hand side of the inequality in Lemma~\ref{lemma::SGD_relationship_between_U_and_l2_dist}—the expected $L_2$-norm of the second-order parameter difference, $\mathbb{E}_{I_1, \dots, I_T} \| w_T^{(1,1)} - w_T^{(1,2)} - w_T^{(2,1)} + w_T^{(2,2)} \|$, and the maximum expected squared $L_2$-norm of pairwise parameter differences, $\max \mathbb{E}_{I_1, \dots, I_T} \left\| w_T^{(j_1, k_1)} - w_T^{(j_2, k_2)} \right\|^2$—need to be controlled. The following two lemmas provide the necessary explicit upper bounds for these quantities.

\begin{restatable}{lemma}{LemmaSGDSecondOrderDist} \label{lemma::SGD_second_order_dist}
Under Assumption~\ref{assump::smoothness_for_SGD}, and choosing step size $\alpha_t \leq \frac{c}{t}$, we have:
\begin{align*}
    \mathbb{E}_{I_1, \dots, I_t} \left\| w_t^{(1,1)} - w_t^{(1,2)} - w_t^{(2,1)} + w_t^{(2,2)} \right\| 
    \leq \frac{24\rho L^2}{(s+2)\beta} \left( \frac{1}{(s+2)\beta^2} + \frac{2c^2}{m} \right) t^{2c\beta} 
    + \frac{8cL}{(s+2)^2} t^{c\beta} \log t.
\end{align*}
\end{restatable}

\begin{restatable}{lemma}{LemmaSGDFirstOrderDistSquaredFinal}
\label{lemma::SGD_first_order_l2_squared_final}
Under Assumption~\ref{assump::smoothness_for_SGD}, if the learning rate satisfies $\alpha_t \leq \frac{c}{t}$, then for any $j_1, k_1, j_2, k_2 \in \{1, 2\}$:
\begin{equation*}
    \mathbb{E}_{I_1, \dots, I_t} \left\| w_t^{(j_1, k_1)} - w_t^{(j_2, k_2)} \right\|^2 
    \leq \frac{16L^2}{s+2} \, t^{2c\beta} \left( \frac{1}{(s+2)\beta^2} + \frac{2c^2}{m} \right).
\end{equation*}
\end{restatable}

Equipped with these results, we are ready to prove Proposition~\ref{proposition::SGD_final_prop}.

\begin{proof}[Proof of Proposition~\ref{proposition::SGD_final_prop}]
By Lemma~\ref{lemma::SGD_second_order_dist} and Lemma~\ref{lemma::SGD_first_order_l2_squared_final}, we have
\begin{align*}
    \mathbb{E}_{I_1, \dots, I_T} \left\| w_T^{(1,1)} - w_T^{(1,2)} - w_T^{(2,1)} + w_T^{(2,2)} \right\| 
    &\lesssim \frac{T^{2c\beta}}{s^2} + \frac{c^2 T^{2c\beta}}{s m} + \frac{c T^{c\beta} \log T}{s^2}, \\
    \max_{j_1, k_1, j_2, k_2} \mathbb{E}_{I_1, \dots, I_T} \left\| w_T^{(j_1, k_1)} - w_T^{(j_2, k_2)}\right\| &\lesssim \frac{T^{2c\beta}}{s^2} + \frac{c^2 T^{2c\beta}}{s m}
\end{align*}
Using the substitutions $c \asymp s^{-\tau_1}$, $m \asymp s^{\tau_2}$, and $T \asymp s^{\tau_3}$, we compute
\begin{align*}
    \frac{T^{2c\beta}}{s^2} 
    &\asymp \frac{1}{s^2} \cdot \exp\left( 2c\beta \log T \right)
    = \frac{1}{s^2} \cdot \exp\left( 2\beta s^{-\tau_1} \cdot \tau_3 \log s \right), \\
    \frac{c^2 T^{2c\beta}}{s m} 
    &\asymp \frac{1}{s^{1 + 2\tau_1 + \tau_2}} \cdot \exp\left( 2\beta s^{-\tau_1} \cdot \tau_3 \log s \right), \\
    \frac{c \log T \cdot T^{c\beta}}{s^2} 
    &\asymp \frac{s^{-\tau_1} \cdot \tau_3 \log s}{s^2} \cdot \exp\left( \beta s^{-\tau_1} \cdot \tau_3 \log s \right).
\end{align*}
Using the fact that $\exp(\gamma s^{-\tau_1} \log s) = 1 + o(1)$ for any $\gamma, \tau_1 > 0$, we conclude:
\begin{align*}
    \mathbb{E}_{I_1, \dots, I_T} \left\| w_T^{(1,1)} - w_T^{(1,2)} - w_T^{(2,1)} + w_T^{(2,2)} \right\| 
    &\lesssim \frac{1}{s^2} + \frac{1}{s^{1 + 2\tau_1 + \tau_2}}, \\
    \max_{j_1, k_1, j_2, k_2} \mathbb{E}_{I_1, \dots, I_T} \left\| w_T^{(j_1, k_1)} - w_T^{(j_2, k_2)}\right\| &\lesssim \frac{1}{s^2} + \frac{1}{s^{1 + 2\tau_1 + \tau_2}}.
\end{align*}
Since $\upsilon = 2\tau_1 + \tau_2 - \frac{1}{2}$, then
\begin{align*}
    \mathbb{E}_{I_1, \dots, I_T} \left\| w_T^{(1,1)} - w_T^{(1,2)} - w_T^{(2,1)} + w_T^{(2,2)} \right\| 
    &\lesssim \frac{1}{s^{\frac{3}{2} +\min\{\frac{1}{2} ,  \upsilon\}}}, \\
    \max_{j_1, k_1, j_2, k_2} \mathbb{E}_{I_1, \dots, I_T} \left\| w_T^{(j_1, k_1)} - w_T^{(j_2, k_2)}\right\| &\lesssim \frac{1}{s^{\frac{3}{2} +\min\{\frac{1}{2} ,  \upsilon\}}}.
\end{align*}
Finally, applying Lemma~\ref{lemma::SGD_relationship_between_U_and_l2_dist} yields the desired second-order stability bound.
\end{proof}

\subsubsection{Preliminary first-order stability results and proof of Lemma~\ref{lemma::SGD_first_order_l2_squared_final}}

To build towards the proof of Lemma~\ref{lemma::SGD_second_order_dist}, we first analyze first-order algorithmic stability. This serves to introduce key proof techniques in a simplified setting and provides intermediate results that are essential for the subsequent second-order analysis. The ideas presented here are closely aligned with those in \citet{hardt2016train}.

Consider a base dataset $\mathcal{S} = \{z_i^\mathcal{S}\}_{i=1}^s \subset \mathcal{Z}^s$. We introduce two additional, distinct data points $z_a, z_b \in \mathcal{Z}$. We then define two augmented data sequences, $z^{(a)}$ and $z^{(b)}$, each of length $s+1$:
\begin{align*}
  z^{(a)}_i &= 
  \begin{cases}
    z_i^\mathcal{S}, & 1 \leq i \leq s, \\
    z_a, & i = s+1, 
  \end{cases}
  &
  z^{(b)}_i &= 
  \begin{cases}
    z_i^\mathcal{S}, & 1 \leq i \leq s, \\
    z_b, & i = s+1.
  \end{cases}
\end{align*}
These sequences $z^{(a)}$ and $z^{(b)}$ differ only in their $(s+1)$-th element. Let $w_t^{(a)}$ and $w_t^{(b)}$ denote the iterates produced by SGD after $t$ steps when trained on their respective data sequences, using a shared sequence of mini-batch indices $(I_1, \dots, I_t)$. Consistent with Eq.~\eqref{eqn::SGD_same_seed_training_param}, these are:
\begin{align*}
  w_t^{(a)} &= w_t\left(z_{I_1}^{(a)}, \dots, z_{I_t}^{(a)}\right), \\
  w_t^{(b)} &= w_t\left(z_{I_1}^{(b)}, \dots, z_{I_t}^{(b)}\right).
\end{align*}
We now state two standard first-order stability results concerning such iterates.

\begin{restatable}{lemma}{LemmaSGDFirstOrderDist}
\label{lemma::SGD_first_order_l2_dist}
Under Assumption~\ref{assump::smoothness_for_SGD}, if the learning rate satisfies $\alpha_t \leq \frac{c}{t}$, then for iterates $w_t^{(a)}$ and $w_t^{(b)}$ (trained on data sequences $z^{(a)}$ and $z^{(b)}$ of effective size $N=s+1$ that differ by one point):
\begin{equation*}
    \mathbb{E}_{I_1, \dots, I_t} \left\| w_t^{(a)} - w_t^{(b)} \right\| \leq \frac{2L}{\beta N} \, t^{c\beta}.
\end{equation*}
\end{restatable}

\begin{restatable}{lemma}{LemmaSGDFirstOrderDistSquared}
\label{lemma::SGD_first_order_l2_squared}
Under Assumption~\ref{assump::smoothness_for_SGD}, if the learning rate satisfies $\alpha_t \leq \frac{c}{t}$, then for iterates $w_t^{(a)}$ and $w_t^{(b)}$ (trained on data sequences $z^{(a)}$ and $z^{(b)}$ of effective size $N=s+1$ that differ by one point):
\begin{equation*}
    \mathbb{E}_{I_1, \dots, I_t} \left\| w_t^{(a)} - w_t^{(b)} \right\|^2 
    \leq \frac{4L^2}{N} \, t^{2c\beta} \left( \frac{1}{N\beta^2} + \frac{2c^2}{m} \right).
\end{equation*}
\end{restatable}
Note: In the context of these lemmas, $N$ represents the total number of data points in the sequences being compared. When applying these to the iterates $w_t^{(j, k)}$ (which are trained on data sequences of length $s+2$), $N$ will correspond to $s+2$.

These foundational first-order stability results, particularly Lemma~\ref{lemma::SGD_first_order_l2_squared}, allow us to derive Lemma~\ref{lemma::SGD_first_order_l2_squared_final}. We restate Lemma~\ref{lemma::SGD_first_order_l2_squared_final} for clarity before its proof.

\LemmaSGDFirstOrderDistSquaredFinal*

\begin{proof}[Proof of Lemma~\ref{lemma::SGD_first_order_l2_squared_final}]
    The data sequences $z^{(j_1, k_1)}$ and $z^{(j_2, k_2)}$ (each of size $N=s+2$) differ at most at two positions (index $s+1$ and $s+2$). We can introduce an intermediate iterate, $w_t^{(j_1, k_2)}$.
    Using the triangle inequality and the property $\|x+y\|^2 \le 2(\|x\|^2+\|y\|^2)$:
    \begin{align*}
        &\mathbb{E}_{I_1, \dots, I_t} \left\| w_t^{(j_1, k_1)} - w_t^{(j_2, k_2)} \right\|^2 
        = \mathbb{E}_{I_1, \dots, I_t} \left\| (w_t^{(j_1, k_1)} - w_t^{(j_1, k_2)}) + (w_t^{(j_1, k_2)} -  w_t^{(j_2, k_2)}) \right\|^2 \\
        &\quad \leq 2 \left\{ \mathbb{E}_{I_1, \dots, I_t} \left\| w_t^{(j_1, k_1)} - w_t^{(j_1, k_2)} \right\|^2 + \mathbb{E}_{I_1, \dots, I_t} \left\| w_t^{(j_1, k_2)} -  w_t^{(j_2, k_2)} \right\|^2 \right\}.
    \end{align*}
    The term $\left\| w_t^{(j_1, k_1)} - w_t^{(j_1, k_2)} \right\|^2$ involves iterates from data sequences $z^{(j_1, k_1)}$ and $z^{(j_1, k_2)}$. These sequences share the same first $s+1$ points (namely, $z_1^\mathcal{S}, \dots, z_s^\mathcal{S}, z_{s+1}^{(j_1)}$) and differ only at the $(s+2)$-th position. This is a one-point difference in data sequences of effective size $N=s+2$.
    Similarly, the term $\left\| w_t^{(j_1, k_2)} - w_t^{(j_2, k_2)} \right\|^2$ involves iterates from data sequences $z^{(j_1, k_2)}$ and $z^{(j_2, k_2)}$. These sequences share the first $s$ points ($z_1^\mathcal{S}, \dots, z_s^\mathcal{S}$) and the $(s+2)$-th point ($z_{s+2}^{(k_2)}$), differing only at the $(s+1)$-th position. This is also a one-point difference in data sequences of effective size $N=s+2$.

    Therefore, applying Lemma~\ref{lemma::SGD_first_order_l2_squared} to each of these one-point difference terms (with $N=s+2$ as the effective dataset size) yields:
    \begin{align*}
        \mathbb{E}_{I_1, \dots, I_t} \left\| w_t^{(j_1, k_1)} - w_t^{(j_1, k_2)} \right\|^2 &\leq \frac{4L^2}{s+2} \, t^{2c\beta} \left( \frac{1}{(s+2)\beta^2} + \frac{2c^2}{m} \right), \\
        \mathbb{E}_{I_1, \dots, I_t} \left\| w_t^{(j_1, k_2)} - w_t^{(j_2, k_2)} \right\|^2 &\leq \frac{4L^2}{s+2} \, t^{2c\beta} \left( \frac{1}{(s+2)\beta^2} + \frac{2c^2}{m} \right).
    \end{align*}
    Substituting these into the inequality:
    \begin{align*}
        \mathbb{E}_{I_1, \dots, I_t} \left\| w_t^{(j_1, k_1)} - w_t^{(j_2, k_2)} \right\|^2 
        &\leq 2\times 2\times  \left\{ \frac{8L^2}{s+2} \, t^{2c\beta} \left( \frac{1}{(s+2)\beta^2} + \frac{2c^2}{m} \right) \right\} \\
        &= \frac{16L^2}{s+2} \, t^{2c\beta} \left( \frac{1}{(s+2)\beta^2} + \frac{2c^2}{m} \right).
    \end{align*}
    This completes the proof.
\end{proof}

We now turn to the second-order behavior of stochastic gradient methods. Analyzing this provides a more refined understanding of how small perturbations to the training dataset influence the final learned parameters. The analysis relies on the higher-order smoothness properties of the loss function $\ell(w;z)$ as detailed in Assumption~\ref{assump::smoothness_for_SGD}, particularly the Lipschitz continuity of its Hessian (second derivative with respect to $w$). The following lemma, which is a consequence of this Lipschitz Hessian property, serves as a key technical tool for bounding second-order differences of gradients.

\begin{restatable}{lemma}{LemmaSGDSmoothnessTool}
\label{lemma::SGD_smoothness_tool}
Under Assumption~\ref{assump::smoothness_for_SGD}, the loss function $\ell(w; z)$ satisfies the following second-order approximation bound for all $w, h \in \mathbb{R}^d$ and $z \in \mathcal{Z}$:
\begin{equation*}
    \left\| \nabla_w \ell(w + h; z) - \nabla_w \ell(w; z) - \nabla_w^2 \ell(w; z) h \right\| \leq \rho \|h\|^2.
\end{equation*}
\end{restatable}

With this tool, we now begin the proof of Lemma~\ref{lemma::SGD_second_order_dist}, which bounds the expected $L_2$-norm of the second-order difference of the SGD iterates. We restate Lemma~\ref{lemma::SGD_second_order_dist} for clarity before its proof.

\LemmaSGDSecondOrderDist*

\begin{proof}[Proof of Lemma~\ref{lemma::SGD_second_order_dist}]
From the SGD update rule, the difference between the second-order differences of the iterates at step $t$ and $t-1$ can be bounded:
\begin{align*}
    &\left\| w_t^{(1,1)} - w_t^{(1,2)} - w_t^{(2,1)} + w_t^{(2,2)} \right\|
    - \left\| w_{t-1}^{(1,1)} - w_{t-1}^{(1,2)} - w_{t-1}^{(2,1)} + w_{t-1}^{(2,2)} \right\| \\
    &\quad \leq \left\| \frac{\alpha_t}{m} \sum_{i \in I_t} \left[
        \nabla_w \ell(w_{t-1}^{(1,1)}; z_i^{(1,1)})
        - \nabla_w \ell(w_{t-1}^{(1,2)}; z_i^{(1,2)})
        - \nabla_w \ell(w_{t-1}^{(2,1)}; z_i^{(2,1)})
        + \nabla_w \ell(w_{t-1}^{(2,2)}; z_i^{(2,2)})
    \right] \right\|.
\end{align*}
Let's define the term inside the norm on the right-hand side, representing the sum of gradient differences for a given mini-batch $I_t$:
\[
\Delta(t-1, i) := 
    \nabla_w \ell(w_{t-1}^{(1,1)}; z_i^{(1,1)})
    - \nabla_w \ell(w_{t-1}^{(1,2)}; z_i^{(1,2)})
    - \nabla_w \ell(w_{t-1}^{(2,1)}; z_i^{(2,1)})
    + \nabla_w \ell(w_{t-1}^{(2,2)}; z_i^{(2,2)}).
\]
Taking the expectation with respect to the mini-batch sequences $I_1, \dots, I_t$, we have:
\begin{align}
    &\mathbb{E}_{I_1, \dots, I_t} \left\| w_t^{(1,1)} - w_t^{(1,2)} - w_t^{(2,1)} + w_t^{(2,2)} \right\|
    - \mathbb{E}_{I_1, \dots, I_{t-1}} \left\| w_{t-1}^{(1,1)} - w_{t-1}^{(1,2)} - w_{t-1}^{(2,1)} + w_{t-1}^{(2,2)} \right\| \nonumber \\
    &\quad = \mathbb{E}_{I_1, \dots, I_{t-1}} \left\{ \mathbb{E}_{I_t} \left[ \left\| w_t^{(1,1)} - w_t^{(1,2)} - w_t^{(2,1)} + w_t^{(2,2)} \right\| \right] - \left\| w_{t-1}^{(1,1)} - w_{t-1}^{(1,2)} - w_{t-1}^{(2,1)} + w_{t-1}^{(2,2)} \right\| \right\} \nonumber \\
    &\quad \leq \mathbb{E}_{I_1, \dots, I_{t-1}} \left\{ \mathbb{E}_{I_t} \left[ \left\|\frac{\alpha_t}{m} \sum_{i \in I_t} \Delta(t-1, i) \right\| \right] \right\} \nonumber \\
    &\quad \leq \mathbb{E}_{I_1, \dots, I_{t-1}} \left\{ \frac{\alpha_t}{m} \mathbb{E}_{I_t} \left[ \sum_{i \in I_t} \left\|\Delta(t-1, i)\right\| \right] \right\} \nonumber \\
    &\quad = \mathbb{E}_{I_1, \dots, I_{t-1}} \left\{ \frac{\alpha_t}{m} \sum_{i=1}^{s+2} \left\| \Delta(t-1, i) \right\| \cdot \mathbb{P}(z_j \in I_t) \right\} \nonumber  \\
    &\quad = \frac{\alpha_t}{s+2} \sum_{i=1}^{s+2} \mathbb{E}_{I_1, \dots, I_{t-1}} \left\|\Delta(t-1, i)\right\| \label{eqn::exp_diff_recursion_step}
\end{align}
In the last step, $\mathbb{P}(i \in I_t) = \frac{m}{s+2}$ because $I_t$ is sampled uniformly from the $s+2$ points. 

Now, we analyze the bounds for $\mathbb{E}_{I_1, \dots, I_{t-1}} \left\|\Delta(t-1, i) \right\|$.

Case 1: $i \leq s$.
For these points, $z_i^{(1,1)} = z_i^{(1,2)} = z_i^{(2,1)} = z_i^{(2,2)} = z_i^S$. Applying Lemma~\ref{lemma::SGD_smoothness_tool} (which relies on the Lipschitz Hessian property):
\begin{align*}
    \left\|\Delta(t-1, i)\right\| & \leq \left\| \nabla_w^2 \ell(w_{t-1}^{(1,1)}; z_i^S) \left[
        - (w_{t-1}^{(1,2)} - w_{t-1}^{(1,1)})
        - (w_{t-1}^{(2,1)} - w_{t-1}^{(1,1)})
        + (w_{t-1}^{(2,2)} - w_{t-1}^{(1,1)})
    \right] \right\| \\
    &\quad + \rho \|w_{t-1}^{(1,2)} - w_{t-1}^{(1,1)}\|^2
    + \rho \|w_{t-1}^{(2,1)} - w_{t-1}^{(1,1)}\|^2
    + \rho \|w_{t-1}^{(2,2)} - w_{t-1}^{(1,1)}\|^2 \\
    & \leq \beta \|w_{t-1}^{(1,1)} - w_{t-1}^{(1,2)} - w_{t-1}^{(2,1)} + w_{t-1}^{(2,2)}\| \\
    &\quad + \rho \|w_{t-1}^{(1,2)} - w_{t-1}^{(1,1)}\|^2
    + \rho \|w_{t-1}^{(2,1)} - w_{t-1}^{(1,1)}\|^2 \\
    &\quad + 2\rho \left( \|w_{t-1}^{(2,2)} - w_{t-1}^{(1,2)}\|^2 + \|w_{t-1}^{(1,2)} - w_{t-1}^{(1,1)}\|^2 \right).
\end{align*}
Taking the expectation and applying Lemma~\ref{lemma::SGD_first_order_l2_squared} for the squared difference terms:
\begin{align*}
    \mathbb{E}_{I_1, \dots, I_{t-1}}\left\| \Delta(t-1, i)\right\| & \leq \beta \cdot  \mathbb{E}_{I_1, \dots, I_{t-1}} \|w_{t-1}^{(1,1)} - w_{t-1}^{(1,2)} - w_{t-1}^{(2,1)} + w_{t-1}^{(2,2)}\| \\
    & \quad + \rho \cdot \mathbb{E}_{I_1, \dots, I_{t-1}} \|w_{t-1}^{(1,2)} - w_{t-1}^{(1,1)}\|^2
    + \rho  \cdot \mathbb{E}_{I_1, \dots, I_{t-1}} \|w_{t-1}^{(2,1)} - w_{t-1}^{(1,1)}\|^2 \\
    & \quad + 2\rho \left(  \mathbb{E}_{I_1, \dots, I_{t-1}} \|w_{t-1}^{(2,2)} - w_{t-1}^{(1,2)}\|^2 +  \mathbb{E}_{I_1, \dots, I_{t-1}}\|w_{t-1}^{(1,2)} - w_{t-1}^{(1,1)}\|^2 \right) \\
    & \leq \beta \cdot  \mathbb{E}_{I_1, \dots, I_{t-1}} \|w_{t-1}^{(1,1)} - w_{t-1}^{(1,2)} - w_{t-1}^{(2,1)} + w_{t-1}^{(2,2)}\| \\
    & \quad + 6\rho \cdot \frac{4L^2}{s+2} (t-1)^{2c\beta} \left( \frac{1}{(s+2)\beta^2} + \frac{2c^2}{m} \right).
\end{align*}

Case 2: $i = s+1$.
Here, $z_{s+1}^{(1,1)} = z_{s+1}^{(1,2)} (=z_1)$ and $z_{s+1}^{(2,1)} = z_{s+1}^{(2,2)} (=z_2')$.
\begin{align*}
    \Delta(t-1, s+1) &= \left\|
    (\nabla_w \ell(w_{t-1}^{(1,1)}; z_1) - \nabla_w \ell(w_{t-1}^{(1,2)}; z_1))
    - (\nabla_w \ell(w_{t-1}^{(2,1)}; z_2') - \nabla_w \ell(w_{t-1}^{(2,2)}; z_2'))
    \right\| \\
    &\leq \left\| \nabla_w \ell(w_{t-1}^{(1,1)}; z_1) - \nabla_w \ell(w_{t-1}^{(1,2)}; z_1) \right\|
    + \left\| \nabla_w \ell(w_{t-1}^{(2,1)}; z_2') - \nabla_w \ell(w_{t-1}^{(2,2)}; z_2') \right\| \\
    &\leq \beta \|w_{t-1}^{(1,1)} - w_{t-1}^{(1,2)}\| + \beta \|w_{t-1}^{(2,1)} - w_{t-1}^{(2,2)}\|.
\end{align*}
Taking expectation and using Lemma~\ref{lemma::SGD_first_order_l2_dist} (first-order stability for one-point difference):
\begin{align*}
    \mathbb{E}_{I_1, \dots, I_{t-1}} \left\{\Delta(t-1, s+1) \right\}
    &\leq \beta \cdot \mathbb{E}_{I_1, \dots, I_{t-1}}\|w_{t-1}^{(1,1)} - w_{t-1}^{(1,2)}\| + \beta \cdot \mathbb{E}_{I_1, \dots, I_{t-1}} \|w_{t-1}^{(2,1)} - w_{t-1}^{(2,2)}\| \\
    &\leq \beta \cdot \frac{2L}{\beta (s+2)} (t-1)^{c\beta} + \beta \cdot \frac{2L}{\beta (s+2)} (t-1)^{c\beta} \\
    &\leq \frac{4L}{s+2} t^{c\beta}.
\end{align*}
A similar bound holds for $i=s+2$.

Substituting these bounds into Eq.~\eqref{eqn::exp_diff_recursion_step}:
The sum over $i=1, \dots, s+2$ has $s$ terms from Case 1 and 2 terms from Case 2.
\begin{align*}
    &\mathbb{E}_{I_1, \dots, I_t} \left\| w_t^{(1,1)} - w_t^{(1,2)} - w_t^{(2,1)} + w_t^{(2,2)} \right\|
    - \mathbb{E}_{I_1, \dots, I_{t-1}} \left\| w_{t-1}^{(1,1)} - w_{t-1}^{(1,2)} - w_{t-1}^{(2,1)} + w_{t-1}^{(2,2)} \right\| \\
    & \quad \leq \frac{\alpha_t s}{s+2} \left[ \beta \mathbb{E}_{I_1, \dots, I_{t-1}} \|w_{t-1}^{(1,1)} - \dots + w_{t-1}^{(2,2)}\|
    + \frac{24\rho L^2}{s+2} t^{2c\beta} \left( \frac{1}{(s+2)\beta^2} + \frac{2c^2}{m} \right) \right] \\
    & \qquad + \frac{\alpha_t \cdot 2}{s+2} \left[ \frac{4L}{s+2} t^{c\beta} \right] \\
    & \quad \leq \alpha_t \beta \mathbb{E}_{I_1, \dots, I_{t-1}} \|w_{t-1}^{(1,1)} - w_{t-1}^{(1,2)} - w_{t-1}^{(2,1)} + w_{t-1}^{(2,2)}\| \quad (\text{since } s/(s+2) < 1) \\
    & \qquad + \alpha_t \frac{24\rho L^2}{s+2} t^{2c\beta} \left( \frac{1}{(s+2)\beta^2} + \frac{2c^2}{m} \right)
    + \alpha_t \frac{8L}{(s+2)^2} t^{c\beta}.
\end{align*}
Using $\alpha_t \leq c/t$:
\begin{align*}
    &\mathbb{E}_{I_1, \dots, I_t} \left\| w_t^{(1,1)} - w_t^{(1,2)} - w_t^{(2,1)} + w_t^{(2,2)} \right\| \\
    & \quad \leq \left(1 + \frac{c\beta}{t}\right) \mathbb{E}_{I_1, \dots, I_{t-1}} \|w_{t-1}^{(1,1)} - w_{t-1}^{(1,2)} - w_{t-1}^{(2,1)} + w_{t-1}^{(2,2)}\| \\
    & \qquad + \frac{c}{t} \left[ \frac{24\rho L^2}{s+2} t^{2c\beta} \left( \frac{1}{(s+2)\beta^2} + \frac{2c^2}{m} \right)
    + \frac{8L}{(s+2)^2} t^{c\beta} \right] \\
    & \quad \leq \exp\left( \frac{c\beta}{t} \right) \mathbb{E}_{I_1, \dots, I_{t-1}} \|w_{t-1}^{(1,1)} - w_{t-1}^{(1,2)} - w_{t-1}^{(2,1)} + w_{t-1}^{(2,2)}\| \\
    & \qquad + \frac{24c\rho L^2}{s+2} \left( \frac{1}{(s+2)\beta^2} + \frac{2c^2}{m} \right) t^{2c\beta - 1}
    + \frac{8cL}{(s+2)^2} t^{c\beta - 1}.
\end{align*}
Unrolling this recurrence from $t_0 = 1$ to $t$:
\begin{align*}
    &\mathbb{E}_{I_1, \dots, I_t} \left\| w_t^{(1,1)} - w_t^{(1,2)} - w_t^{(2,1)} + w_t^{(2,2)} \right\| \\
    &\quad \leq \sum_{t_0=1}^t \left[
        \frac{24c \rho L^2}{s+2} \left( \frac{1}{(s+2)\beta^2} + \frac{2c^2}{m} \right) t_0^{2c\beta - 1}
        + \frac{8cL}{(s+2)^2} t_0^{c\beta - 1}
    \right] \prod_{k=t_0+1}^t \exp\left( \frac{c\beta}{k} \right) \\
    &\quad = \sum_{t_0=1}^t \left[
        \frac{24c \rho L^2}{s+2} \left( \frac{1}{(s+2)\beta^2} + \frac{2c^2}{m} \right) t_0^{2c\beta - 1}
        + \frac{8cL}{(s+2)^2} t_0^{c\beta - 1}
    \right] \exp\left( \sum_{k=t_0+1}^t \frac{c\beta}{k} \right).
\end{align*}
Using the approximation $\sum_{k=t_0+1}^t \frac{1}{k} \approx \log(t/t_0)$:
\begin{align*}
    &\mathbb{E}_{I_1, \dots, I_t} \left\| w_t^{(1,1)} - w_t^{(1,2)} - w_t^{(2,1)} + w_t^{(2,2)} \right\| \\
    &\quad \leq \sum_{t_0=1}^t \left[
        \frac{24c \rho L^2}{s+2} \left( \frac{1}{(s+2)\beta^2} + \frac{2c^2}{m} \right) t_0^{2c\beta - 1}
        + \frac{8cL}{(s+2)^2} t_0^{c\beta - 1}
    \right] \left( \frac{t}{t_0} \right)^{c\beta} \\
    &\quad = \frac{24c \rho L^2 t^{c\beta}}{s+2} \left( \frac{1}{(s+2)\beta^2} + \frac{2c^2}{m} \right) \sum_{t_0=1}^t t_0^{c\beta - 1}
    + \frac{8cL t^{c\beta}}{(s+2)^2} \sum_{t_0=1}^t t_0^{-1}.
\end{align*}
Using $\sum_{t_0=1}^t t_0^{c\beta - 1} \approx \frac{t^{c\beta}}{c\beta}$ for $c\beta > 0$, and $\sum_{t_0=1}^t t_0^{-1} \approx \log t$:
\begin{align*}
    &\mathbb{E}_{I_1, \dots, I_t} \left\| w_t^{(1,1)} - w_t^{(1,2)} - w_t^{(2,1)} + w_t^{(2,2)} \right\| \\
    &\quad \leq \frac{24c \rho L^2 t^{c\beta}}{s+2} \left( \frac{1}{(s+2)\beta^2} + \frac{2c^2}{m} \right) \frac{t^{c\beta}}{c\beta}
    + \frac{8cL t^{c\beta}}{(s+2)^2} \log t \\
    &\quad = \frac{24\rho L^2}{(s+2)\beta} \left( \frac{1}{(s+2)\beta^2} + \frac{2c^2}{m} \right) t^{2c\beta}
    + \frac{8cL}{(s+2)^2} t^{c\beta} \log t.
\end{align*}
This completes the proof.
\end{proof}

\subsubsection{Technical lemmas used in the proof of Proposition~\ref{proposition::SGD_final_prop}}

\LemmaSGDRelationhipBetweenUAndParams*

\begin{proof}[Proof of Lemma~\ref{lemma::SGD_relationship_between_U_and_l2_dist}]
Let $\mathcal{S}$ be the base dataset. We expand the expression using the definition of $U(\mathcal{S}; T) = \mathbb{E}[u(w_T)]$ where the expectation is over the choice of mini-batches $I_1, \dots, I_T$:
\begin{align*}
    &\left| U(\mathcal{S} \cup \{z_1, z_1'\}; T) - U(\mathcal{S} \cup \{z_1, z_2\}; T) - U(\mathcal{S} \cup \{z_1', z_2'\}; T) + U(\mathcal{S} \cup \{z_2, z_2'\}; T) \right| \\
    & \quad = \left| \mathbb{E}_{I_1, \dots, I_T} \left[u(w_T^{(1,1)})\right] - \mathbb{E}_{I_1, \dots, I_T} \left[u(w_T^{(1,2)})\right] - \mathbb{E}_{I_1, \dots, I_T} \left[u(w_T^{(2,1)})\right] + \mathbb{E}_{I_1, \dots, I_T} \left[u(w_T^{(2,2)})\right] \right| \\
    & \quad = \left| \mathbb{E}_{I_1, \dots, I_T} \left[ u(w_T^{(1,1)}) - u(w_T^{(1,2)}) - u(w_T^{(2,1)}) + u(w_T^{(2,2)}) \right] \right| \\
    & \quad \leq \mathbb{E}_{I_1, \dots, I_T} \left| u(w_T^{(1,1)}) - u(w_T^{(1,2)}) - u(w_T^{(2,1)}) + u(w_T^{(2,2)}) \right|. \quad \text{(by Jensen's inequality)}
\end{align*}

We now expand the differences of $u$ using a Taylor expansion. For $u(w_A) - u(w_B)$, by Taylor's theorem with remainder:
$u(w_A) - u(w_B) = \langle \nabla_\theta u(w_B), w_A - w_B \rangle + R_{A,B}$,
where the remainder $R_{A,B}$ satisfies $|R_{A,B}| \leq C_u \|w_A - w_B\|^2$ if $\nabla_\theta u$ is $C_u$-Lipschitz (from Assumption~\ref{assump::smoothness_for_SGD}, using $C_u$ for this constant). Thus, this remainder is $O(\|w_A - w_B\|^2)$.

So, we have:
\begin{align*}
    u(w_T^{(1,1)}) - u(w_T^{(1,2)}) &= \left\langle \nabla_\theta u(w_T^{(1,2)}), w_T^{(1,1)} - w_T^{(1,2)} \right\rangle + O\left( \| w_T^{(1,1)} - w_T^{(1,2)} \|^2 \right), \\
    u(w_T^{(2,1)}) - u(w_T^{(2,2)}) &= \left\langle \nabla_\theta u(w_T^{(2,2)}), w_T^{(2,1)} - w_T^{(2,2)} \right\rangle + O\left( \| w_T^{(2,1)} - w_T^{(2,2)} \|^2 \right).
\end{align*}
Then, the term inside the expectation is:
\begin{align*}
    &u(w_T^{(1,1)}) - u(w_T^{(1,2)}) - (u(w_T^{(2,1)}) - u(w_T^{(2,2)})) \\
    &\quad = \left\langle \nabla_\theta u(w_T^{(1,2)}), w_T^{(1,1)} - w_T^{(1,2)} \right\rangle - \left\langle \nabla_\theta u(w_T^{(2,2)}), w_T^{(2,1)} - w_T^{(2,2)} \right\rangle \\
    &\quad \quad + O\left( \| w_T^{(1,1)} - w_T^{(1,2)} \|^2 \right) + O\left( \| w_T^{(2,1)} - w_T^{(2,2)} \|^2 \right).
\end{align*}
The sum of the $O(\cdot)$ terms is $O\left( \| w_T^{(1,1)} - w_T^{(1,2)} \|^2 + \| w_T^{(2,1)} - w_T^{(2,2)} \|^2 \right)$.

We manipulate the inner product terms:
\begin{align*}
    &\left\langle \nabla_\theta u(w_T^{(1,2)}), w_T^{(1,1)} - w_T^{(1,2)} \right\rangle - \left\langle \nabla_\theta u(w_T^{(2,2)}), w_T^{(2,1)} - w_T^{(2,2)} \right\rangle \\
    &= \left\langle \nabla_\theta u(w_T^{(1,2)}), w_T^{(1,1)} - w_T^{(1,2)} - w_T^{(2,1)} + w_T^{(2,2)} \right\rangle \\
    &\quad + \left\langle \nabla_\theta u(w_T^{(1,2)}) - \nabla_\theta u(w_T^{(2,2)}), w_T^{(2,1)} - w_T^{(2,2)} \right\rangle.
\end{align*}
For the second term in this sum, using Cauchy-Schwarz and the $C_u$-Lipschitz continuity of $\nabla_\theta u$:
\begin{align*}
    &\left| \left\langle \nabla_\theta u(w_T^{(1,2)}) - \nabla_\theta u(w_T^{(2,2)}), w_T^{(2,1)} - w_T^{(2,2)} \right\rangle \right| \\
    &\quad \leq \left\| \nabla_\theta u(w_T^{(1,2)}) - \nabla_\theta u(w_T^{(2,2)}) \right\| \cdot \left\| w_T^{(2,1)} - w_T^{(2,2)} \right\| \\
    &\quad \leq C_u \left\| w_T^{(1,2)} - w_T^{(2,2)} \right\| \cdot \left\| w_T^{(2,1)} - w_T^{(2,2)} \right\|.
\end{align*}
Using the inequality $ab \leq \frac{1}{2}(a^2 + b^2)$, this is bounded by:
\[
    \frac{C_u}{2} \left( \left\| w_T^{(1,2)} - w_T^{(2,2)} \right\|^2 + \left\| w_T^{(2,1)} - w_T^{(2,2)} \right\|^2 \right).
\]
This term is also of the order $O\left( \|w_T^{(1,2)} - w_T^{(2,2)}\|^2 + \|w_T^{(2,1)} - w_T^{(2,2)}\|^2 \right)$.

Combining all terms, the absolute value $\left| u(w_T^{(1,1)}) - u(w_T^{(1,2)}) - u(w_T^{(2,1)}) + u(w_T^{(2,2)}) \right|$ is bounded by:
\begin{align*}
    &\left| \left\langle \nabla_\theta u(w_T^{(1,2)}), w_T^{(1,1)} - w_T^{(1,2)} - w_T^{(2,1)} + w_T^{(2,2)} \right\rangle \right| \\
    &\quad + O\left( \| w_T^{(1,1)} - w_T^{(1,2)} \|^2 + \| w_T^{(2,1)} - w_T^{(2,2)} \|^2 + \| w_T^{(1,2)} - w_T^{(2,2)} \|^2 \right).
\end{align*}
Let $C_{\|\nabla u\|}$ be the bound on $\|\nabla_\theta u(\cdot)\|$ from Assumption~\ref{assump::smoothness_for_SGD}. Then the first term is bounded by:
\[ C_{\|\nabla u\|} \left\| w_T^{(1,1)} - w_T^{(1,2)} - w_T^{(2,1)} + w_T^{(2,2)} \right\|. \]
The sum of $O(\cdot)$ terms involves various pairwise squared differences. Each such term is of the form $\|w_T^{(j_1,k_1)} - w_T^{(j_2,k_2)}\|^2$. The sum is thus $O\left( \max_{j_1,k_1,j_2,k_2} \|w_T^{(j_1,k_1)} - w_T^{(j_2,k_2)}\|^2 \right)$.

Taking the expectation over $I_1, \dots, I_T$:
\begin{align*}
    &\left| U(\mathcal{S} \cup \{z_1, z_1'\}; T) - U(\mathcal{S} \cup \{z_1, z_2\}; T) - U(\mathcal{S} \cup \{z_1', z_2'\}; T) + U(\mathcal{S} \cup \{z_2, z_2'\}; T) \right| \\
    &\quad \leq C_{\|\nabla u\|} \mathbb{E}_{I_1, \dots, I_T} \left[ \left\| w_T^{(1,1)} - w_T^{(1,2)} - w_T^{(2,1)} + w_T^{(2,2)} \right\| \right] \\
    &\qquad + O\left( \max_{j_1, k_1, j_2, k_2 \in \{1, 2\}} \mathbb{E}_{I_1, \dots, I_T} \left[ \left\| w_T^{(j_1, k_1)} - w_T^{(j_2, k_2)} \right\|^2 \right] \right).
\end{align*}
This completes the proof.
\end{proof}

\LemmaSGDFirstOrderDist*

\begin{proof}[Proof of Lemma~\ref{lemma::SGD_first_order_l2_dist}]
Let $w_t^{(a)}$ and $w_t^{(b)}$ be the iterates at step $t$ when training on data sequences $z^{(a)}$ and $z^{(b)}$ respectively. These sequences are of length $N=s+1$ and differ only at the $N$-th position (i.e., $z_N^{(a)} = z_a$ and $z_N^{(b)} = z_b$, while $z_j^{(a)} = z_j^{(b)}$ for $j < N$).
The SGD update implies $w_t^{(a)} = w_{t-1}^{(a)} - \frac{\alpha_t}{m} \sum_{j \in I_t} \nabla_w \ell(w_{t-1}^{(a)}; z_j^{(a)})$.

The difference in iterates is:
\begin{align*}
    \left\|w_t^{(a)} - w_t^{(b)}\right\| &= \left\| \left[w_{t-1}^{(a)} - \frac{\alpha_t}{m} \sum_{j \in I_t} \nabla_w \ell(w_{t-1}^{(a)}; z_j^{(a)})\right] - \left[w_{t-1}^{(b)} - \frac{\alpha_t}{m} \sum_{j \in I_t} \nabla_w \ell(w_{t-1}^{(b)}; z_j^{(b)})\right] \right\| \\
    &= \left\| (w_{t-1}^{(a)} - w_{t-1}^{(b)}) - \frac{\alpha_t}{m} \sum_{j \in I_t} \left( \nabla_w \ell(w_{t-1}^{(a)}; z_j^{(a)}) - \nabla_w \ell(w_{t-1}^{(b)}; z_j^{(b)}) \right) \right\| \\
    &\leq \left\| w_{t-1}^{(a)} - w_{t-1}^{(b)} \right\| + \frac{\alpha_t}{m} \sum_{j \in I_t} \left\| \nabla_w \ell(w_{t-1}^{(a)}; z_j^{(a)}) - \nabla_w \ell(w_{t-1}^{(b)}; z_j^{(b)}) \right\|. \quad \text{(by triangle inequality)}
\end{align*}
Let $E_{t-1} = \|w_{t-1}^{(a)} - w_{t-1}^{(b)}\|$.
Consider the term $\left\| \nabla_w \ell(w_{t-1}^{(a)}; z_j^{(a)}) - \nabla_w \ell(w_{t-1}^{(b)}; z_j^{(b)}) \right\|$ for $j \in I_t$.

Case 1: The data point $z_j$ is common to both sequences.
If $z_j^{(a)} = z_j^{(b)} = z_j$ (i.e., $j < N$ or the point at index $j$ in the sequences is from the common part $\mathcal{S}$), then by $\beta$-smoothness of $\ell$ (Assumption~\ref{assump::smoothness_for_SGD}):
\[ \left\| \nabla_w \ell(w_{t-1}^{(a)}; z_j) - \nabla_w \ell(w_{t-1}^{(b)}; z_j) \right\| \leq \beta \left\| w_{t-1}^{(a)} - w_{t-1}^{(b)} \right\| = \beta E_{t-1}. \]

Case 2: The data point $z_j$ is the one that differs.
If $z_j^{(a)} = z_a$ and $z_j^{(b)} = z_b$ (i.e., $j=N$, the differing $(s+1)$-th point), then by Lipschitz assumption in Assumption~\ref{assump::smoothness_for_SGD}, $\|\nabla_w \ell(w, z)\|$ is bounded by $L$, so that,
\begin{align*}
    \left\| \nabla_w \ell(w_{t-1}^{(a)}; z_a) - \nabla_w \ell(w_{t-1}^{(b)}; z_b) \right\| &\leq \left\| \nabla_w \ell(w_{t-1}^{(a)}; z_a) - \nabla_w \ell(w_{t-1}^{(b)}; z_a) \right\| \\ & \quad + \left\| \nabla_w \ell(w_{t-1}^{(b)}; z_a) - \nabla_w \ell(w_{t-1}^{(b)}; z_b) \right\| \\
    &\leq \beta \left\| w_{t-1}^{(a)} - w_{t-1}^{(b)} \right\| + (\|\nabla_w \ell(w_{t-1}^{(b)}; z_a)\| + \|\nabla_w \ell(w_{t-1}^{(b)}; z_b)\|) \\
    &\leq \beta E_{t-1} + 2L.
\end{align*}

Then the sum $\sum_{j \in I_t} \left\| \nabla_w \ell(w_{t-1}^{(a)}; z_j^{(a)}) - \nabla_w \ell(w_{t-1}^{(b)}; z_j^{(b)}) \right\|$ can be bounded:
It consists of $(m-\mathbb{I}(N \in I_t))$ terms of Case 1 and $\mathbb{I}(N \in I_t)$ terms of Case 2.
Sum $\leq (m-\mathbb{I}(N \in I_t)) \beta E_{t-1} + \mathbb{I}(N \in I_t) (\beta E_{t-1} + 2L) = m \beta E_{t-1} + 2L\mathbb{I}(N \in I_t)$.
So,
\begin{align*}
    \left\|w_t^{(a)} - w_t^{(b)}\right\| &\leq E_{t-1} + \frac{\alpha_t}{m} (m \beta E_{t-1} + 2L \cdot \mathbb{I}(N \in I_t)) \\
    &= (1 + \alpha_t \beta) E_{t-1} + \frac{2\alpha_t L}{m} \cdot \mathbb{I}(N \in I_t) \\
    &= (1 + \alpha_t \beta) \left\| w_{t-1}^{(a)} - w_{t-1}^{(b)} \right\| + \frac{2\alpha_t L}{m} \cdot \mathbb{I}(N \in I_t).
\end{align*}
This recurrence relation for the norm (holding for each realization of $I_t$) leads to:
\begin{equation}
\label{eqn::SGD_no_convexity_first_order_diff_closed_form_lem12}
    \left\| w_t^{(a)} - w_t^{(b)} \right\| \leq \sum_{t_0=1}^t \frac{2\alpha_{t_0}L}{m} \mathbb{I}(N \in I_{t_0}) \prod_{t'=t_0+1}^t (1 + \alpha_{t'} \beta).
\end{equation}
Taking expectation over the mini-batch choices $I_1, \dots, I_t$, and substituting $\alpha_t \leq \frac{c}{t}$:
\begin{align*}
    \mathbb{E} \left\| w_t^{(a)} - w_t^{(b)} \right\|
    &\leq \sum_{t_0=1}^t \frac{2\alpha_{t_0}L}{m} \mathbb{E}[\mathbb{I}(N \in I_{t_0})] \prod_{t'=t_0+1}^t (1 + \alpha_{t'} \beta)  \\
    &\leq \sum_{t_0=1}^t \frac{2cL}{t_0 m} \cdot \mathbb{P}(N \in I_{t_0}) \prod_{t'=t_0+1}^t \left(1 + \frac{c\beta}{t'}\right).
\end{align*}
Since the mini-batch $I_{t_0}$ of size $m$ is sampled uniformly from the $N=s+1$ data points, $\mathbb{P}(N \in I_{t_0}) = \frac{m}{N} = \frac{m}{s+1}$.
\begin{align*}
    \mathbb{E} \left\| w_t^{(a)} - w_t^{(b)} \right\|
    &\leq \sum_{t_0=1}^t \frac{2cL}{t_0 m} \cdot \frac{m}{s+1} \cdot \exp\left(c\beta \sum_{t'=t_0+1}^t \frac{1}{t'} \right) \quad \left(\text{using } 1+x \le e^x\right) \\
    &\leq \frac{2cL}{s+1} \sum_{t_0=1}^t \frac{1}{t_0} \exp\left(c\beta \log\left(\frac{t}{t_0}\right)\right) \quad \left(\text{using } \sum_{k=a+1}^b 1/k \approx \log(b/a)\right) \\
    &= \frac{2cL}{s+1} \sum_{t_0=1}^t \frac{1}{t_0} \left(\frac{t}{t_0}\right)^{c\beta}.
\end{align*}
Using the bound $\sum_{k=1}^n \frac{1}{k} \left( \frac{n}{k} \right)^{\gamma} \leq \frac{n^{\gamma}}{\gamma}$ for $\gamma > 0$ (here $\gamma = c\beta$, $n=t$, $k=t_0$):
\[ \mathbb{E} \left\| w_t^{(a)} - w_t^{(b)} \right\| \leq \frac{2cL}{s+1} \cdot \frac{t^{c\beta}}{c\beta} = \frac{2L}{(s+1)\beta} t^{c\beta}. \]
This completes the proof.
\end{proof}

\LemmaSGDFirstOrderDistSquared*

\begin{proof}[Proof of Lemma~\ref{lemma::SGD_first_order_l2_squared}]
Let $N=s+1$ be the effective size of the data sequences $z^{(a)}$ and $z^{(b)}$.
From Equation~\eqref{eqn::SGD_no_convexity_first_order_diff_closed_form_lem12} in the proof of Lemma~\ref{lemma::SGD_first_order_l2_dist}, we have the bound for each realization of mini-batch choices:
\[ \left\| w_t^{(a)} - w_t^{(b)} \right\| \leq X_t := \sum_{t_0=1}^t \frac{2\alpha_{t_0}L}{m} \, \mathbb{I}(N \in I_{t_0}) \prod_{t'=t_0+1}^t (1 + \alpha_{t'}\beta). \]
We want to bound $\mathbb{E} \left\| w_t^{(a)} - w_t^{(b)} \right\|^2$. Since $\left\| w_t^{(a)} - w_t^{(b)} \right\| \ge 0$, if $\left\| w_t^{(a)} - w_t^{(b)} \right\| \leq X_t$, then $\left\| w_t^{(a)} - w_t^{(b)} \right\|^2 \leq X_t^2$. Thus, $\mathbb{E} \left\| w_t^{(a)} - w_t^{(b)} \right\|^2 \leq \mathbb{E}[X_t^2]$.
Using the property $\mathbb{E}[X_t^2] = (\mathbb{E}[X_t])^2 + \operatorname{Var}[X_t]$:
\begin{align*}
    \mathbb{E} \left\| w_t^{(a)} - w_t^{(b)} \right\|^2
    &\leq \left( \mathbb{E} [X_t] \right)^2 + \operatorname{Var} \left[ X_t \right] \\
    &= \left( \mathbb{E} \left[ \sum_{t_0=1}^t \frac{2\alpha_{t_0}L}{m} \, \mathbb{I}(N \in I_{t_0}) \prod_{t'=t_0+1}^t (1 + \alpha_{t'}\beta) \right] \right)^2 \\
    &\quad + \operatorname{Var} \left[ \sum_{t_0=1}^t \frac{2\alpha_{t_0}L}{m} \, \mathbb{I}(N \in I_{t_0}) \prod_{t'=t_0+1}^t (1 + \alpha_{t'}\beta) \right].
\end{align*}
Let $Y_{t_0} = \frac{2\alpha_{t_0}L}{m} \, \mathbb{I}(N \in I_{t_0}) \prod_{t'=t_0+1}^t (1 + \alpha_{t'}\beta)$. Since the random variables $\mathbb{I}(N \in I_{t_0})$ are independent across different time steps $t_0$ (as mini-batches $I_{t_0}$ are sampled independently), the terms $Y_{t_0}$ are independent. Thus, the variance of the sum is the sum of variances:
\[ \operatorname{Var} \left[ \sum_{t_0=1}^t Y_{t_0} \right] = \sum_{t_0=1}^t \operatorname{Var} \left( Y_{t_0} \right). \]
Let $K_{t_0} = \frac{2\alpha_{t_0}L}{m} \prod_{t'=t_0+1}^t (1 + \alpha_{t'}\beta)$. This term is deterministic once $\alpha$ values are fixed.
Let $p_{t_0} = \mathbb{P}(N \in I_{t_0}) = \frac{m}{N} = \frac{m}{s+1}$.
Then $\operatorname{Var}(Y_{t_0}) = \operatorname{Var}(K_{t_0} \mathbb{I}(N \in I_{t_0})) = K_{t_0}^2 \operatorname{Var}(\mathbb{I}(N \in I_{t_0})) = K_{t_0}^2 p_{t_0}(1-p_{t_0}) \leq K_{t_0}^2 p_{t_0}$.
Using this observation for the variance term and the bound from Lemma~\ref{lemma::SGD_first_order_l2_dist} for the $(\mathbb{E}[X_t])^2$ term:
\begin{align*}
    \mathbb{E} \left\| w_t^{(a)} - w_t^{(b)} \right\|^2
    &\leq \left( \frac{2L}{(s+1)\beta} t^{c\beta} \right)^2
    + \sum_{t_0=1}^t \left( \frac{2\alpha_{t_0}L}{m} \prod_{t'=t_0+1}^t (1 + \alpha_{t'}\beta) \right)^2 \frac{m}{s+1} \\
    &= \frac{4L^2}{(s+1)^2\beta^2} t^{2c\beta}
    + \sum_{t_0=1}^t \frac{4\alpha_{t_0}^2 L^2}{m^2} \cdot \frac{m}{s+1} \left( \prod_{t'=t_0+1}^t (1 + \alpha_{t'}\beta) \right)^2 \\
    &\leq \frac{4L^2}{(s+1)^2 \beta^2} t^{2c\beta}
    + \sum_{t_0=1}^t \frac{4c^2 L^2}{t_0^2 m (s+1)} \left( \prod_{t'=t_0+1}^t \exp\left( \frac{c\beta}{t'} \right) \right)^2 \\
    &\leq \frac{4L^2}{(s+1)^2 \beta^2} t^{2c\beta}
    + \sum_{t_0=1}^t \frac{4c^2 L^2}{t_0^2 m (s+1)} \exp\left( 2c\beta \sum_{t'=t_0+1}^t \frac{1}{t'} \right) \\
    &\leq \frac{4L^2}{(s+1)^2 \beta^2} t^{2c\beta}
    + \frac{4c^2 L^2}{m(s+1)} \sum_{t_0=1}^t \frac{1}{t_0^2} \left( \frac{t}{t_0} \right)^{2c\beta}.
\end{align*}

Using the bound $\sum_{k=1}^\infty \frac{1}{k^{2 + \gamma}} \leq (1 + \frac{1}{1+\gamma})$ for $\gamma \ge 0$ (here, $k=t_0$, $\gamma = 2c\beta$):
\begin{align*}
    \mathbb{E} \left\| w_t^{(a)} - w_t^{(b)} \right\|^2
    &\leq \frac{4L^2}{(s+1)^2 \beta^2} t^{2c\beta}
    + \frac{4c^2 L^2}{m(s+1)} \left( 1 + \frac{1}{1 + 2c\beta} \right) t^{2c\beta} \\
    &\leq \frac{4L^2}{s+1} \, t^{2c\beta} \left( \frac{1}{(s+1)\beta^2} + \frac{2c^2}{m} \right)
\end{align*}
Replacing $s+1$ with $N$ (the effective dataset size as per the lemma statement):
\[ \mathbb{E} \left\| w_t^{(a)} - w_t^{(b)} \right\|^2 \leq \frac{4L^2}{N} \, t^{2c\beta} \left( \frac{1}{N\beta^2} + \frac{2c^2}{m} \right). \]
This completes the proof.
\end{proof}

\LemmaSGDSmoothnessTool*

\begin{proof}[Proof of Lemma~\ref{lemma::SGD_smoothness_tool}]
Define the function $\phi(t) := \nabla_w \ell(w + t h; z)$ for $t \in [0, 1]$. Then, the function $\phi : [0,1] \to \mathbb{R}^d$ is continuously differentiable. By the mean-value-type expansion result of \citet{mcleod1965mean}, we can write:
\begin{equation*}
    \phi(1) = \phi(0) + \sum_{k=1}^d \lambda_k \phi'(t_k),
\end{equation*}
for some $t_k \in (0,1)$, non-negative weights $\lambda_k \geq 0$ such that $\sum_{k=1}^d \lambda_k = 1$.

The derivative $\phi'(t)$ is $\frac{d}{dt} \nabla_w \ell(w + th; z) = \nabla_w^2 \ell(w + th; z) h$.
Substituting the definitions of $\phi(0)$, $\phi(1)$, and $\phi'(t_k)$ into the expansion, we get:
\begin{align*}
    \nabla_w \ell(w + h; z) - \nabla_w \ell(w; z)
    &= \sum_{k=1}^d \lambda_k \left( \nabla_w^2 \ell(w + t_k h; z) h \right) \\
    &= \left( \sum_{k=1}^d \lambda_k \nabla_w^2 \ell(w + t_k h; z) \right) h.
\end{align*}
To isolate the term $\nabla_w^2 \ell(w; z) h$, we add and subtract it:
\begin{align*}
    \nabla_w \ell(w + h; z) - \nabla_w \ell(w; z)
    &= \nabla_w^2 \ell(w; z) h
    + \left( \sum_{k=1}^d \lambda_k \left[ \nabla_w^2 \ell(w + t_k h; z) - \nabla_w^2 \ell(w; z) \right] \right) h.
\end{align*}
Rearranging the terms, we have:
\[
    \nabla_w \ell(w + h; z) - \nabla_w \ell(w; z) - \nabla_w^2 \ell(w; z) h
    = \sum_{k=1}^d \lambda_k \left[ \nabla_w^2 \ell(w + t_k h; z) - \nabla_w^2 \ell(w; z) \right] h.
\]
Taking norms on both sides and applying the triangle inequality for sums, followed by the properties of matrix norms:
\begin{align*}
    \left\| \nabla_w \ell(w + h; z) - \nabla_w \ell(w; z) - \nabla_w^2 \ell(w; z) h \right\|
    &\leq \left\| \sum_{k=1}^d \lambda_k \left[ \nabla_w^2 \ell(w + t_k h; z) - \nabla_w^2 \ell(w; z) \right] h \right\| \\
    &\leq \sum_{k=1}^d \lambda_k \left\| \nabla_w^2 \ell(w + t_k h; z) - \nabla_w^2 \ell(w; z) \right\| \cdot \|h\|.
\end{align*}
By the assumption of Lipschitz continuity of the Hessian of $\ell$ with constant $\rho$ in Assumption~\ref{assump::smoothness_for_SGD}:
\begin{align*}
    \left\| \nabla_w^2 \ell(w + t_k h; z) - \nabla_w^2 \ell(w; z) \right\| \leq \rho \|(w + t_k h) - w\| = \rho t_k \|h\|.
\end{align*}
Substituting this into the inequality:
\begin{align*}
    \left\| \nabla_w \ell(w + h; z) - \nabla_w \ell(w; z) - \nabla_w^2 \ell(w; z) h \right\|
    &\leq \sum_{k=1}^d \lambda_k \cdot (\rho t_k \|h\|) \cdot \|h\| \\
    &= \rho \|h\|^2 \sum_{k=1}^d \lambda_k t_k.
\end{align*}
Since $t_k \in (0,1)$, we have $t_k \leq 1$. Also, $\lambda_k \geq 0$ and $\sum_{k=1}^d \lambda_k = 1$.
Therefore, $\sum_{k=1}^d \lambda_k t_k \leq \sum_{k=1}^d \lambda_k \cdot 1 = 1$.
Thus,
\[ \left\| \nabla_w \ell(w + h; z) - \nabla_w \ell(w; z) - \nabla_w^2 \ell(w; z) h \right\| \leq \rho \|h\|^2. \]
This completes the proof.
\end{proof}

\subsection{Proof of Proposition~\ref{proposition::application::IF} (algorithmic stability of Influence Function (IF))}
\label{sec::proof_proposition_application_IF}

The proof of Proposition~\ref{proposition::application::IF} uses some familiar techniques from the standard IF theory, but also features our original analysis.
We first define an auxiliary parameter vector $\tilde{\theta}_{\mathcal{S}}$. 
This vector is the minimizer of the empirical loss over $\mathcal{S}$ but with its denominator scaled as if two additional points were present in the averaging, which provides a good anchor point that facilitates a cleaner Taylor expansion form for $\hat{\theta}_{\mathcal{S}\cup\{z_1,z_2\}}$.
\begin{equation*}
    \tilde{\theta}_{\mathcal{S}} := \arg \min_{\theta} \left\{ \frac{1}{s + 2} \sum_{z \in \mathcal{S}} \ell(\theta; z) + \frac{\lambda}{2} \|\theta\|_2^2 \right\}.
\end{equation*}
Next, we present two technical lemmas, based on which, the validity of Proposition~\ref{proposition::application::IF} would become clear.
The proofs of these technical lemmas are relegated Appendix~\ref{sec::proofs_for_IF_lemmas_appendix}.

\begin{restatable}[Expansion of parameter difference via IF]{lemma}{LemmaThetaDiffIFPolished} \label{lemma::theta_diff_IF_polished}
    Suppose the loss function $\ell(\theta; z)$ is convex in $\theta$, three times continuously differentiable with respect to $\theta$, and its first-, second-, and third-order derivatives with respect to $\theta$ are uniformly bounded for all $z \in \mathcal{Z}$. For any dataset $\mathcal{S} \in \mathcal{Z}^s$ (where $s=|\mathcal{S}|$) and any two data points $z_1, z_2 \in \mathcal{Z}$, let $\hat{\theta}_{\mathcal{S} \cup \{z_1, z_2\}}$ be the parameter vector minimizing the regularized loss on $\mathcal{S} \cup \{z_1, z_2\}$ as per \eqref{eqn::param_definition_IF}. 
    Then we have
    \begin{enumerate}
        \item $\|\hat{\theta}_{\mathcal{S} \cup \{z_1, z_2\}} - \tilde{\theta}_{\mathcal{S}}\| = O(s^{-1})$.
        \item $\hat{\theta}_{\mathcal{S} \cup \{z_1, z_2\}} - \tilde{\theta}_{\mathcal{S}} = -H_{\tilde{\theta}_{\mathcal{S}}}^{-1} \left[ \nabla_\theta\ell(\tilde{\theta}_{\mathcal{S}}; z_1) + \nabla_\theta\ell(\tilde{\theta}_{\mathcal{S}}; z_2) \right]\frac{1}{s + 2} + O(s^{-2})$,
        where $H_{\tilde{\theta}_{\mathcal{S}}} := \left( \frac{1}{s + 2} \sum_{z \in \mathcal{S}} \nabla_\theta^2\ell(\tilde{\theta}_{\mathcal{S}}; z) \right) + \lambda I$.
    \end{enumerate}
\end{restatable}

\begin{restatable}[Expansion of utility]{lemma}{LemmaUDiffIFPolished} \label{lemma::u_diff_IF_polished}
    Suppose the performance metric $u(\theta)$ is once continuously differentiable and its gradient $\nabla_\theta u(\theta)$ is $L_u$-Lipschitz continuous and bounded. 
    Then, for any parameter vectors $\theta_A$ and $\theta_B$, we have
    \begin{equation*}
        u(\theta_A) - u(\theta_B) = \langle\nabla_\theta u(\theta_B), \theta_A - \theta_B \rangle + O\left( \| \theta_A - \theta_B \|^2\right).
    \end{equation*}
\end{restatable}

\begin{proof}[Proof of Proposition~\ref{proposition::application::IF}]
Let $\Delta_U = U(\mathcal{S}\cup\{z_1,z_1'\}) - U(\mathcal{S}\cup\{z_1,z_2\}) - U(\mathcal{S}\cup\{z_1',z_2'\}) + U(\mathcal{S}\cup\{z_2,z_2'\})$.
Recall $U(\mathcal{X}) = u(\hat{\theta}_{\mathcal{X}})$. We use Lemma~\ref{lemma::u_diff_IF_polished} to expand each $U(\mathcal{S} \cup \{a,b\})$ around $u(\tilde{\theta}_{\mathcal{S}})$:
\begin{equation*}
    U(\mathcal{S} \cup \{a,b\}) - u(\tilde{\theta}_{\mathcal{S}}) = \langle \nabla_\theta u(\tilde{\theta}_{\mathcal{S}}), \hat{\theta}_{\mathcal{S}\cup\{a,b\}} - \tilde{\theta}_{\mathcal{S}} \rangle + O(\|\hat{\theta}_{\mathcal{S}\cup\{a,b\}} - \tilde{\theta}_{\mathcal{S}}\|^2).
\end{equation*}
Let $\delta\hat{\theta}_{ab} := \hat{\theta}_{\mathcal{S}\cup\{a,b\}} - \tilde{\theta}_{\mathcal{S}}$. 
Lemma~\ref{lemma::theta_diff_IF_polished} implies that $\|\delta\hat{\theta}_{ab}\|^2 = O(s^{-2})$.
It also states that the first-order component of $\delta\hat{\theta}_{ab}$ is:
\begin{equation*}
    \delta\hat{\theta}_{ab}^{(1)} := -H_{\tilde{\theta}_{\mathcal{S}}}^{-1} \left[ \nabla_\theta\ell(\tilde{\theta}_{\mathcal{S}}; a) + \nabla_\theta\ell(\tilde{\theta}_{\mathcal{S}}; b) \right]\frac{1}{s + 2}.
\end{equation*}
Thus, $\delta\hat{\theta}_{ab} = \delta\hat{\theta}_{ab}^{(1)} + O(s^{-2})$.
Substituting this into the expansion of $U(\mathcal{S} \cup \{a,b\}) - u(\tilde{\theta}_{\mathcal{S}})$ yields
\begin{equation*}
    U(\mathcal{S} \cup \{a,b\}) = u(\tilde{\theta}_{\mathcal{S}}) + \langle \nabla_\theta u(\tilde{\theta}_{\mathcal{S}}), \delta\hat{\theta}_{ab}^{(1)} \rangle + \langle \nabla_\theta u(\tilde{\theta}_{\mathcal{S}}), O(s^{-2}) \rangle + O(s^{-2}).
\end{equation*}
Since $\nabla_\theta u(\tilde{\theta}_{\mathcal{S}})$ is bounded, the term $\langle \nabla_\theta u(\tilde{\theta}_{\mathcal{S}}), O(s^{-2}) \rangle$ is also $O(s^{-2})$.
Therefore, $U(\mathcal{S} \cup \{a,b\}) = u(\tilde{\theta}_{\mathcal{S}}) + \langle \nabla_\theta u(\tilde{\theta}_{\mathcal{S}}), \delta\hat{\theta}_{ab}^{(1)} \rangle + O(s^{-2})$.

Now, we substitute this expansion into the expression of $\Delta_U$ and obtain
\begin{align*}
    \Delta_U = &\left( u(\tilde{\theta}_{\mathcal{S}}) + \langle \nabla_\theta u(\tilde{\theta}_{\mathcal{S}}), \delta\hat{\theta}_{z_1 z_1'}^{(1)} \rangle \right) - \left( u(\tilde{\theta}_{\mathcal{S}}) + \langle \nabla_\theta u(\tilde{\theta}_{\mathcal{S}}), \delta\hat{\theta}_{z_1 z_2}^{(1)} \rangle \right) \\
    &- \left( u(\tilde{\theta}_{\mathcal{S}}) + \langle \nabla_\theta u(\tilde{\theta}_{\mathcal{S}}), \delta\hat{\theta}_{z_1' z_2'}^{(1)} \rangle \right) + \left( u(\tilde{\theta}_{\mathcal{S}}) + \langle \nabla_\theta u(\tilde{\theta}_{\mathcal{S}}), \delta\hat{\theta}_{z_2 z_2'}^{(1)} \rangle \right) + \sum O(s^{-2}).
\end{align*}
The $u(\tilde{\theta}_{\mathcal{S}})$ terms cancel. The sum of the four $O(s^{-2})$ remainder terms is still $O(s^{-2})$. The sum of the first-order inner product terms is:
\begin{equation*}
    \langle \nabla_\theta u(\tilde{\theta}_{\mathcal{S}}), \delta\hat{\theta}_{z_1 z_1'}^{(1)} - \delta\hat{\theta}_{z_1 z_2}^{(1)} - \delta\hat{\theta}_{z_1' z_2'}^{(1)} + \delta\hat{\theta}_{z_2 z_2'}^{(1)} \rangle.
\end{equation*}
Let $h_x = -\frac{1}{s+2}H_{\tilde{\theta}_{\mathcal{S}}}^{-1} \nabla_\theta\ell(\tilde{\theta}_{\mathcal{S}}; x)$. Then $\delta\hat{\theta}_{ab}^{(1)} = h_a + h_b$.
The sum of influence terms becomes:
\begin{align*}
    &\langle \nabla_\theta u(\tilde{\theta}_{\mathcal{S}}), (h_{z_1} + h_{z_1'}) - (h_{z_1} + h_{z_2}) - (h_{z_1'} + h_{z_2'}) + (h_{z_2} + h_{z_2'}) \rangle \\
    &= \langle \nabla_\theta u(\tilde{\theta}_{\mathcal{S}}), h_{z_1} + h_{z_1'} - h_{z_1} - h_{z_2} - h_{z_1'} - h_{z_2'} + h_{z_2} + h_{z_2'} \rangle \\
    &= \langle \nabla_\theta u(\tilde{\theta}_{\mathcal{S}}), \mathbf{0} \rangle = 0.
\end{align*}
Thus, the first-order influence terms cancel out completely. The remaining terms are all $O(s^{-2})$.
The proof of Proposition~\ref{proposition::application::IF} is complete.
\end{proof}

\subsubsection{Lemmas used in the proof of Proposition~\ref{proposition::application::IF}} 
\label{sec::proofs_for_IF_lemmas_appendix}

\LemmaThetaDiffIFPolished* 
\begin{proof}[Proof of Lemma~\ref{lemma::theta_diff_IF_polished}]
For $\delta_{\rm val} \in [0, \frac{1}{s + 2}]$, define the perturbed objective function:
\begin{equation*}
    \mathcal{L}_{\rm pert}(\theta, \delta_{\rm val}) := \frac{1}{s + 2} \sum_{z \in \mathcal{S}} \ell(\theta; z) + \delta_{\rm val}\left[\ell(\theta; z_1) + \ell(\theta; z_2)\right] + \frac{\lambda}{2} \|\theta\|_2^2,
\end{equation*}
and its minimizer $\hat{\theta}(\delta_{\rm val}) := \arg \min_{\theta} \mathcal{L}_{\rm pert}(\theta, \delta_{\rm val})$.
By construction, $\tilde{\theta}_{\mathcal{S}} = \hat{\theta}(0)$. The parameter vector $\hat{\theta}_{\mathcal{S} \cup \{z_1, z_2\}}$ minimizes $\frac{1}{s+2}\sum_{z \in \mathcal{S}\cup\{z_1,z_2\}} \ell(\theta;z) + \frac{\lambda}{2}\|\theta\|_2^2$. This corresponds to $\hat{\theta}(\frac{1}{s+2})$. Thus, $\hat{\theta}_{\mathcal{S} \cup \{z_1, z_2\}} - \tilde{\theta}_{\mathcal{S}} = \hat{\theta}(\frac{1}{s+2}) - \hat{\theta}(0)$.

The first-order condition for $\hat{\theta}(\delta_{\rm val})$ is $F(\hat{\theta}(\delta_{\rm val}), \delta_{\rm val}) = 0$, where
\begin{equation*}
    F(\theta, \delta_{\rm val}) := 
    \frac{\partial \mathcal{L}_{\rm pert}}{\partial \theta}
    =
    \frac{1}{s + 2} \sum_{z \in \mathcal{S}} \nabla_\theta\ell(\theta; z) + \delta_{\rm val} \left[ \nabla_\theta \ell(\theta; z_1) + \nabla_\theta \ell(\theta; z_2)\right] + \lambda \theta = 0.
\end{equation*}
Given the smoothness conditions on $\ell$ (convexity, continuous third derivatives, bounded derivatives), the Higher-Order Implicit Function Theorem (e.g., \citet{zorich2016differential}) ensures that $\hat{\theta}(\delta_{\rm val})$ is twice continuously differentiable with respect to $\delta_{\rm val}$ around $\delta_{\rm val}=0$.
Differentiating $F(\hat{\theta}(\delta_{\rm val}), \delta_{\rm val})=0$ w.r.t. $\delta_{\rm val}$ gives 
\[ 
\frac{\partial F}{\partial \theta} \frac{\partial \hat{\theta}}{\partial \delta_{\rm val}} + \frac{\partial F}{\partial \delta_{\rm val}} = 0.
\]
Let 
\[ 
H(\delta_{\rm val}) := \frac{\partial F}{\partial \theta} = \frac{1}{s + 2} \sum_{z \in \mathcal{S}} \nabla_\theta^2\ell(\hat{\theta}(\delta_{\rm val}); z) + \delta_{\rm val} [\nabla_\theta^2\ell(\hat{\theta}(\delta_{\rm val}); z_1) + \nabla_\theta^2\ell(\hat{\theta}(\delta_{\rm val}); z_2)] + \lambda I,
\]
and 
\[ 
\frac{\partial F}{\partial \delta_{\rm val}} = \nabla_\theta \ell(\hat{\theta}(\delta_{\rm val}); z_1) + \nabla_\theta \ell(\hat{\theta}(\delta_{\rm val}); z_2).
\]
So, 
\[
\frac{\partial\hat{\theta}(\delta_{\rm val})}{\partial \delta_{\rm val}} = -H(\delta_{\rm val})^{-1} \left[ \nabla_\theta\ell(\hat{\theta}(\delta_{\rm val}); z_1) + \nabla_\theta\ell(\hat{\theta}(\delta_{\rm val}); z_2) \right].
\]
Due to convexity of $\ell$, $\nabla_\theta^2\ell \succeq 0$, so $H(\delta_{\rm val}) \succeq \lambda I$. With $\lambda > 0$, $H(\delta_{\rm val})$ is positive definite and its inverse is bounded. Uniform boundedness of $\nabla_\theta \ell$ implies $\frac{\partial\hat{\theta}(\delta_{\rm val})}{\partial \delta_{\rm val}}$ is uniformly bounded. Similarly, uniform boundedness of up to third-order derivatives of $\ell$ ensures $\frac{\partial^2\hat{\theta}(\delta_{\rm val})}{\partial \delta_{\rm val}^2}$ is uniformly bounded.

By Taylor's theorem with Lagrange remainder, for some $\xi \in [0, \frac{1}{s + 2}]$:
\begin{equation*}
    \hat{\theta}\left(\frac{1}{s + 2}\right) - \hat{\theta}(0) = \left. \frac{\partial\hat{\theta}(\delta_{\rm val})}{\partial\delta_{\rm val}} \right|_{\delta_{\rm val} = 0} \cdot \frac{1}{s + 2} + \frac{1}{2} \left. \frac{\partial^2\hat{\theta}(\delta_{\rm val})}{\partial\delta_{\rm val}^2} \right|_{\delta_{\rm val} = \xi} \cdot \left(\frac{1}{s + 2}\right)^2.
\end{equation*}
Noting $\hat{\theta}(0) = \tilde{\theta}_{\mathcal{S}}$ and $H(0) = H_{\tilde{\theta}_{\mathcal{S}}}$ (as defined in the lemma statement), we have:
\begin{align*}
    \hat{\theta}_{\mathcal{S}\cup \{z_1, z_2\}} - \tilde{\theta}_{\mathcal{S}} = -H_{\tilde{\theta}_{\mathcal{S}}}^{-1} \left[ \nabla_\theta\ell(\tilde{\theta}_{\mathcal{S}}; z_1) + \nabla_\theta\ell(\tilde{\theta}_{\mathcal{S}}; z_2) \right]\frac{1}{s + 2} + O\left(\frac{1}{(s+2)^2}\right).
\end{align*}
This establishes part (2) of the lemma with the corrected negative sign. Part (1), $\|\hat{\theta}_{\mathcal{S}\cup \{z_1, z_2\}} - \tilde{\theta}_{\mathcal{S}}\| = O(s^{-1})$, follows because the first term is $O(s^{-1})$ (since $H_{\tilde{\theta}_{\mathcal{S}}}^{-1}$ and gradients are bounded) and dominates the $O(s^{-2})$ remainder for large $s$.
\end{proof}

\LemmaUDiffIFPolished*
\begin{proof}[Proof of Lemma~\ref{lemma::u_diff_IF_polished}]
By Taylor's theorem (or the Mean Value Theorem for vector functions), since $u$ is once continuously differentiable, for some $\xi$ on the line segment connecting $\theta_A$ and $\theta_B$:
\begin{align*}
    u(\theta_A) - u(\theta_B) &= \langle \nabla_\theta u(\xi), \theta_A - \theta_B \rangle \\
    &= \langle\nabla_\theta u(\theta_B), \theta_A - \theta_B \rangle + \langle \nabla_\theta u(\xi) -\nabla_\theta u(\theta_B), \theta_A - \theta_B \rangle.
\end{align*}
The second term can be bounded using the Cauchy-Schwarz inequality and the $L_u$-Lipschitz continuity of $\nabla_\theta u$:
\begin{align*}
    |\langle \nabla_\theta u(\xi) -\nabla_\theta u(\theta_B), \theta_A - \theta_B \rangle|
    &\le \| \nabla_\theta u(\xi) -\nabla_\theta u(\theta_B) \|_2 \cdot \| \theta_A - \theta_B \|_2 \\
    &\le L_u \| \xi - \theta_B \|_2 \cdot \| \theta_A - \theta_B \|_2.
\end{align*}
Since $\xi$ lies on the line segment between $\theta_A$ and $\theta_B$, we have $\|\xi - \theta_B\|_2 \le \|\theta_A - \theta_B\|_2$.
Therefore, the absolute value of the second term is bounded by $L_u \| \theta_A - \theta_B \|_2^2$, which is $O(\|\theta_A - \theta_B\|^2)$.
This establishes the result.
\end{proof}

\clearpage

\section{Experimental Details}
\label{sec::exp_detail}

This appendix outlines the implementation details for all experiments in the main paper.
To evaluate the effectiveness, efficiency, and robustness of FGSV, four experimental settings were employed.
These consist of: (i) a synthetic data example in the introduction to illustrate the shell company attack (Section~\ref{sec:intro}); 
(ii) a benchmark comparison on the SOU cooperative game (Section~\ref{sec::benchmark});
(iii) an application to copyright attribution using a generative AI model with FlickrLogo-27 Dataset (Section~\ref{sec::copyright});
and (iv) an application to explainable AI using the Diabetes dataset (Section~\ref{sec::explainable}).

All experiments were conducted on the Unity high-performance computing cluster, provided by 
the College of Arts and Sciences at The Ohio State University.  
The copyright attribution experiment (Section~\ref{sec::copyright}) was executed on a GPU node equipped with an NVIDIA V100 GPU, 32GB VRAM, whereas all other experiments were conducted on CPU nodes featuring Intel Xeon E5-2699 v4 processors and 256 GB RAM.

\subsection{Motivational Example (Section~\ref{sec:intro})}

A synthetic setting is constructed to illustrate the shell company attack and its effect on group-level data valuation.
We generate synthetic data for binary classification from a Gaussian mixture. 
Specifically, given a class label $y_i\in\{0,1\}$, each sample $x_i \in \mathbb{R}^2$ is drawn from $\mathcal{N}(\mu_1, I_2)$ if $y_i = 0$, and from $\mathcal{N}(\mu_2, I_2)$ if $y_i = 1$, where $\mu_1 = [-3, 0]^\top$ and $\mu_2 = [3, 0]^\top$.
The class label is sampled independently with equal probability $p = 0.5$.  
An additional $n=200$ samples are independently generated from the same distribution for testing.

The training data is partitioned into two equal-sized groups of 100 samples each.
To simulate the shell company attack, the second group is further subdivided evenly into 2, 3, or 4 subgroups.
This results in a total of $k \in \{2,3,4,5\}$ disjoint groups: the first group consistently contains 100 samples, while the remaining $k-1$ groups are allocated approximately $100/(k-1)$ samples each, with any remainder distributed to maintain balance (e.g., 33, 33, and 34 when $k=4$).

For each subset $S$ of the training data, a logistic regression classifier is trained.  
The utility function $U$ is defined as its classification accuracy on a held-out test set.  
If $S$ contains only a single class, the utility is set to a constant value of $1/2$, corresponding to the expected accuracy of a random classifier.

While the experiment is qualitative and does not include confidence intervals, repeated runs with different random seeds produced consistent trends.  
FGSV is estimated using Algorithm~\ref{alg::gsv} with threshold parameter $\bar{s} = 20$ and Monte Carlo sample sizes $m_1 = m_2 = 2000$.

\subsection{Comparison with benchmark methods (Section~\ref{sec::benchmark})}

For all methods compared here, we adjust their configurations, such that the total number of utility evaluations is aligned to the fixed number of 20,000.
Note that since FGSV estimates the value for one group at a time, for fairness, the total budget is equally divided among the four groups, with 5,000 utility evaluations allocated per group. For FGSV, we set the threshold parameter to $\bar{s} = 10$ and choose $m_1 = m_2$ such that the total number of utility function calls sums to 20,000. 
Below, we describe the setup and computational structure of each method in detail, largely based on the implementation and descriptions from \cite{li2024one}.

\begin{itemize}
    \item 
    {\bf Permutation estimator \citep{castro2009polynomial, ghorbani2019data}.}

    The permutation estimator computes Shapley values by averaging marginal contribution across random permutations.
    At each iteration, it draws a permutation $\pi$ uniformly at random from all possible orderings of the $n$ elements and computes a Monte Carlo estimate based on the following alternative representation of the individual Shapley value:
    $$
    \SV(i) = \ep_{\pi}\left[U(S_{\pi,i}\cup\{i\}) - U(S_{\pi,i})\right],
    $$
    where $S_{\pi,i}$ denotes the set of indices that appear before $i$ in permutation $\pi$. 
    The final estimate is obtained by averaging the marginal contributions over $T$ sampled permutations:
    $$
    \widehat{\SV}_{\text{perm}}(i) = \frac{1}{T} \sum_{t=1}^T \left[ U(S_{\pi_t,i} \cup {i}) - U(S_{\pi_t,i}) \right],
    $$
    where $\pi_t$ denotes the $t$-th sample of random permutation. 
    To simultaneously estimate $\SV(i)$ for all data points $i \in [n]$ per permutation, this method requires $n+1$ utility evaluations -- one for the empty set $U(\varnothing)$, and one each within incremental updates, $U(\{\pi(1),\dots,\pi(j)\})$ for $j=1,\dots,n$.
    Therefore, we set $T = \lceil\frac{20000}{n+1}\rceil$ to maintain the total utility evaluation budget.
    For the last sampled permutation, the estimator stopped when the number of evaluations reached exactly 20,000.

    \item 
    {\bf Group Testing estimator \citep{jia2019towards, wang2023note2}.}
    
    This estimator computes pairwise differences of Shapley values rather than individual values.
    At each step, it draws a random size of subset $s\in[n]$ with $p(s)\propto\frac{1}{s} + \frac{1}{n-s+1}$, and draws a random subset $S_t\subseteq[n+1]$ of size $s$, where $(n+1)$-th element represents a dummy player, i.e., a player whose contribution is always zero and serves as a reference for baseline utility.
    In particular, for a given number of utility evaluations, $T$, a matrix $B\in\mathbb{R}^{T\times (n+1)}$ is constructed, where the $(t,i)$-th entry is given by:
    $$
    B_{t,i} = \begin{cases}
        U(S_t \setminus \{n+1\}) & \text{if } i\in S_t, \\
        0 & \text{otherwise}.
    \end{cases}
    $$
    We set $T=20000$.
    Then, the Shapley value is approximated by
    $$
    \widehat{\SV}_{\text{gt}}(i) = \frac{Z}{T}  \sum_{t=1}^T \left(B_{t,i} -  B_{t,n+1} \right),
    $$
    
    where $B_{t,i}$ denotes the utility observed at iteration $t$ when player $i \in S_t$, and $Z = \sum_{s=1}^np(s)$ is a normalization constant.

    \item 
    {\bf Complement Contribution estimator \citep{zhang2023efficient}.}

    The complement estimator approximates Shapley values by leveraging the symmetry between a subset and its complement.
    In particular, it relies on another equivalent representation of the Shapley value:
    $$
    \SV(i) = \frac{1}{n} \sum_{S\subseteq [n]\setminus \{i\}}\frac{U(S\cup\{i\}) - U([n]\setminus(S\cup\{i\}))}{\binom{n-1}{|S|}}.
    $$
    This expression enables the estimator to reuse utility evaluations for both $S$ and its complement $[n] \setminus S$.
    At each iteration, it samples a subset size $s \in [n]$ uniformly and then uniformly samples a subset $S \subseteq [n]$ of size $s$.
    Then, the final estimator is:
    $$
    \widehat{\SV}_{\text{cc}}(i) = \frac{1}{n} \sum_{s=1}^{n} \frac{1}{T_{i,s}}\sum_{t=1}^{T} \left[v_t  \left\{\mathbb{I}(i \in S_t, |S_t| = s) - \mathbb{I}(i \notin S_t, |S_t| = n-s)\right\}\right],
    $$
    where $v_t = U(S_t) - U([n] \setminus S_t)$ and $T_{i,s}$ is the number of such samples satisfying one of the two conditions.
    In our experiment, we set $T =  20000 / 2 = 10000$ so that each iteration, which requires two utility evaluations, results in exactly 20,000 total evaluations.

    \item 
    {\bf One-for-All estimator \citep{li2024one}.}

    The One-for-All estimator is based on the following alternative formulation of the Shapley value\footnote{While the One-for-All estimator can estimate a variety of values (e.g., Banzhaf \citep{wang2023data}, Beta-Shapley \citep{kwon2021beta}), we focus on its use for Shapley value estimation.}:
    $$
    \SV(i) = \frac{1}{n}\sum_{s=1}^n \left( \ep_{S \subseteq [n], |S| = s, i \in S}[U(S)] - \ep_{S \subseteq [n], |S| = s-1, i \notin S}[U(S)] \right),
    $$
    This formulation allows each sampled subset to contribute to the estimates of all players, enabling efficient sample reuse.
    First, it deterministically allocates $2n + 2$ utility evaluations for subset sizes $s \in \{0, 1, n-1, n\}$, and uses them to compute the corresponding expectations exactly.
    Then, for the remaining possible subset sizes $\{2, \dots, n-2\}$, it samples $s$ with a predefined sampling probability $q(s)$, and then samples a subset $S_t \subseteq [n]$ of size $s$ uniformly at random.
    The final estimation is:
    \begin{align*}
        \widehat{\SV}_{\text{ofa}}(i) &=\frac{1}{n}(U([n]) - U(\varnothing)) \\ &+ 
    \frac{1}{n(n-1)} \left(\sum_{i\in S, |S|=n-1}U(S)  - \sum_{i\notin S, |S|=1}U(S) \right) \\
    &+ \frac{1}{n}\sum_{s=2}^{n-2}\left(\frac{1}{T_{i,s}^{\text{in}}}\sum_{t=1}^T U(S_t)\cdot\mathbb{I}(|S_t| = s, i\in S_t) - \frac{1}{T_{i,s}^{\text{out}}}\sum_{t=1}^T U(S_t)\cdot\mathbb{I}(|S_t| = s, i\notin S_t) \right),
    \end{align*}
    where  $T_{i,s}^{\text{in}}$ and $T_{i,s}^{\text{out}}$ denote the number of Monte Carlo samples of size $s$ in which $i$ is included and excluded, respectively.
    In our experiment, we used $q(s)\propto \frac{1}{\sqrt{s(n - s)}}$, which is proved to be the optimal sampling distribution for Shapley value estimation.
    Next, after allocating $2n + 2$ evaluations deterministically for subset sizes $s \in \{0,1,n-1,n\}$, the remaining evaluations are used for Monte Carlo sampling.
    That is, we set $T = 20000 - (2n + 2)$ for sampling subset sizes $s \in \{2,\dots,n-2\}$.

    \item

    {\bf KernelSHAP \citep{lundberg2017unified}.}
    
    The Shapley value can be characterized as the solution to the following constrained optimization problem:
    $$
    \SV = \arg\min_{\phi \in \mathbb{R}^n} \sum_{\varnothing \subsetneq S \subsetneq [n]} w_S \left( U(S) - U(\varnothing) - \sum_{i \in S} \phi_i \right)^2, \quad \text{subject to} \quad \sum_{i=1}^n \phi_i = U([n]) - U(\varnothing),
    $$
    where $w_S \propto \frac{1}{\binom{n-2}{|S|-1}}$ is the kernel weight assigned to each subset $S$.
    Based on this formulation, KernelSHAP constructs a sample-based approximation of the objective and solves the resulting weighted least squares problem to obtain an estimate of the Shapley value.
    In particular, at each iteration, it samples a subset size $s\in \{1, \dots,n-1\}$ according to the distribution $p(s) \propto \frac{1}{s(n-s)}$, and then draws a subset $S_t\subseteq [n]$ of size $s$ uniformly at random.
    Each subset is encoded as a binary indicator vector $\mathbf{1}_S \in \{0,1\}^n$, and the following quantities are updated:
    $$
    \widehat{A} = \frac{1}{T} \sum_{t=1}^T \mathbf{1}_{S_t} \mathbf{1}_{S_t}^\top, \quad \widehat{b} = \frac{1}{T} \sum_{t=1}^T \left(U(S_t) - U(\varnothing)\right) \cdot \mathbf{1}_{S_t}.
    $$
    The KernelSHAP estimator is obtained in closed form as:
    $$
    \widehat{\SV}_{\text{ks}} = \widehat{A}^{-1} \left( \widehat{b} - \frac{\mathbf{1}^\top \widehat{A}^{-1} \widehat{b} - U([n]) + U(\varnothing)}{\mathbf{1}^\top \widehat{A}^{-1} \mathbf{1}} \cdot \mathbf{1} \right),
    $$
    where $\mathbf{1} \in \mathbb{R}^n$ denotes the vector with all-one entries.
    We set $T=20000$ in our implementation of KernelSHAP.

    \item 
    {\bf Unbiased KernelSHAP \citep{covert2021improving}.}

    The Unbiased KernelSHAP estimator is a variant of KernelSHAP that utilizes the fact that the exact gram matrix, $A = \ep[\mathbf{1}_S \mathbf{1}_S^\top]$, admits a closed-form expression under the same distribution $w_S \propto \frac{1}{\binom{n-2}{|S|-1}}$.
    Instead of estimating $A$ empirically from samples, this estimator directly computes $A$, where its $(i,j)$-entry is defined as:
    $$
    A_{ij} = \begin{cases}
    \frac{1}{2} & \text{ if } i = j, \\
    \frac{1}{n(n-1)}\frac{\sum_{s=2}^{n-1}\frac{s-1}{n-s}}{\sum_{s=1}^{n-1}\frac{1}{s(n - s)}} & \text{otherwise.}
    \end{cases}
    $$
    Unbiased KernelSHAP replaces the empirical matrix $\widehat{A}$ in KernelSHAP with its closed-form expectation $A$:
    $$
    \widehat{\SV}_{\text{uks}} = A^{-1} \left( \widehat{b} - \frac{\mathbf{1}^\top A^{-1} \widehat{b} - U([n]) + U(\varnothing)}{\mathbf{1}^\top A^{-1} \mathbf{1}} \cdot \mathbf{1} \right).
    $$
    In our implementation, we set $T=20000$.

    \item 
    {\bf LeverageSHAP \citep{musco2025provably}.}

    LeverageSHAP is another variant of KernelSHAP designed to reduce the variance of the estimator and improve computational efficiency.
    Unlike KernelSHAP, which samples subset sizes $s\in\{1,\dots,n-1\}$ with the weighting $p(s)\propto \frac{1}{s(n-s)}$, LeverageSHAP draws $s$ uniformly (i.e., $p(s)\propto 1$) and compensates for the mismatch via a correction factor $w(s)=\sqrt{s(n-s)}$.
    Also, it adopts paired sampling, which also includes complement $\bar{S}_t = S_t^c$ in the sample pool when each subset $S_t$ is sampled.
    Based on these modification, it computes the following quantities:
    \begin{align*}
        \widetilde{A} &= \frac{1}{2T} \sum_{t=1}^T w(|S_t|) \left\{\mathbf{1}_{S_t} \mathbf{1}_{S_t}^\top + \mathbf{1}_{\bar{S}_t} \mathbf{1}_{\bar{S}_t}^\top\right\} \\ 
    \widetilde{b} &= \frac{1}{2T} \sum_{t=1}^T  w(|S_t|)\left\{\left(U(S_t) - U(\varnothing)\right) \cdot \mathbf{1}_{S_t} + \left(U(\bar{S}_t) - U(\varnothing)\right) \cdot \mathbf{1}_{\bar{S}_t} \right\}.
    \end{align*}
    The final LeverageSHAP estimator is then computed by:
    $$
    \widehat{\SV}_{\text{ls}} = \widetilde{A}^{-1} \left( \widetilde{b} - \frac{\mathbf{1}^\top \widetilde{A}^{-1} \widetilde{b} - U([n]) + U(\varnothing)}{\mathbf{1}^\top \widetilde{A}^{-1} \mathbf{1}} \cdot \mathbf{1} \right).
    $$
    For LeverageSHAP, we set $T = 20000 / 2 = 10000$ since each iteration requires two utility evaluations due to the use of paired sampling.
    
\end{itemize}

Although all methods use the same total number of utility evaluations (20,000), their actual runtimes differ substantially, as shown in the second row of Figure~\ref{fig:sou_results}.
These discrepancies are attributable to several algorithmic and implementation-specific factors, including computational complexity and sampling distribution over subset sizes.

First, \textbf{Group Testing} requires the construction of a utility matrix $B \in \mathbb{R}^{T \times (n+1)}$.
Maintaining this large matrix introduces relatively high memory usage and per-iteration computational overhead, as each utility evaluation must be copied into multiple locations with structural indexing.
In contrast, other estimators maintain only low-dimensional accumulators or regression statistics, which impose negligible memory cost.

Second, the cost of evaluating $U(S)$ varies dramatically with the size of the subset, $|S| = s$.
Consequently, the \emph{subset size sampling distribution} required by each method significantly impacts speed.
Among the benchmarks, 
\textbf{Group Testing}: $p(s)\propto \frac{1}{s} + \frac{1}{n-s+1}$;
\textbf{One-for-All}: $q(s)\propto \frac{1}{\sqrt{s(n-s)}}$;
and both \textbf{KernelSHAP} and \textbf{Unbiased KernelSHAP}: $p(s)\propto \frac{1}{s(n-s)}$.
They all employ non-uniform subset size distributions that tend to oversample either very small or very large subsets.
This matters significantly in the SOU game, where evaluating $U(S)$ requires checking whether $\mathcal{A}_j\subseteq S$ holds for all $j \in [d]$.
Figure~\ref{fig:sou_time} empirically demonstrates the near-exponential growth in computation time for evaluating $U(S)$ as a function of subset size $s$, across varying values of $n \in \{64,128,256\}$.
This occurs because the time required for set inclusion checks (e.g., \texttt{issubset()} in Python) increases rapidly with subset size.
As a result, estimators that disproportionately sample such extreme subset sizes tend to exhibit higher average computation time per evaluation compared to those that sample subset sizes uniformly (e.g., \textbf{Permutation}, \textbf{Complement Contribution}, \textbf{LeverageSHAP}).
Thus, sampling behavior---not just the number of evaluations---directly impacts computational efficiency.
On the other hand, \textbf{FGSV} does not sample subset sizes $s$, but instead performs estimation through an explicit loop over all possible values of $s$.  
As shown in Algorithm~\ref{alg::gsv_v2}, when $s < \bar{s}$, it evaluates the utility function over all grid points $s_1 \in [\max\{0, s + s_0 - n\}, \min\{s, s_0\}]$.  
For $ s \geq \bar{s}$, the approximation is performed using only two representative values of $s_1$, thereby significantly reducing the number of required evaluations.  
This thresholding mechanism ensures that, under a fixed budget of utility calls, FGSV concentrates more computation on smaller subset sizes, leading to lower overall runtime.

\setcounter{figure}{4}
\begin{figure}[h]
    \centering
    \includegraphics[width=\linewidth]{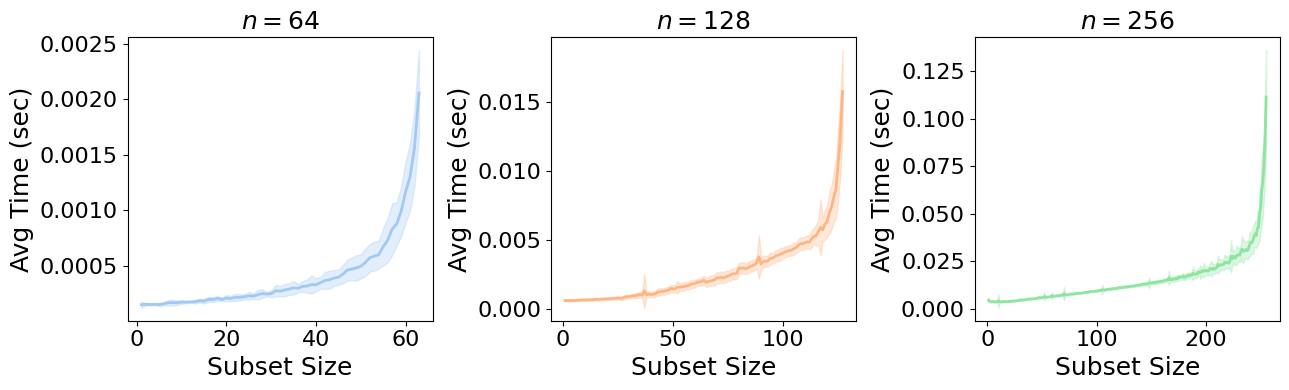}
    \vspace{-1.5em}
    \caption{Empirical average runtime (in seconds) of $U(S)$ evaluation as a function of subset size $s=|S|$. Each curve represents the mean over 50 randomly sampled subsets of size $s$; shaded areas indicate $\pm1$ standard deviation.}
        \vspace{-0.5em}
    \label{fig:sou_time}
\end{figure}

Third, \textbf{KernelSHAP} and its variants—\textbf{Unbiased KernelSHAP} and \textbf{LeverageSHAP}—require solving a constrained least squares problem involving dense matrix inversions and multiplications.
These linear algebra operations, performed either during or after sampling, dominate the overall computational cost, making these methods significantly slower than alternatives that rely solely on running averages or marginal contribution estimates.

\subsection{Faithful copyright attribution in generative AI (Section~\ref{sec::copyright})}

We evaluate SRS and FSRS in the context of group-level copyright attribution for generative models, using a logo generation task based on Stable Diffusion.  
To facilitate comparison, we mostly followed the experimental setup in \cite{wang2024economic}.
Nonetheless, for clarity and completeness, we summarize the relevant details below.

To construct the experiment, we select four logo classes (Google, Sprite, Vodafone, and Starbucks) for evaluation, and use the remaining 23 brands to initialize a baseline model via fine-tuning of Stable Diffusion v1.4~\citep{rombach2022high}.
We then assess the impact of fine-tuning with each of the four excluded brands by measuring how their inclusion alters the generation behavior of the model.

Fine-tuning is performed using Low-Rank Adaptation (LoRA)~\citep{hu2022lora}, a technique that inserts trainable low-rank matrices into the attention layers of the diffusion model.
The rank $r$ and scaling factor $\alpha$ are both set to 8, and fine-tuning is only applied to the attention layers of the UNet module, while all other weights of the Stable Diffusion model are kept frozen.
This configuration ensures efficient and stable adaptation to the small-scale dataset, in line with standard LoRA usage.\footnote{\url{https://huggingface.co/blog/lora}}

We use a learning rate of 0.0001, a batch size of 4, and train for 10 epochs.
After fine-tuning, we generate $N_{\rm MC} = 20$ images for each of the four selected brands using the prompt ``\texttt{A logo by [brand name]}.''  
Image generation is performed with 25 denoising steps and classifier-free guidance (scale 7.5), implemented using the \texttt{DDPMScheduler}.

Following the formulation of \cite{wang2024economic}, we define the utility function as the log-likelihood of generated images.
For a given length of the denoising steps $T$, let $(x_T,\dots,x_1,x_0)$ denote the reverse trajectory defined by the diffusion scheduler, where $x_T\sim \mathcal{N}(0,I)$ is the initial latent noise, and $x_0=x_{(\text{gen})}$ is the final generated image.
The likelihood of $x_{(\text{gen})}$ under model parameters $\theta$ is defined as:
\begin{align*}
    p_\theta(x_{\text{(gen)}}) &= p_\theta(x_0) =\mathbb{E}_{x_1} \left[p_\theta(x_0 \mid x_1)\right],
\end{align*}
by the Markov property of the diffusion process.
The conditional distribution $p_\theta(x_0 \mid x_1)$ is Gaussian, given by
$$
p_\theta(x_0 \mid x_1) \sim \mathcal{N}\left(x_0 ; \frac{1}{\sqrt{\alpha_1}} \left( x_1^{(j)} - \sqrt{1 - \alpha_1} \, \hat{\epsilon}_\theta(x_1^{(j)}, 1) \right), \sigma_1^2 I \right),
$$
where $\hat{\epsilon}_\theta(x_t, t)$ denotes the predicted noise at step $t$, and $\alpha_1, \sigma_1^2$ are scheduler-specific constants.
Based on these formulation, the likelihood of $x_{\text{(gen)}}$ can be approximated as
$$
     \frac{1}{N_{\rm MC}} \sum_{j=1}^{N_{\rm MC}} \mathcal{N}\left(x_0 ; \frac{1}{\sqrt{\alpha_1}} \left( x_1^{(j)} - \sqrt{1 - \alpha_1} \hat{\epsilon}_\theta(x_1^{(j)}, 1) \right), \sigma_1^2 I \right),
$$
where $x_1^{(j)}$ for $j=1,\dots,N_{\rm MC}$ is sampled by reversing the diffusion process from a standard Gaussian noise vector $x_T^{(j)} \stackrel{iid}{\sim} \mathcal{N}(0, I)$.

In our setting, the model is initialized from a pretrained checkpoint based on the remaining 23 brands, which stabilizes the fine-tuning process even when the input subset is small.  
This design parallels the data augmentation strategy described in Section~\ref{sec::noninf}, where non-informative examples are added to ensure that the utility function remains well-defined. 
Ultimately, the resulting setup can be conceptually viewed as one where the utility function consistently receives a sufficiently large dataset as input.
Hence, the computation procedure can be understood as effectively a special case of Algorithm~\ref{alg::gsv_v2}, with the difference that our computation injects \emph{informative}, rather than non-informative, data points.
As a result, when theoretically understanding the performance of our method here, we resort to Theorem~\ref{proposition::calT}, thinking that $s$ is lower-bounded.

Finally, we used $m = 2$ for this approximation.
This choice was primarily informed by practical constraints:
the computation for a single sample, including both fine-tuning and utility evaluation, takes nearly 30 hours given the limited hardware resources available to us.
On the other hand, however, we observed that the FSRS estimates well-illustrates the \emph{faithfulness} property of FGSV---the preservation of group-level value under further partitioning---even within each single experiment.
As shown in Figure~\ref{fig:srs-comparison}, which displays the result from a single experiment, the FSRS values remained stable before and after splitting the Google and Sprite brands.
This suggests that the empirical evidence is clear even with a small $m$.
To understand this phenomenon, we attribute such empirical stability to the fact that the diffusion model here was initialized by a pretrained baseline model.
The baseline model was trained by on big data with very significant computing resources, which tends to yield rather stable initializations.
Consequently, fine-tuning outcomes across random seeds tend to inherit some stability and would not vary wildly.

\subsection{Faithful explainable AI (Section~\ref{sec::explainable})}

We evaluated FGSV and GSV on the Diabetes dataset~\citep{efron2004least}, which includes 442 samples with 10 demographic and health-related features.
The dataset is preprocessed by centering and standardizing the response variable and splitting the dataset into 350 training and 92 test samples using a fixed random seed.

For the base model, we use ridge regression with regularization strength $\alpha = 0.01$.  
The utility $U(S)$ is defined as the negative mean squared error (MSE) on the test set; for empty subsets $S$, the utility is defined as the negative variance of the test labels, which corresponds to a baseline predictor that always outputs the test mean.
FGSV is estimated using Algorithm~\ref{alg::gsv_v2} with parameters $\bar{s} = 35$ and $m_1= m_2 = 1000$, while GSV is computed exactly.
Each setting is repeated for 30 times using independent Monte Carlo replications.

\clearpage

\section{Approximation algorithm with non-informative data augmentation}
\label{sec::algorithm-2}

\begin{algorithm}[H]
\caption{Approximate FGSV with non-informative data augmentation for small input sizes}
\label{alg::gsv_v2}
\begin{algorithmic}[1]
\REQUIRE Dataset $\mathcal{D}$, group $\mathcal{S}_0$, size threshold $B$, subsample size $m$, non-informative distribution $\mathcal{P}_{\operatorname{null}}$.
\STATE Initialize $n \gets |\mathcal{D}|$, $s_0 \gets |S_0|$, $\alpha_0 \gets s_0 / n$.
\STATE Initialize total sum for correction terms $\hat{\mathcal{T}}_{sum} \gets 0$.
\FOR{$s = 1$ to $n - 1$}
    \STATE Set $s_1^* \gets \lfloor s \alpha_0 \rfloor$. \COMMENT{Expected intersection size for current $s$}
    \STATE Initialize sum of utility differences $S_{\Delta U} \gets 0$.
    \FOR{$j=1$ to $m$}
        \STATE Sample a base tuple $(\mathcal{S}, z_1, z_2)$ i.i.d. from $\mathscr{B}_{s, s_1^*} = \{(\mathcal{S}, z_1, z_2): \mathcal{S} \subseteq \mathcal{D}, |\mathcal{S}|=s, |\mathcal{S}\cap \mathcal{S}_0| = s_1, z_1 \in \mathcal{S}_0 \setminus \mathcal{S}, z_2 \in \mathcal{S}_0^c \setminus \mathcal{S}\}$, where $\mathcal{S}_0 = \{z_i: i \in S_0\}$.
        \IF{$s < B$}
            \STATE Sample augmentation set $\mathcal{S}_{null} = \{z'_1, \dots, z'_{B - s}\}$ where $z'_l \stackrel{\text{i.i.d.}}{\sim} \mathcal{P}_{\operatorname{null}}$.
            \STATE $U_1 \gets U(\mathcal{S} \cup \{z_1\} \cup \mathcal{S}_{null})$.
            \STATE $U_2 \gets U(\mathcal{S} \cup \{z_2\} \cup \mathcal{S}_{null})$.
        \ELSE
            \STATE $U_1 \gets U(\mathcal{S} \cup \{z_1\})$.
            \STATE $U_2 \gets U(\mathcal{S} \cup \{z_2\})$.
        \ENDIF
        \STATE $S_{\Delta U} \gets S_{\Delta U} + (U_1 - U_2)$.
    \ENDFOR
    \STATE $\widehat{\Delta\mu}\left(\frac{s_1^*}{s}; s, s_0, n\right) \gets \frac{1}{m} S_{\Delta U}$.
    \STATE $\hat{\mathcal{T}}(s) \gets \frac{n}{n - 1} \alpha_0(1 - \alpha_0) \cdot \widehat{\Delta\mu}\left(\frac{s_1^*}{s}; s, s_0, n\right)$. 
    \STATE $\hat{\mathcal{T}}_{sum} \gets \hat{\mathcal{T}}_{sum} + \hat{\mathcal{T}}(s)$.
\ENDFOR
\STATE $G_0 \gets \frac{s_0}{n} \left[ U([n]) - U(\varnothing) \right]$.
\RETURN $G_0 + \hat{\mathcal{T}}_{sum}$. 
\end{algorithmic}
\end{algorithm}

\end{document}